\documentclass[lettersize,journal]{IEEEtran}
\usepackage{amsmath,amsfonts,amsthm}
\usepackage{array}
\usepackage{textcomp}
\usepackage{stfloats}
\usepackage{url}
\usepackage{verbatim}
\usepackage{graphicx}

\usepackage{enumitem}
\usepackage{amssymb}
\usepackage{xcolor}
\usepackage[ruled, noend, linesnumbered]{algorithm2e}
\usepackage{caption}
\usepackage{subcaption}
\usepackage{dsfont}
\usepackage{booktabs}
\usepackage{multirow}
\usepackage{cite}

\SetKw{KwPhase}{Phase}
\let\oldnl\nl
\newcommand{\nlnonumber}{\renewcommand{\nl}{\let\nl\oldnl}}

\newcommand{\Conv}{\mathop{\scalebox{1.5}{\raisebox{-0.2ex}{$\ast$}}}}%

\newcommand{\IG}[1]{\textcolor{blue}{[IG: #1]}}
\newcommand{\ml}[1]{\textcolor{cyan}{[ML: #1]}}

\newtheorem{theorem}{Theorem}
\newtheorem{problem}{Problem}
\newtheorem{remark}{Remark}
\newtheorem{lemma}{Lemma}
\newtheorem{assumption}{Assumption}
\newtheorem{example}{Example}
\newtheorem{definition}{Definition}
\newtheorem{proposition}{Proposition}

\newcommand{\naturals}{\mathbb{N}}
\newcommand{\reals}{\mathbb{R}}
\newcommand{\WS}{\text{WS}}

\def\BibTeX{{\rm B\kern-.05em{\sc i\kern-.025em b}\kern-.08em
    T\kern-.1667em\lower.7ex\hbox{E}\kern-.125emX}}
\usepackage{balance}
\begin{document}
\title{\LARGE \bf Provably Safe Motion Planning Under Unknown Disturbances
% *\thanks{*This work was not supported by any organization}
}

\author{Ibon Gracia$^{1*}$\thanks{* Corresponding Author}\thanks{$^{1}$The authors are with the Department of Aerospace Engineering Sciences at the University of Colorado Boulder, CO , USA {\tt\small \{\textit{firstname}.\textit{lastname}\}@colorado.edu}}, Qi Heng Ho$^{2}$\thanks{$^{2}$Qi Heng Ho is with the Department of Aerospace and Ocean Engineering, Virginia Tech, VA, USA {\tt\small qihengho@vt.ed}}, Luca Laurenti$^{3}$\thanks{$^{3}$The author is with the Delft University of Technology, the Netherlands, and the Italian Institute of Artificial Intelligence, Italy {\tt\small luca.laurenti@tudelft.nl}}, Morteza Lahijanian$^{1}$}% <-this % stops a space
% \author{Anonymous Authors}% <-this % stops a space
% <-this % stops a space
%
%%
%%
%         Delft University of Technology, 2628 CD Delft, Netherlands
%         {\tt\small albert.author@papercept.net}}
% }

% \markboth{Journal of \LaTeX\ Class Files,~Vol.~18, No.~9, September~2020}%
% {How to Use the IEEEtran \LaTeX \ Templates}

\maketitle

\begin{abstract}
We present a provably safe sampling-based motion planning algorithm for robotic systems affected by random disturbances of \emph{unknown} distribution. We consider systems with linear or linearizable dynamics evolving in workspace with arbitrary-shaped obstacles subject to state and control constraints.  Safety requirements are formulated as chance-constraints. Our approach leverages data from trajectories of the system to learn a \emph{Wasserstein} \emph{ambiguity tube}, i.e., a sequence of ambiguity sets, which contains the trajectory of the system's state distribution with high confidence. This ambiguity tube is then used in a probabilistically complete algorithm to grow a sampling-based motion planning tree that respects the constraints of the problem. We show that learning several lower-dimensional ambiguity tubes instead of a single high-dimensional one effectively reduces the conservatism and boosts scalability. Additionally, we design an efficient bandit-based validity checker that remarkably increases the empirical performance of our approach  without sacrificing probabilistic completeness. 
% We demonstrate through several case studies that our algorithm is able to find valid paths even in cluttered environments while respecting a high threshold in the probability of safety. We also compare our method against state-of-the-art approaches, showing its superiority.
Case studies show our algorithm finds valid plans in cluttered environments under strict safety thresholds, outperforming state-of-the-art methods.
\end{abstract}


\section{Introduction}
    \label{sec:intro}

\IEEEPARstart{M}{otion} planning is a central problem in robotics, with applications spanning autonomous exploration, manufacturing, and surgical systems. Sampling-based methods have been particularly successful \cite{lavelle1998rrt,karaman2011sampling,SST, cortes2020samplingbased}: by drawing random samples, they rapidly construct a search tree and return a collision-free trajectory from start to goal. Traditionally, these algorithms assume a known, deterministic model of the underlying system. In reality, however, robotic systems are subject to disturbances arising from unmodeled dynamics, actuator noise, environmental variability, or model mismatch. Prior work, spanning both sampling-based and trajectory optimization frameworks, has typically attempted to capture such uncertainties using bounded disturbance models \cite{luders2014optimizing,wu2022robust,Majumdar2017,Lopez2019tubeMPC}
or Gaussian noise assumptions \cite{CC-RRT,Ho2022gbt, lindemann2021RobustMP,Kantaros2022perceptionTLplanning,prentice2009belief}, and to design motion plans that are robust under these models. Yet, these strategies are often conservative and rely on strong distributional assumptions that rarely hold in practice. In realistic settings, disturbances are stochastic with \emph{unknown} or \emph{partially known} distributions, making reliable planning significantly more challenging. This work aims to address this challenge by developing an efficient, data-driven motion planning framework that quantifies and propagates uncertainty directly from system data to produce provably safe trajectories.

Recent work has increasingly focused on motion planning under uncertainty \cite{luders2014optimizing,CC-RRT, Bry2011,Ho2022gbt}. Approaches based on bounded non-deterministic disturbances often yield overly conservative plans, as they ignore the low likelihood of extreme disturbance realizations. To incorporate probabilistic information, methods such as \cite{CC-RRT, Bry2011,Ho2022gbt} model disturbances as Gaussian and enforce chance constraints. In practice, however, the disturbance distribution is rarely known and must be learned from trajectory data. Distributionally robust methods address this by constructing ambiguity sets over (sets that contain) plausible distributions and planning against all distributions in these sets \cite{summers2018distributionally,ekenberg2023distributionally}. While theoretically appealing, these approaches remain highly conservative (even when obstacles are convex), and the conservatism grows with environmental complexity, rendering them impractical for cluttered planning scenarios.

In this work, we introduce an efficient planning algorithm for systems under unknown disturbance distributions. We use trajectory data to construct a \emph{Wasserstein ambiguity tube}, a high-confidence sequence of ambiguity sets capturing stochastic evolution, and incorporate it into a sampling-based planner. We show that, for linear (or feedback-linearizable) systems, ambiguity propagation decouples from planning, allowing a single offline-learned tube with tight uncertainty characterization. Then, we enforce safety via chance constraints using the optimization-free worst-case collision checker of \cite{gao2023distributionally}, enabling planning in cluttered environments and handling nonconvex constraints. 
% The tube has infinite temporal length with finite confidence, eliminating the need to predefine the planning horizon. 
% We further develop multiple vality-checking algorithms with different trad-off between conservatism and computational complexity.  Then, we unify these methods in a decision framework that uses appropriate method based on multi-arm bandit formulation.
% We prove soudness of the construction of the tubes and validity checking and probabilistic completeness with respect to the tube and show that leveraging multiple low-dimensional tubes significantly reduces data and computation. 
The tube has infinite temporal length with finite confidence, removing the need to specify a planning horizon. We further develop complementary validity-checking procedures with distinct efficiency–conservatism trade-offs, and unify them through a multi-armed bandit framework that adaptively selects the appropriate checker. We prove soundness of both the tube construction and validity checking, as well as probabilistic completeness of the algorithm with respect to the tube, and show that leveraging multiple low-dimensional tubes significantly reduces data and computation.
% Our evaluations and benchmarks demonstrate efficacy of the algorithm in complex environments and its superior performance over distributionally robust baselines \cite{summers2018distributionally, ekenberg2023distributionally}.
Benchmark results confirm the efficacy of the proposed algorithm in complex environments and its consistent outperformance of state-of-the-art methods.

The main contributions of this paper are five-fold:
\begin{itemize}
    % \item A probabilistically complete sampling-based motion planning algorithm for linear systems with unknown additive disturbances from trajectory observation data. 
    % \item A proof of probabilistic completeness for a Wasserstein sampling-based motion planning algorithm under relaxed assumptions. 
    % \item Novel theoretical results about propagation of distributional ambiguity.
    % \item A tighter expression for the size of data-driven ambiguity sets defined by the $1$-Wasserstein distance than the ones used in the literature.
    % \item Three efficient validity checking approaches that trade off efficiency and conservativeness, while all being less conservative than state-of-the-art approaches, as well as, a novel multi-armed bandit framework that adaptively selects the appropriate checker.
    % \item Detailed evaluation and benchmark studies  which validate our approach and which show its advantages against state-of-the-art approaches.
    \item A sampling-based motion planning algorithm for linear systems with unknown additive disturbances that uses trajectory observation data.
    \item Proofs of soundness and probabilistic completeness for the planner under relaxed assumptions.
    \item New theoretical results on the propagation of distributional ambiguity.
    \item A tighter characterization of the size of data-driven ambiguity sets defined by the 1-Wasserstein distance than those used in the existing literature.
    \item 
    % Three efficient validity-checking methods that trade off computational efficiency and conservativeness, each less conservative than state-of-the-art methods, together with a novel multi-armed bandit scheme that adaptively selects among them.
    Three validity-checking methods trading off efficiency and conservativeness, each less conservative than existing work, as well as a multi-armed bandit scheme that adaptively selects among them.
    \item Comprehensive evaluation and benchmark studies demonstrating the efficacy of the approach and its advantages over the state-of-the-art approaches.
\end{itemize}

\subsection{Related Work}

% Classical chance-constraint motion planning approaches assume linear dynamics with additive Gaussian noise. \cite{luders2014optimizing} presents a tree-based motion planning algorithm for uncertain linear robotic systems around uncertain obstacles. The dynamics of both the robot and the obstacles are assumed known except for set-bounded additive disturbances. Assuming a known Gaussian law for the disturbances, \cite{CC-RRT, Bry2011, Ho2022gbt} reduces the conservatism of the previous setting. Validity of the nodes in the tree is then checked via \emph{chance constraints}.

Classical chance-constrained motion planning assumes linear dynamics with additive Gaussian noise. \cite{luders2014optimizing} presents a tree-based algorithm for uncertain linear systems with set-bounded disturbances around uncertain obstacles, while \cite{CC-RRT,Bry2011,Ho2022gbt} reduce conservatism by assuming a known Gaussian disturbance law and checking node validity via chance constraints.

% In order to generalize to non-Gaussian disturbances, \cite{summers2018distributionally} proposes a distributionally robust approach: assuming that the mean and covariance of the noise and initial state distributions are precisely known, they obtain the mean and covariance of the state distribution in closed form, and then construct \emph{moment ambiguity sets} around them. Such a set contains all distributions with fixed mean and covariance, and therefore the state distribution. The work then concludes a node is valid by checking if chance-constraints are satisfied with respect to all distributions in the ambiguity set. Assuming convex polytopic obstacles, the approach this condition simplifies to checking a deterministic inequality, although introducing great conservatism in the computed probability of collision. Besides, this method uniformly allocates the admissible risk of collision to all obstacles, not taking into account their distance to the nodes. In consequence, this approach struggles to find a valid paths in cluttered environments, when the number of obstacles and the desired probability of safety are high. In order to mitigate these issues, \cite{ekenberg2023distributionally} allocates, to each obstacle, a risk equal to the upper bound in the probability of collision with that individual obstacle.  Although this logic eases planning around a high number of obstacles  and when few obstacles are located closely to each other, he solution is still fairly conservative. 

To generalize beyond Gaussian disturbances, \cite{summers2018distributionally} proposes a distributionally robust approach using \emph{moment ambiguity sets} built from the (assumed known) mean and covariance of the noise and initial state. For convex polytopic obstacles, the chance constraint reduces to a deterministic inequality, but at the cost of substantial conservatism. Moreover, admissible risk is allocated uniformly across obstacles regardless of proximity, which hinders planning in cluttered environments with high safety thresholds. \cite{ekenberg2023distributionally} mitigates this by allocating per-obstacle risk equal to its individual collision-probability upper bound, but the solution remains fairly conservative.

Other works use Wasserstein ambiguity sets, defined as balls centered on a nominal (often empirical) distribution. Unlike moment sets, Wasserstein sets shrink to a single distribution as samples grow \cite{mohajerin2018data}, avoiding the irreducible conservatism of moment-based approaches \cite{gao2023distributionally}. Early work focused on \emph{static} sets \cite{mohajerin2018data,gao2023distributionally,blanchet2019quantifying}; more recent work studies propagation through dynamical systems \cite{boskos2023high,aolaritei2022uncertainty,aolaritei2023capture,figueiredo2025efficient}. \cite{aolaritei2022uncertainty} characterizes when optimal-transport ambiguity sets are closed under linear and nonlinear mappings, and \cite{aolaritei2022uncertainty,aolaritei2023capture} bound ambiguity propagation through linear stochastic systems.  However, these bounds grow unbounded in time, ruling out long horizons. \cite{boskos2023high} give conditions for time-bounded ambiguity under process and measurement noise, and \cite{figueiredo2025efficient} provides a tighter bound leveraging the contractiveness of the dynamics. However, these results are still too conservative for linear robotic models and cluttered environments. In this work, we provide a tighter characterization: unlike \cite{aolaritei2022uncertainty,aolaritei2023capture,lathrop2021distributionally}, we learn an ambiguity tube of infinite length but finite radius that contains the state distribution over time with user-defined confidence.

% Wasserstein ambiguity sets have been recently employed in sampling-based motion planning algorithms \cite{hakobyan2022distributionally,lathrop2021distributionally}. In \cite{hakobyan2022distributionally}, the robot evolves according to known deterministic dynamics, whereas the obstacles are modeled as \emph{Gaussian Process} (GP) regression models. The authors propose to predict the motion of the obstacles by linearizing the GP predictions and then robustifying said predictions via Wasserstein ambiguity sets. However, no guarantee in the safety of the computed plan is obtained. The risk measure employed is the \emph{conditional value at risk} (CVaR), because of its coherence \cite{artzner1999coherent} and computational tractability. 

Wasserstein sets have been used in sampling-based planning \cite{hakobyan2022distributionally,lathrop2021distributionally} and in 
\emph{model predictive control} (MPC) \cite{hakobyan2020wasserstein,aolaritei2023capture,aolaritei2023wasserstein}, typically with the \emph{conditional value at risk} (CVaR) measure for its coherence \cite{artzner1999coherent} and tractability. \cite{hakobyan2022distributionally} robustifies linearized GP obstacle predictions via Wasserstein sets but provides no safety guarantee. \cite{hakobyan2020wasserstein} applies nonlinear MPC with obstacles characterized by data-driven Wasserstein sets, and \cite{aolaritei2023wasserstein} gives a convex tube-MPC reformulation for stochastic LTI systems under convex safety constraints. A key drawback of CVaR in these settings is that program complexity scales linearly with the sample count \cite{hakobyan2020wasserstein,aolaritei2023capture}, which becomes prohibitive at the sample sizes needed for tight guarantees \cite{mohajerin2018data}.

We instead define collision risk as a chance constraint and use the optimization-free algorithm of \cite{gao2023distributionally} to compute the exact worst-case collision probability. Unlike \cite{summers2018distributionally,ekenberg2023distributionally,aolaritei2023capture,lathrop2021distributionally}, this handles obstacles of arbitrary shape with cost independent of their number. Combined with our tighter ambiguity characterization, this enables planning in cluttered environments, control-effort limits (unlike \cite{summers2018distributionally,ekenberg2023distributionally,lathrop2021distributionally}), and general non-convex constraints, with tube size and conservatism shrinking in the sample count.

% Unlike \cite{hakobyan2022distributionally}, \cite{lathrop2021distributionally} considers the setting where the uncertainty is present in both the robotic system and the obstacles, and defines the risk measure as the Wasserstein distance between them. However, the work assumes that the disturbance is gaussianly distributed, which is a strong assumption. Furthermore, although the algorithm in \cite{lathrop2021distributionally} is shown to be probabilistically complete, strong assumptions are required for this statement to hold: specifically, it requires existence of a motion plan such that the support of the robot and obstacles' distributions never overlap, which neglects probabilistic information. By leveraging it instead, in this paper we prove that our algorithm is probabilistically complete under weaker assumptions.

Unlike \cite{hakobyan2022distributionally}, \cite{lathrop2021distributionally} handles uncertainty in both robot and obstacles, using the Wasserstein distance between them as the risk measure, but assumes Gaussian disturbances. Its probabilistic completeness also requires a motion plan whose robot and obstacle supports never overlap, discarding probabilistic information. We prove probabilistic completeness under weaker assumptions by exploiting that information.

% We highlight that \cite{hakobyan2020wasserstein,hakobyan2022distributionally,lathrop2021distributionally,aolaritei2023capture,aolaritei2023wasserstein} and, to the best of our knowledge, all sampling-based motion planning algorithms using Wasserstein ambiguity sets treat the size of the ambiguity set as a tuning parameter, instead of determining a sufficient size that guarantees containment of the unknown distributions. 
% We believe that this choice is a consequence of the large number of samples required to obtain guarantees for Wasserstein ambiguity sets, which translates into excessive computational burden if not dealt with carefully. A big portion of the efforts in this paper is aimed at mitigating this issue. To this end, we show how learning several lower-dimensional ambiguity tubes instead of a higher dimensional one often reduces the sample and computational complexity by orders of magnitude. For a more thorough research on such techniques, see \cite{chaouach2023structured}.

% Finally, besides approaches that leverage ambiguity sets to hedge against unknown distributions, it is worth mentioning \cite{janson2017monte}, which proposes a motion planning framework for linear systems by leveraging \emph{Monte Carlo} sampling.

We note that \cite{hakobyan2020wasserstein,hakobyan2022distributionally,lathrop2021distributionally,aolaritei2023capture,aolaritei2023wasserstein} and, to our knowledge, all sampling-based planners using Wasserstein sets treat the set size as a tuning parameter rather than computing a size that guarantees containment, likely due to the sample and compute burden of formal guarantees. We address this by learning several lower-dimensional ambiguity tubes rather than one high-dimensional tube, often reducing sample and computational complexity by orders of magnitude \cite{chaouach2023structured}.

\section{Preliminaries}
\label{sec:preliminaries}

We first introduce notation and review the fundamentals of sampling-based planners before formalizing the problem.

\subsection{Basic Notation}
    \label{sec:notations}
    % \section{Basic Notation}

Consider the space $\mathbb R^n$ equipped with the Euclidean distance. Given a set $X\subseteq \mathbb R^n$ and a point $x\in \mathbb R^n$, we denote by $\text{dist}(x, X)$ the minimum Euclidean distance between $x$ and $X$, and by $\text{diam}(X)=\sup_{x,x' \in X} \|x-x'\|$
the diameter of $X$. We also denote by $X+x$ the \emph{Minkowski sum} of the sets $X$ and $\{x\}$, and by $\mathds{1}_X(x)$ the indicator function of $X$, which returns $1$ if $x \in X$ and $0$ otherwise.  
%Let $M \in \mathbb R^{m\times n}$ be a real-valued matrix. We denote by $M(X)$ the image of $X$ through $M$ and
We use bold symbols to denote random variables, e.g., $\boldsymbol{x}$. We define $\mathcal{D}(X)$ to be the set of Borel probability measures over the metric space $X$ such that $\int_X \|x\| \, dP(x) < \infty$. 

Let $\boldsymbol{x}$ be distributed according to $P \in \mathcal{D}(X)$, and $X' \subseteq X$ 
be Borel-measurable. We write $P(X') \equiv P[\boldsymbol{x} \in X']$ 
to denote the measure of set $X'$ with respect to $P$. Given a probability distribution $P\in\mathcal{D}(X)$ and a measurable map $f:X\longrightarrow Y$, we denote by $f_\#P \in \mathcal{D}(Y)$ the \emph{pushforward measure} of $P$ by $f$, i.e., the measure defined as $f_\#P(B) := P(f^{-1}(B))$ for all Borel sets $B \subseteq Y$, with $f^{-1}(B)$ denoting the pre-image of set $B$. When $Y\subseteq \mathbb R^m$ and $f$ is a linear transformation with matrix $M \in \mathbb R^{m\times n}$, we also use the notation $M_\#P\in\mathcal{D}(\mathbb R^n)$.
%
\iffalse
Let $c:X\times X\rightarrow \mathbb{R}_{\ge 0 }$ be a continuous cost function. We denote by $\mathcal{D}(X)$ the set of distributions $P\in\mathcal{D}(X)$ such that $\int_X c(x,x') dP(x) < \infty$ for some $x'\in X$.
%
\fi
%
The $1$-\emph{Wasserstein distance} between distributions $P,P'\in \mathcal{D}(X)$ is then defined as
\begin{align*}
    \mathcal{W}(P,P') = \inf_{\pi\in\Pi(P,P')}\int_{X\times X} \|x-x'\| \, d\pi(x,x'),
\end{align*}
where $\Pi(P,P')$ is the set of probability distributions on $\mathcal{D}(X\times X)$ with marginals $P$ and $P'$.
%
\iffalse
Let $\|\cdot\|$ be the Euclidean norm on $\mathbb R^n$. If $c(x,x') = \|x - x'\|$ for all $x,x'\in X$, we write $\mathcal{D}_1(X)$ and $\mathcal{W}(P,P') := \mathcal{T}_c(P,P')$ 
is the $1$-Wasserstein distance between $P$ and $P'$.
%We let $\mathbb B(\widehat P, \varepsilon) \subset \mathcal{D}_1(X)$ denote the ball in the metric space $(\mathcal{D}_1(X), \mathcal{W})$ with center $\widehat P$ and radius $\varepsilon > 0$.
%
\fi
%
We denote by $\text{supp}(P)$ the support of $P$ and by $\mathcal{M}_q(P) : = \big(\mathbb E_{P}[\|\boldsymbol{x}\|^q]\big)^{1/q}$ the  $q$-th moment of $P$. 
Finally, we let $\delta_{x}\in\mathcal{D}(X)$ be the Dirac measure located at $x\in X$ and we denote by $\mathbb B(\widehat P, \varepsilon) := \{P \in \mathcal{D}(X) : \mathcal{W}(P, \widehat P) \le \varepsilon\}$ the Wasserstein Ball with center (nominal distribution) $\widehat P\in \mathcal{D}(X)$ and radius $\varepsilon > 0 $. Given a set of $N$ i.i.d. samples $\{\hat{\boldsymbol x}^{(i)}\}_{i = 1}^N$ from some distribution $P$, the corresponding \emph{empirical distribution} is
\begin{align*}
    \widehat P := \frac{1}{N}\sum_{i = 1}^N \delta_{\hat{\boldsymbol x}^{(i)}}.
\end{align*}

\subsection{Kinodynamic Sampling-based Planners}

\label{sec:treebasedplanners}
% Single query sampling-based tree search algorithms construct a tree in the search space. For kinodynamic systems without any uncertainty, this search space is the state space. In this work, we aim to develop a method to transform such existing planners into planners that can handle unknown stochastic disturbances and uncertain system initial distributions. To this end, we first present an overview of such planners.

Single-query sampling-based tree search algorithms construct a tree in the search space, which is the state space for kinodynamic systems without uncertainty. We aim to transform such planners into ones that handle unknown stochastic disturbances and uncertain initial distributions.
% , and begin with an overview of them.

Alg.~\ref{alg:treebasedplanners} shows a generic form of a
tree-based planner for kinodynamical systems. It takes state space $X$, input space $U$, goal region $X_{goal} \subset X$, obstacle regions $X_{obs} \subset X$, an initial state $x_{init} \in X$, and a maximum planning time or iteration count (N) as input and returns a near-optimal solution if one is found. Search is performed by growing a motion tree in which states and the connections between them are stored as nodes in $\mathbb{V}$ and edges in $\mathbb{E}$, respectively.
The six main subroutines for the planner are: \texttt{sample}, \texttt{select}, \texttt{extend}, \texttt{validity check} and \texttt{goal check}. In \texttt{sample}, a state is randomly sampled from the state space. Then, \texttt{select} chooses an existing node on the tree based on this sample. The node is extended by sampling a control and propagating the system's dynamics in \texttt{extend}. A \texttt{validity check} on this new state is conducted, at which point the new node is added to the tree. Then, \texttt{goal check} assesses whether or not the new node has reached the goal, in which case the algorithm returns the resulting path. If $\mathcal{X}$ is an asymptotically optimal planner, then the algorithm will terminate in finite time with probability $1$.
%\ml{there is nothing in the description that makes Alg 1 asymptotically optimal.  Should we remove that part and replace the FOR-loop with a WHILE=loop in Alg.1?}

\begin{algorithm}[t]
    \caption{Generic Tree-based Planner $\mathcal{X}$ 
    % ($X, U, X_{goal}, X_{obs}, x_{init}$, N)
    }
    \label{alg:treebasedplanners}
    \SetKwInOut{Input}{Input}
    \SetKwInOut{Output}{Output}
    \Input{$X, U, X_{goal}, X_{obs}, x_{init}$, N}
    \Output{Valid Trajectory $x_{1:T}$ if one is found}
    $G = (\mathbb{V} \leftarrow \{x_{init}$\}, $\mathbb{E} \leftarrow \emptyset)$\\
    \While{$\text{True}$}{
        $x_{rand}, u_{rand} \leftarrow$ \texttt{Sample}()\\
        $n_{select} \leftarrow$ \texttt{Select}($x_{rand}$)\\
        $n_{new} \leftarrow$ \texttt{Extend}($n_{select}, u_{rand}$)\\
        \uIf {\texttt{ValidityCheck}($n_{select}, n_{new}$)}{
            $\mathbb{V} \leftarrow \mathbb{V} \cup \{n_{new}$\}\\
            $\mathbb{E} \leftarrow \mathbb{E} \cup \{edge(n_{select}, n_{new})\}$\\
        }
    \If{\texttt{GoalCheck}($n_{\text{new}}$)}{
            \Return \texttt{ExtractPath}($G, n_{\text{new}}$)\;
        }
    }
\Return Failure
\end{algorithm}
\section{Problem Formulation}
\label{sec:problem_formulation}

% In this section we formalize the class of systems that we consider and state the problem to be solved. 
We focus on robotic systems whose motion can be described by the stochastic linear dynamics
\begin{align}
\label{eq:system}
    \boldsymbol{x}_{t+1} = A\boldsymbol{x}_t + Bu_t + G\boldsymbol{w}_t,
\end{align}
where $\boldsymbol{x}_t \in X \subseteq \mathbb{R}^n$ is the state at time $t\in\mathbb N_0$, $u_t \in U \subseteq \mathbb{R}^m$ is the control input, 
$A$, $B$, and $G$ are real matrices of appropriate dimensions, and
$\boldsymbol{w}_t \in W \subset \mathbb{R}^d$ is an i.i.d. random disturbance (noise) with distribution $\boldsymbol{w}_t \sim P_w \in \mathcal{D}(W)$. The initial state $\boldsymbol{x}_0 \in X_0 \subset X$ is distributed according to $P_0\in \mathcal{D}(X_0)$, i.e., $\boldsymbol{x}_0 \sim P_0$.

In this work, we consider the settings in which distributions of the noise $P_w$ and the initial state $P_0$ are \emph{unknown}, but their supports $W$ and $X_0$ are known. This assumption is realistic and in fact commonly encountered in robotic systems operating under uncertainty. First, the linear time-invariant modeling assumption is justified by the fact that many robotic platforms are control‐affine and hence feedback‐linearizable.
% subjected to noise: first, the LTI assumption is justified by most robotic systems being control-affine, and thus feedback-linearizable. 
Second, our assumptions on $P_0$ and $P_w$ relax those in existing work, which often require not only the knowledge of these distributions but also that they be Gaussian.
% 
% We assume that the initial state 
%$x_0 \sim P_0$ and $w_t \sim P_w$ i.i.d., with $P_0, P_w$ unknown, but we know their support.
%$\phi_0 := \text{diam}(P_0)$ and $\phi_w := \text{diam}(P_w)$ 
% 
% \IG{actually knowing a higher-order moment also works}\IG{Also our approach can be extended to deal with input constraints}.
% 
In lieu of not knowing $P_w$ and $P_0$, we assume that we have access to $N$ i.i.d. trajectories 
$\{(\hat{\boldsymbol{x}}^{(i)}_0, \hat{\boldsymbol{x}}^{(i)}_1, \ldots, \hat{\boldsymbol{x}}^{(i)}_H)\}_{i=1}^N$ of the system for some horizon $H \in \mathbb N_0$.
%\ml{Justify that many robotics systems can be linearized, e.g., through feedback linearization since they are mostly control affine, hence our framework is general for such systems. i.e., our linear assumption on the dynamics is not really a major limitation.}

% We assume, without loss of generality, that the robot is a point that 
The robot operates in workspace $\WS \subset \reals^\text{ws}$, where $n_\text{ws} \in \{2, 3\}$, surrounded by workspace obstacles which it must avoid to reach a goal region. In addition to these obstacles, the state of the robot might be subject to constraints, such as velocity or rate limits. We represent all these workspace obstacles and state limits by the %(possibly time-varying)
state-space \emph{obstacle set} $X_\text{obs} \subset X$. Similarly, the goal region that the robot has to reach in the state space is denoted by $X_\text{goal} \subset X$.  
% , and we let this one and $X_\text{obs}$ have arbitrary shape.
We allow $X_\text{goal}$ and $X_\text{obs}$ to have arbitrary shapes. Our goal is to motion plan for this robot with safety guarantees. 

To enable stable robot motion, similar to \cite{Bry2011}, we consider feedback controllers of the form:
\begin{align}
\label{eq:control_input}
    \boldsymbol{u}_t := -K(\boldsymbol{x}_t - \bar x_t) + \bar u_t,
\end{align}
where $K \in \reals^{m\times n}$ is the feedback control gain, $\bar u_t \in \mathbb{R}^m$ is feedforward control, and $\bar x_t \in X$ is the \emph{reference state}. The latter is described by the reference open-loop dynamics
\begin{align}
\label{eq:reference_dynamics}
    \bar x_{t+1} = A \bar x_t + B\bar u_t,
\end{align}
with arbitrary $\bar x_0$. Defining $A_\text{cl} := A - BK$, the closed-loop dynamics of the system become 
\begin{align}
\label{eq:closed_loop_dynamics}
    \boldsymbol{x}_{t+1} = A_\text{cl} \boldsymbol{x}_t + B(K\bar x_t + \bar u_t) + G \boldsymbol{w}_t.
\end{align}
Intuitively, given a sequence of feedforward controls, a reference trajectory is induced by \eqref{eq:reference_dynamics}, and the controller in \eqref{eq:control_input} enables the robot to follow the nominal trajectory.
% We call a sequence of states $\{x_t\}$ of System~\eqref{eq:system} a \emph{trajectory}. 
% Given a feedforward signal $\{\bar u_t\}$, the dynamics \eqref{eq:closed_loop_dynamics} and \eqref{eq:reference_dynamics}, together with the distributions $P_0, P_w$ induce a unique probability distribution $\text{Pr}$ over the trajectories of the system. 
% For ease of notation we omit the dependence on the feedforward signal, and let $P_t$ denote the marginal of $\text{Pr}$ at time $t$. \IG{Maybe last part goes into the technical discussion}
% Further, for this nominal trajectory, the distributions $P_0$ and $P_w$ induce a unique probability distribution $\text{Pr}$ over the trajectories of the closed-loop system in \eqref{eq:closed_loop_dynamics}. 
We define a \textit{motion plan} for this system to be a sequence of pairs $((\bar u_t,\bar{x}_t))_{t = 0}^T$ for some $T \in \naturals_0$.

%Let $\Pi: X \rightarrow \WS$ be the operator that projects a state into the workspace. We define state space obstacles and goal as
\iffalse
% 
\begin{equation*}
    X_\text{obs} = \bigcup_{i = 1}^{n_\text{obs}} \Pi^{-1}(O_i), \qquad  X_\text{goal} = \Pi^{-1}(\mathrm{G}).
\end{equation*}
% 
\fi

Note that, given a motion plan 
$((\bar u_t,\bar{x}_t))_{t = 0}^T$,
% sequence of feedforward controls, and unlike the reference state $\bar{x}_t$, 
the evolution of system's state $\boldsymbol{x}_t$ is a stochastic process induced by distributions $P_w$ and $P_0$ and dynamics in \eqref{eq:closed_loop_dynamics}. We denote the distribution of $\boldsymbol{x}_t$ by $P_t$, i.e., $\boldsymbol{x}_t \sim P_t$. Similarly, since $\boldsymbol{u}_t$ depends on $\boldsymbol{x}_t$ through the feedback term in~\eqref{eq:control_input}, it is also a random variable. Leveraging this dependence on $\boldsymbol{x}_t$, the control constraint $\boldsymbol{u}_t \in U$ is easily embedded into the collision-avoidance constraint $\boldsymbol{x}_t \notin X_\text{obs}$ by letting
\begin{align}
\label{eq:control_constraint_embed}
    \{ x \in \reals^n : -K(x - \bar x_t ) + \bar u_t \notin U \} \subset X_\text{obs}.
\end{align}
We assume that this is the case unless stated otherwise. In turn, Expression~\eqref{eq:control_constraint_embed} implies that $X_\text{obs}$ depends, in the general case, on $\bar x_t$ and $\bar u_t$. However, to simplify notation, we denote $X_\text{obs} \equiv X_\text{obs}(\bar x_t,\bar u_t)$.

The probabilities of collision and of reaching the goal at time $t$ are given by
\begin{equation}
    \label{eq:coll-goal-probs}
    P_t\big[\boldsymbol{x}_t \in X_i\big] = \int_{\reals^n} \mathds{1}_{X_i}[y] P_t(dy), \qquad i \in \big\{\text{obs}, \text{goal}\big\}.
\end{equation}
%
\iffalse
Similarly, the probability that the control input is within the desired bounds at time $t$ is
% 
\begin{equation}
    \label{eq:control_saturation-prob}
    P_t^u\big[\boldsymbol{u}_t \in U\big] = \int_{\reals^m} \mathds{1}_{U}[v] P_t^u(dv),
\end{equation}
%
where $\boldsymbol u_t \sim P_t^u$.
\fi

% We let $X_\text{obs}^\text{ws} = \bigcup_{i = 1}^{n_\text{obs}} O_i$ and $X_\text{goal}^\text{ws} \subset \mathbb R^{n_\text{ws}}$ denote the workspace obstacles and the goal. Let $\Pi:\mathbb R^n \rightarrow \mathbb R^{n_\text{ws}}$ be the operator that projects a state into the workspace. We define the probability of collision and the probability of reaching the goal at time $t$ as
% %
% \begin{align*}
%     Pr\big[x_t \in X_i\big] = Pr\big[x_t \in \Pi^{-1}(X_i^\text{ws}) \big],
% \end{align*}
% %
% where $i \in\{\text{obs},\text{goal}\}$.

% The motion planing algorithm outputs a \emph{motion plan} $\{(\bar u_t,\bar x_t)\}_{t=0}^T$ for some \emph{planning horizon} $T\in\mathbb N$, which is guaranteed to avoid the obstacles and reach the goal with high probability:

Our objective is to generate a motion plan that guarantees that the robot avoids the obstacles, respects the control constraints, and reaches the goal with high probability.  However, since $P_w$ and $P_0$ are unknown, we need to rely on the sample trajectories $\{(\hat{\boldsymbol{x}}^{(i)}_0, \hat{\boldsymbol{x}}^{(i)}_1, \ldots, \hat{\boldsymbol{x}}^{(i)}_H)\}_{i=1}^N$ to reason about the probabilities in \eqref{eq:coll-goal-probs}.  Such reasoning can be done with some confidence related to the distribution of the sample trajectories.
% The formal statement of the problem is as follows.
We formulate this data driven, confidence-based motion planning problem as follows.
\begin{problem}[Safe Motion Planning]
\label{problem}
    % Given robotic System~\eqref{eq:system}, the controller in \eqref{eq:control_input}, the sets $X_\text{obs}, X_\text{goal}$, the supports $X_0$ and $W$ of $P_0$ and $P_w$, $N$ i.i.d. trajectories $\{\{\hat x^{(i)}_t\}\}_{i = 1}^N$ from \eqref{eq:system} and $p_\text{safe}, \beta\in(0,1)$, obtain a motion plan $\{(\bar u_t,\bar x_t)\}_{t=0}^T$ that is valid, i.e., that guarantees
    % 
    Consider System~\eqref{eq:system} under the controller in \eqref{eq:control_input} with the state space obstacle set $X_\text{obs}$ and goal set $X_\text{goal}$ where $X_\text{obs}$ satisfies \eqref{eq:control_constraint_embed}.
    Given the initial state and noise supports $X_0$ and $W$, $N$ i.i.d. sample trajectories $\{(\hat{\boldsymbol{x}}^{(i)}_0, \hat{\boldsymbol{x}}^{(i)}_1, \ldots, \hat{\boldsymbol{x}}^{(i)}_H)\}_{i = 1}^N$ of \eqref{eq:system}, a safety probability threshold $p_\text{safe} \in (0,1)$, and a confidence $\beta\in(0,1)$, generate a motion plan $((\bar u_t,\bar x_t))_{t=0}^T$ such that
    \begin{subequations}
    % \begin{equation}
    \begin{align}
        & P_t\big[\boldsymbol{x}_t \notin  X_\text{obs}\big] > p_\text{safe} && \forall t \in \{0,\dots,T\},\label{eq:cc_constraints_safety}\\
        & P_t\big[\boldsymbol{x}_T \in X_\text{goal}\big] > p_\text{safe}\label{eq:cc_constraints_goal}
    \end{align}
    % \end{equation}
    \label{eq:cc_constraints}
    \end{subequations}
    %
    % for some $T\in\mathbb N_0$ 
    with confidence $1-\beta$.
    %\IG{Why $\mathbb P$ and not $P_t$...?}\LL{It should be $P_t$}
    %\ml{control constraint $U$ is not previously discussed. It's not even immidiately clear why it's probabilistic.  Instead of putting it in the problem formulation, can put it in a remark?}
\end{problem}
%
% Note that, while the chance-constraints in \eqref{eq:cc_constraints} must hold step-wise, the confidence has to hold over the entire trajectory of System~\eqref{eq:system}.

Note that, unlike the chance constraints in \eqref{eq:cc_constraints} that need to only hold stepwise, the confidence must hold over the entire trajectory of System~\eqref{eq:system}.

\paragraph*{Overview of the Approach}

Problem~\ref{problem} is challenging due to the unknown distributions $P_0$ and $P_w$ and the non-convexity of $X_\text{obs}$ and $X_\text{goal}$. We propose a sampling-based algorithm that grows a probabilistically collision-free tree from start to goal. In Section~\ref{sec:ambiguity_tube}, we use trajectory samples to learn, at each time step, an \emph{ambiguity set} containing the unknown state distribution with high confidence; we call the resulting time sequence an \emph{ambiguity tube} and prove its soundness. To grow the tree, we check constraints~\eqref{eq:cc_constraints} at each node against the worst-case distribution in the corresponding ambiguity set
% ; the validity-checking algorithms and their analysis appear in 
(Section~\ref{sec:validity_checking}). To reduce sample complexity and improve scalability, Section~\ref{sec:lower_dimensional} shows that leveraging several lower-dimensional ambiguity tubes, which reduce conservatism; the corresponding algorithms (Section~\ref{sec:algorithm}) and analysis (Section~\ref{sec:analysis}) retain probabilistic completeness and solution guarantees while substantially lowering computational cost.  Section~\ref{sec:algorithm} presents the full motion planning algorithms, and Section~\ref{sec:analysis} proves their completeness and that every returned plan solves Problem~\ref{problem}. Finally, Section~\ref{sec:evaluation} evaluates our approach in simulation on two systems from the literature and compares it against \cite{summers2018distributionally,ekenberg2023distributionally}.
% , demonstrating its superiority.\footnote{All proofs are in Appendix~\ref{sec:appendix}.}
All proofs are in Appendix~\ref{sec:appendix}.
%\ml{check the order of sections VI and VII.  We first explain Sec. VII and then VI.} 

\begin{remark}[Arbitrary Obstacles]
    Unlike other works in which the obstacles are assumed polytopic and convex \cite{luders2010chance,summers2018distributionally}, our framework is able to handle obstacles of an arbitrary shape. This allows us to effectively find safe paths in cluttered environments and enforce control constraints without the need to employ conservative polytopic approximations. Furthermore, the collision probability that our method yields does not depend on the number of obstacles, unlike \cite{luders2010chance,summers2018distributionally}.
    % , since we treat them as a single, possibly non-convex, obstacle set. %\IG{Last part could be a remark after we formalize our algorithm. Also discuss it in the intro.}{\color{red}LL: How do you formulate the optimization problem if the sets are non-convex and possibly non-connected? Are there some assumptions we need?}\IG{We're able to solve the problem via a tree-based algorithm, which doesn't require to solve the optimization problem}
\end{remark}

\section{Construction and Dynamics of Ambiguity Sets}
\label{sec:ambiguity_tube}

% In this section, we describe how to learn an ambiguity tube for the trajectories of System~\eqref{eq:system}. In particular, we obtain a single ambiguity tube that is independent from the feedforward control and reference state, and show that the effect of the latter is easily accounted for by simply translating the learned ambiguity tube. We then detail how to construct an ambiguity sets from samples of the state $\boldsymbol x_t$. Finally, we obtain a bound on the growth of ambiguity with time, which allows us to obtain an ambiguity tube of infinite-length.

This section describes how to learn an ambiguity tube for trajectories of System~\eqref{eq:system}. Using error dynamics, we first obtain a single tube independent of the feedforward control and reference state, showing that their effect reduces to a translation of the learned tube. We then construct ambiguity sets from samples of $\boldsymbol x_t$ and bound the growth of ambiguity over time, yielding an infinite-length tube.

\subsection{Error Dynamics}

We define the \emph{state error} at time $t$ as $\boldsymbol{e}_t := \boldsymbol{x}_t - \bar x_t$, which yields the error dynamics
\begin{align}
    \label{eq:error dynamics}
    \boldsymbol{e}_{t+1} = A_\text{cl}\boldsymbol{e}_t + G \boldsymbol{w}_t.
\end{align}
Doing this allows us to express $x_t$ as the superposition of the reference state, whose dynamics are deterministic, and the error, which evolves randomly and is independent of the feedforward control and reference dynamics. The distribution $P_t^e$ of $\boldsymbol e_t$ is therefore
\begin{align}
    P_{t}^e &:= A_{\text{cl}\#}^t P_0^e* (A_\text{cl}^{t-1}G)_\#P_w * (A_\text{cl}^{t-2}G)_\#P_w * \dots * G_\#P_w \nonumber \\
    &= A_{\text{cl}\#}^t P_0^e*\Conv_{i=0}^{t-1} (A_\text{cl}^{t-1-i}G)_\#P_w,
    \label{eq: error distribution}
\end{align}
where `$*$' is the convolution operator, and with $P_0^e := P_0*\delta_{-\bar x_0}$ 
by the definition of the error. 
% Note that knowing $P_t^e$, we readily obtain $P_t$ by translating the former, i.e., $P_{t} = P_t^e*\delta_{\bar x_t}$. In the following proposition, which is an immediate consequence of the previous discussion, we describe how to easily derive an ambiguity set for the state $\boldsymbol x_t$ when given another ambiguity set for the error $\boldsymbol e_t$ and the reference state $\bar x_t$:
Note that if $P_t^e$ is known, we can obtain $P_t$ by translation: $P_{t} = P_t^e * \delta_{\bar x_t}$. The following proposition derives an ambiguity set for the state $\boldsymbol x_t$ from one for the error $\boldsymbol e_t$ and the reference state $\bar x_t$.

\iffalse
%
\begin{proposition}
\label{prop:ambiguity_tube_error}
    Let $(\mathcal{P}_t^e := \mathbb B(\widehat P_t^e, \varepsilon_t))_{t \in \mathbb N_0}$ be the \emph{error ambiguity tube} for some $\widehat P_0^e, \widehat P_1^e,\dots \in \mathcal{D}_1(\reals^n)$ and $\varepsilon_0, \varepsilon_1,\dots > 0$, such that $P_t^e \in \mathcal{P}_t^e$ for all $t \in \naturals_0$, and $(\bar x_t)_{t \in \mathbb N_0}$ be the reference state sequence. Define the \emph{state ambiguity tube} as $(\mathcal{P}_t := \mathbb B(\widehat P_t, \varepsilon_t))_{t \in \mathbb N_0}$ with $\widehat P_t := \widehat P_t^e * \delta_{\bar x_t}$ for all $t \in \naturals_0$. Then, $x_t \sim P_t \in \mathcal{P}_t$ for all $t \in \naturals_0$.
\end{proposition}
%
\begin{proof}
    Let $t \in \naturals_0$, and note that
    %
    \begin{align*}
        \mathcal{W}(P_t, \widehat P_t) = \mathcal{W}(P_t^e   * \delta_{\bar x_t}, \widehat P_t^e  * \delta_{\bar x_t}) \le \mathcal{W}(P_t^e, \widehat P_t^e) \le \varepsilon_t.
    \end{align*}
    %
    Thus $P_t \in \mathcal{P}_t$. Since $t$ was chosen arbitrarily, the previous result holds for all $t$, which concludes the proof.
\end{proof}
%
\fi

%
\begin{proposition}
\label{prop:ambiguity_tube_error}
    Let $\mathcal{P}_t^e := \mathbb B(\widehat P_t^e, \varepsilon_t)$ 
    be the \emph{error ambiguity set} at time step $t$ for some $\widehat P_t^e \in \mathcal{D}(\reals^n)$ and $\varepsilon_t > 0$, such that $P_t^e \in \mathcal{P}_t^e$. Furthermore, let $\bar x_t$ be the reference state at time $t$ and define the \emph{state ambiguity set} $\mathcal{P}_t := \mathbb B(\widehat P_t, \varepsilon_t)$, with $\widehat P_t := \widehat P_t^e * \delta_{\bar x_t}$. Then, $\boldsymbol x_t \sim P_t \in \mathcal{P}_t$.
\end{proposition}
%
% Proposition~\ref{prop:ambiguity_tube_error} implies that we only need to learn an ambiguity tube for the error, which is independent of the feedforward control, and then easily transform it into an ambiguity tube for the state via a translation by the reference state sequence. This decomposition is important for a tree-based motion planning algorithm, since it makes it easy to check whether or not a node of the tree, which corresponds to some state $\boldsymbol x_t$, is in collision with the obstacles: the collision probability is obtained simply by knowing the time step $t$ corresponding to that node, and thus the ambiguity set $\mathcal{P}_t^e$, and the reference state $\bar x_t$.
% 
Proposition~\ref{prop:ambiguity_tube_error} implies that, to obtain an ambiguity tube for the system trajectory, we only need to learn an ambiguity tube for the error $\boldsymbol{e}$ independent of the feedforward control. The tube then can be translated from $\boldsymbol{e}$ to $\boldsymbol{x}$ using the reference state sequence. This decomposition is key for efficient planning since collision-checking a node corresponding to $\boldsymbol x_t$ requires only its time $t$ (and thus $\mathcal{P}_t^e$) and the reference $\bar x_t$.

\subsection{Data-Driven Ambiguity Sets}
\label{sec:data_driven_construction_ambiguity_set}

% We now describe how to obtain an ambiguity set $\mathcal{P}_t$ for $P_t$ from samples $\{\hat{\boldsymbol x}_t^{(i)}\}_{i = 1}^N$ of $\boldsymbol{x}_t$, in such a way that the obtained $\mathcal{P}_t$ contains $P_t$ with high confidence. We refer to such sets as \emph{data-driven} ambiguity sets. For readability, we drop the subscript $t$. For such ambiguity sets there exist results \cite{fournier2015rate,boskos2023high} that guarantee that the unknown distribution belongs to the ambiguity set with user-defined confidence provided that its radius is big enough. In Lemma~\ref{lemma:data_driven_ambiguity_set}, we improve existing work by presenting a tighter bound for the case of the $1$-Wasserstein distance. First, we state the following technical lemma:
We now describe how to construct a \emph{data-driven} ambiguity set $\mathcal{P}_t$ for $P_t$ from samples $\{\hat{\boldsymbol x}_t^{(i)}\}_{i = 1}^N$ of $\boldsymbol{x}_t$ such that $P_t \in \mathcal{P}_t$ with high confidence. For readability, we drop the subscript $t$. Existing results \cite{fournier2015rate,boskos2023high} ensure containment with user-defined confidence given a sufficiently large radius; in Lemma~\ref{lemma:data_driven_ambiguity_set}, we tighten this bound for the $1$-Wasserstein distance. We first state a technical lemma.
\begin{lemma}(\cite[Thm. 1]{fournier2022convergence})
    \label{lemma:concentration_mean}
    Let $\{\hat{\boldsymbol x}^{(i)}\}_{i = 1}^N$ be a set of $N$ i.i.d. samples from a distribution $P$ over $\mathbb R^d$ with $\phi := \text{diam}(\text{supp}(P))$ and $\widehat P$ be their empirical distribution. Then, there exists a function 
    $(d, q, \mathcal{M}_q(P), N) \mapsto g(d, q, \mathcal{M}_q(P), N) \in \reals_{>0}$, 
    such that $\mathbb E[\mathcal{W}(P, \widehat P)] \le g(d, q, \mathcal{M}_q(P), N)$ and $\lim_{N \to \infty} g(d, q, \mathcal{M}_q(P), N) = 0$.
\end{lemma}
\begin{lemma}[Data-Driven Ambiguity Set]
    \label{lemma:data_driven_ambiguity_set}
    Let $\{\hat{\boldsymbol x}^{(i)}\}_{i = 1}^N$ be a set of $N$ i.i.d. samples from a distribution $P$ over $\mathbb R^d$ with $\phi := \text{diam}(\text{supp}(P))$, $\widehat P$ be their empirical distribution, and $\beta_1,\beta_2 \in(0,1)$. Let $g(d, q, \mathcal{M}_q(P), N)$ be as in Lemma~\ref{lemma:concentration_mean} with $q\in\mathbb N$
    . Define
    $\varepsilon := g(d, q, \widehat{\mathcal{M}}_q(P), N) + \phi\sqrt{\frac{\log(1/\beta_2)}{2N}}$ and
    \begin{equation}
    \label{eq:data_driven_ambiguity_set}
    \begin{split}
        \widehat{\mathcal{M}}_q(P) &:= \left( \frac{1}{N}\sum_{i=1}^N\|\hat{\boldsymbol x}^{(i)}\|^q + \bigg(\frac{\phi}{2}\bigg)^{\!\!q}\sqrt{\frac{\log(1/\beta_1)}{2N}} \right)^{\!\!1/q}\!\!\!.
    \end{split}
    \end{equation}
     Then, $P \in \mathbb B(\widehat P, \varepsilon)$ with confidence $1-(\beta_1+\beta_2)$.
\end{lemma}
\begin{remark}
    This bound is tighter than the ones on \cite{boskos2023high,gracia2024data}, as it relies on a one-sided McDiarmid inequality and leverages the $q$-th moment of the distribution instead of just its support. We also note that recent results \cite{figueiredo2025efficient} might produce tighter ambiguity sets. Given that our approach is independent of the way the ambiguity sets have been constructed, we leave incorporating the bounds of  \cite{figueiredo2025efficient} for future research.
\end{remark}

% In the next section we show how to use data-driven ambiguity sets to derive an ambiguity tube for the system's trajectories.

\subsection{Ambiguity Dynamics and Error Ambiguity Tube}

We now learn an ambiguity tube for the trajectories of $\boldsymbol{e}_t$ governed by \eqref{eq:error dynamics}. Even when the ambiguity set for $P_w$ is small, directly propagating it via \eqref{eq: error distribution} can cause it to grow by orders of magnitude \cite{boskos2023high}. A naive alternative is to build a data-driven ambiguity set at each time step from samples of $\boldsymbol e_t$, but this has three issues: (i) it requires an a priori upper bound on the planning horizon, which is unknown in practice; (ii) storing samples at every time step is infeasible for realistic horizons, especially given the data needed for tight guarantees; and (iii) Lemma~\ref{lemma:data_driven_ambiguity_set} provides only pointwise confidences, so combining them via a union bound yields a trajectory-wide confidence that shrinks as the horizon grows.

% Our approach can be seen as an intermediate one between the two previous alternatives, which combines the best of each one. We obtain data-driven ambiguity sets from state (error) observations only for some time steps. For the remaining time steps, we infer ambiguity sets from the data-driven ones taking into account the system dynamics. In this way, we obtain an ambiguity set for every time step, i.e., an ambiguity tube. 

% We refer to an ambiguity set that has been derived from a data-driven method as \emph{derived} ambiguity set. To obtain a derived ambiguity set at time step $t$, we reuse the center of some data-driven one, e.g., at time $\tau \neq t$, and quantify and account for how much the distribution of the error changes between $\tau$ and $t$. Lemma~\ref{lemma:ambiguity_dynamics} quantifies this change in the distribution of the error between two time steps:

Our approach is an intermediate one that combines the best of both approaches: we obtain data-driven ambiguity sets at only some time steps and infer the rest from them via the system dynamics, yielding an ambiguity tube. We call such an inferred set a \emph{derived} ambiguity set. To construct one at time $t$, we reuse the center of a data-driven set at some $\tau \neq t$ and quantify how the error distribution shifts between $\tau$ and $t$; Lemma~\ref{lemma:ambiguity_dynamics} gives this quantification.

%
\iffalse

In practice, we sample from some $P_\tau^e$ with $\tau$ large enough so that $\mathcal{W}(P_\tau^e, P_\infty^e) \approx 0$, and construct its empirical distribution $\widehat P_\tau^e$ and the ambiguity set $\mathbb B_{\varepsilon_\tau}(\widehat P_\tau^e)$ for $P_\tau^e$. Then, we reuse $\widehat P_\tau^e$ by defining $\widehat P_t^e := \widehat P_\tau^e$ and obtain $\varepsilon_t$ as a function of $\varepsilon_\tau$, for all $t$. 
\fi
%
\begin{lemma}[Ambiguity Dynamics]
\label{lemma:ambiguity_dynamics}
    Let the closed-loop dynamics~\eqref{eq:closed_loop_dynamics} be stable, $\tau \in \mathbb N_0$, and $P_\tau^e \in \mathbb B(\widehat P_\tau^e, \varepsilon_\tau)$ for some $\widehat P_\tau^e \in \mathcal{D}_1(\mathbb R^n)$ and $\varepsilon_\tau > 0$. Then,
    \begin{enumerate}[label=(\roman*)]
    \item $\mathcal{W}(\widehat P_\tau^e, P_t^e) \le f_\tau(t)$ for all $t \in \mathbb N_0$, with 
    % $f_\tau(t)$ given by%$f_\tau(t) := \varepsilon_\tau + \|A_\text{cl}^\tau - A_\text{cl}^t\| \mathcal{M}_p(P_0) + \mathcal{M}_p(P_w)\sum_{i=t}^{\tau-1} \|A_\text{cl}^iG\|$ if $t \le \tau$ and $f_\tau(t) := \varepsilon_\tau + \|A_\text{cl}^t - A_\text{cl}^\tau\| \mathcal{M}_p(P_0) + \mathcal{M}_p(P_w)\sum_{i=\tau}^{t-1} \|A_\text{cl}^iG\|$ otherwise.
    %\iffalse
    %
    \begin{align}
    \label{eq:ambiguity_radius}
    \!\!\!\!
    f_\tau(t) = 
        \begin{cases}
                \varepsilon_\tau + \|A_\text{cl}^\tau - A_\text{cl}^t\| \mathcal{M}_p(P_0) + & \\
                \quad \quad \mathcal{M}_p(P_w)\sum_{i=t}^{\tau-1} \|A_\text{cl}^iG\|
                & \text{if } t < \tau\\
                \varepsilon_\tau + \|A_\text{cl}^t - A_\text{cl}^\tau\| \mathcal{M}_p(P_0) + & \\
                \quad \quad \mathcal{M}_p(P_w)\sum_{i=\tau}^{t-1} \|A_\text{cl}^iG\|
                & \text{otherwise}
        \end{cases}
    \end{align}
    %\fi
    %
    %\item $\mathcal{W}(\widehat P_\tau^e, P_t^e)$ is uniformly bounded for all $t, \tau\in\naturals_0$.

    \item $\mathcal{W}(\widehat P_\tau^e, P_t^e)$ is uniformly bounded for all $t\in\naturals_0$.
    % Then, $P_t^e \in \mathbb B(\widehat P_t^e, \varepsilon_t)$.
    \end{enumerate}
\end{lemma}

Note that the $p$-th moments of $P_0$ and $P_w$ can be bounded empirically as we describe in Proposition~\ref{prop:moments}.
%\ml{Per Luca's comment, is there a need for this note?}
%
\begin{proposition}[Estimation of the $p$-th moment]
\label{prop:moments}
Let $\{\hat{\boldsymbol x}^{(i)}\}_{i=1}^N$ be $N$ i.i.d. samples from a distribution $P$, whose support has diameter $\phi$, and $\beta \in (0,1)$. We obtain via Hoeffding's inequality that
\begin{align*}
%\label{eq:bound_Lp_moment}
    \mathcal{M}_p(P) \le \frac{1}{N}\sum_{i=1}^N\|\hat x^{(i)}\|^p + \bigg(\frac{\phi}{2}\bigg)^p\sqrt{\frac{\log(1/\beta)}{2N}}
\end{align*}
holds with confidence $1-\beta$.\footnote{Note that it does not matter that this is a statistical bound, i.e., it holds with some confidence, since $P_\tau^e \in \mathbb B(\widehat P_\tau^e, \varepsilon_\tau)$ also holds with some confidence according to Lemma~\ref{lemma:data_driven_ambiguity_set}. Therefore, we can use a union bound argument merge both confidences into a single one, and thus Lemma~\ref{lemma:ambiguity_dynamics} holds with the combined confidence.} %\IG{Merge with Lemma~\ref{lemma:data_driven_ambiguity_set}.}
%{\color{red}LL: Why do you care about the moments if you already have the ambiguity sets and your goal is to compute probabilities?}\IG{We have ambiguity sets for the distribution of the state at some time steps, for example $\tau_1=1$ and $\tau_2=5$. In order to obtain the radius of the ambiguity sets at $t \in (\tau_1,\tau_5)$, we make use of Lemma 2, which requires knowing the moments}
\end{proposition}
%

% From \eqref{eq:ambiguity_radius}, we observe that $f_\tau(t)$ has a minimum at $t = \tau$, and monotonically increases with $|t - \tau|$. The intuition behind this shape is that the bigger the difference in time, the bigger the discrepancy between $\widehat P_\tau^e$ and $P_t^e$. In Section~\ref{sec:lower_dimensional}, we show how to significantly reduce the values of the quantities obtained in Lemma~\ref{lemma:ambiguity_dynamics} by employing lower dimensional ambiguity tubes, this is, tubes for the projection of the state on a lower-dimensional subspace.
%the distribution of the error, and the bigger the ambiguity set $\mathcal{P}_t$ needs to be if we want to reuse $\widehat P_\tau^e$ at time $t$.

From \eqref{eq:ambiguity_radius}, $f_\tau(t)$ has a minimum at $t = \tau$ and increases monotonically with $|t - \tau|$: the larger the time gap, the greater the discrepancy between $\widehat P_\tau^e$ and $P_t^e$. In Section~\ref{sec:lower_dimensional}, we significantly reduce these quantities using lower-dimensional ambiguity tubes, i.e., tubes for projections of the state onto a lower-dimensional subspace.

\begin{remark}
    % From Lemma~\ref{lemma:ambiguity_dynamics}, if $A_\text{cl}$ is stable, the values in $\{\mathcal{W}(\widehat P_\tau^e, P_t^e) : t \ge \tau \}$ can be made as small as desired by picking $\tau$ big enough. This implies that if we have an ambiguity set $\mathcal{P}_\tau^e = \mathbb B(\widehat P_\tau^e, \varepsilon_\tau)$ for some $\tau$ big enough, the subsequent variation in $P_t^e$ with respect to $\widehat P_\tau^e$ becomes very close to $\varepsilon_\tau$. This allows us to reuse $\widehat P_\tau^e$ as the center of the ambiguity sets $\mathcal{P}_t = \mathbb B(\widehat P_\tau^e, \varepsilon_t)$ for all $t \ge \tau$ without the need for a big $\varepsilon_t$ to account for said variation. Hence, we only need to construct a finite number of data-driven ambiguity sets, and derive the rest from these ones, making our approach suitable for problems in which the planning horizon cannot be estimated a priory. Furthermore, as we describe later, this result is also key to obtain an ambiguity tube for trajectories of arbitrary length with a confidence that is independent from its length.
    By Lemma~\ref{lemma:ambiguity_dynamics}, if $A_\text{cl}$ is stable, $\mathcal{W}(\widehat P_\tau^e, P_t^e)$ for $t \ge \tau$ can be made arbitrarily small by choosing $\tau$ large enough. Hence, given an ambiguity set $\mathcal{P}_\tau^e = \mathbb B(\widehat P_\tau^e, \varepsilon_\tau)$ for sufficiently large $\tau$, the subsequent variation of $P_t^e$ around $\widehat P_\tau^e$ stays close to $\varepsilon_\tau$, and we can reuse $\widehat P_\tau^e$ as the center of $\mathcal{P}_t = \mathbb B(\widehat P_\tau^e, \varepsilon_t)$ for all $t \ge \tau$ without inflating $\varepsilon_t$. We thus need only finitely many data-driven sets, deriving the rest, which makes the approach suitable when the planning horizon is unknown a priori. This is also key, as shown later, to obtaining an ambiguity tube of arbitrary length with confidence independent of the horizon.
\end{remark}

In Alg.~\ref{alg:tube} we formalize the construction of the ambiguity tube $(\mathcal{P}_t^e)_{t \in \naturals_0}$. Let $\tau_1, \tau_2,\dots, \tau_J \le H$ be a set
%\ml{it's not a set; it's a sequence.  If you want a set, then use $\{ \tau_i\}_{i=1}^J$}
of $J < \infty$ time steps that Alg.~\ref{alg:tube} takes as inputs. For each $\tau_j$, Lines 2-4 construct a data-driven ambiguity set $\mathcal{P}_{\tau_j}^e : =\mathbb B(\widehat P_{\tau_j}^e, \varepsilon_{\tau_j})$ from the samples of $e_{\tau_j}$, each with confidence $1-\beta/J$, as described in Section~\ref{sec:data_driven_construction_ambiguity_set}. Note that to this end we need to know $\text{diam}(\text{supp}(P_{\tau_j}^e))$, which we obtain by propagating the support of $P_0$ and $P_w$ through the dynamics \footnote{This is a well-understood problem for linear systems \cite{althoff2015introduction}}. Then, for each $t\in\mathbb N_0\setminus\{\tau_j\}_{j = 1}^J$, Lines 6-9 derive the ambiguity set $\mathcal{P}_t^e$ from the already computed data-driven ambiguity set at some $\tau \in \{\tau_j\}_{j = 1}^J$: we define $\mathcal{P}_t^e := \mathbb B(\widehat P_t^e, \varepsilon_t)$ with $\widehat P_t^e := \widehat P_\tau^e$ and $\varepsilon_t := f_\tau(t)$ as in \eqref{eq:ambiguity_radius}.
%, thus ensuring that $P_t^e \in \mathcal{P}_t^e$ with the same confidence $1-\beta/J$. 
As a result, $\varepsilon_t$ is defined as the pointwise minimum of $J$ functions where the minimum is attained at each $\tau_j$, leading to the ``sawtooth'' shape shown in Fig.~\ref{fig:sawtooth}.

%
% \begin{algorithm}
% \caption{Obtain Ambiguity Tube}\label{alg:tube}
% \begin{algorithmic}[1]
% \Require{$\mathcal{M}_p(P_w), \mathcal{M}_p(P_0), \beta, \{\tau_j\}_{j = 1}^J, A_\text{cl}, G, X_0, W,\newline \{\hat{\boldsymbol e}^{(i)}_{\tau_1} \: \hat{\boldsymbol e}^{(i)}_{\tau_2} \ldots \hat{\boldsymbol e}^{(i)}_{\tau_J} \}_{i=1}^N$}
% \Ensure{Ambiguity tube $(\mathcal{P}_t^e)_{t \in \mathbb N_0}$}
% \State Construct $\mathbb B(\widehat P_{\tau_j}^e, \varepsilon_{\tau_j})$ from $\{\hat{\boldsymbol e}^{(i)}_{\tau_j}\}_{i=1}^N$ and with $\varepsilon_{\tau_j}$ as in Lemma~\ref{lemma:data_driven_ambiguity_set} with confidence $1 - \beta/J$ for all $j \in \{1, \dots, J\}$
% \State $\mathcal{P}_{\tau_j}^e \gets \mathbb B(\widehat P_{\tau_j}^e, \varepsilon_{\tau_j})$ for all $j \in \{1, \dots, J\}$
% %
% \For{$t \in \mathbb N_0\setminus\{\tau_j\}_{j=1}^J$}
% \State $\tau \gets \arg\min\big\{ f_{\tau_j}(t) : j \in \{1,\dots, J\} \big\}$ with $f_{\tau_j}(t)$ as given in \eqref{eq:ambiguity_radius}
% \State $\widehat P_{t}^e \gets \widehat P_\tau^e$
% \State $\varepsilon_t \gets f_\tau(t)$
% \State $\mathcal{P}_t^e \gets \mathbb B(\widehat P_t^e, \varepsilon_t)$
% \EndFor
% %
% \end{algorithmic}
% \end{algorithm}
%
\begin{algorithm}
\caption{Construct Ambiguity Tube}
\label{alg:tube}
\KwIn{$\mathcal{M}_p(P_w), \mathcal{M}_p(P_0), \beta, \{\tau_j\}_{j = 1}^J, A_\text{cl}, G, X_0, W,$
$\{\hat{\boldsymbol e}^{(i)}_{\tau_1}, \hat{\boldsymbol e}^{(i)}_{\tau_2}, \ldots, \hat{\boldsymbol e}^{(i)}_{\tau_J} \}_{i=1}^N$}
\KwOut{Ambiguity tube $(\mathcal{P}_t^e)_{t \in \mathbb{N}_0}$}

\For{$j \in \{1, \dots, J\}$}{
    Construct $\mathbb{B}(\widehat{P}_{\tau_j}^e, \varepsilon_{\tau_j})$ from $\{\hat{\boldsymbol{e}}^{(i)}_{\tau_j}\}_{i=1}^N$ 
    with $\varepsilon_{\tau_j}$ as per Lemma~\ref{lemma:data_driven_ambiguity_set}, with confidence $1 - \beta/J$\;
    $\mathcal{P}_{\tau_j}^e \gets \mathbb{B}(\widehat{P}_{\tau_j}^e, \varepsilon_{\tau_j})$\;
}

\For{$t \in \mathbb{N}_0 \setminus \{\tau_j\}_{j=1}^J$}{
    $\tau \gets \arg\min\big\{ f_{\tau_j}(t) : j \in \{1,\dots, J\} \big\}$\;
    $\widehat{P}_t^e \gets \widehat{P}_\tau^e$\;
    $\varepsilon_t \gets f_\tau(t)$\;
    $\mathcal{P}_t^e \gets \mathbb{B}(\widehat{P}_t^e, \varepsilon_t)$\;
}

\end{algorithm}

In practice, we fix $J$ and choose $\{\tau_j\}_{j = 1}^J$ so that the upper bound on $(\varepsilon_t)_{t\in\naturals_0}$ is as small as possible. 
\begin{figure}
    \centering
    \includegraphics[width=0.7\linewidth]{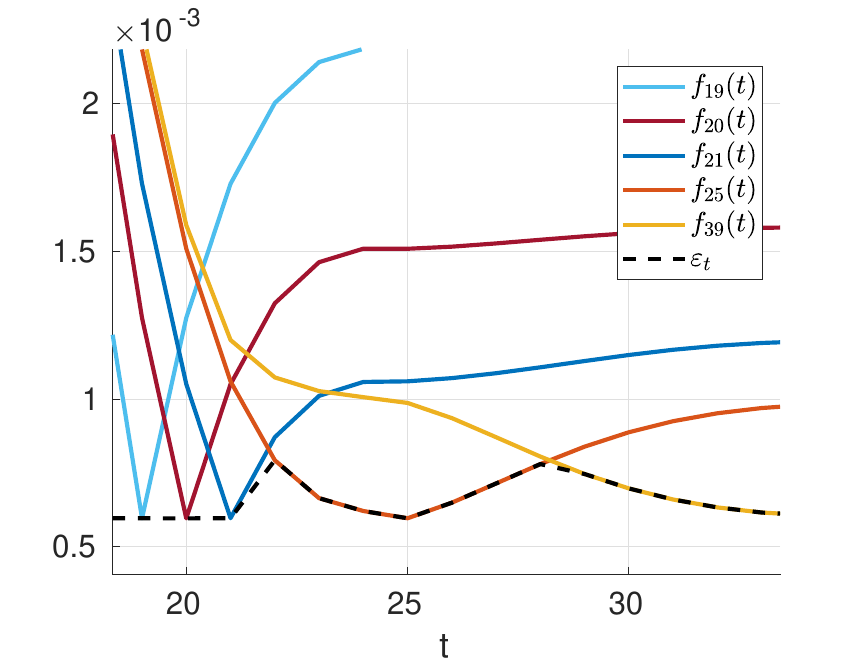}
    \caption{
    Visual depiction of $\varepsilon_t$ as the pointwise minimum of the functions $\{f_{\tau_j}(\cdot):j\in\{1,\dots, J\}\}$. In this case, the data-driven ambiguity sets are computed for $t \in \{19,20,21,25,39\}$, leading to the five solid lines, whereas the ambiguity sets for all other time steps $t \ge 19$ are derived from the former. It is observed that for $\tau_j$ big enough, $f_{\tau_j}(\cdot)$ grows slowly with $t$, thus the other values $\tau_i$ for $i\neq j$ need not be close to $\tau_j$ to guarantee that $\varepsilon_t$ remains small. In fact, only two values of $\tau_{J-1} = 25$ and $\tau_{J} = 39$ are used to obtain $\varepsilon_t$ for all $t \ge 22$.
    }
    \label{fig:sawtooth}
\end{figure}

The following theorem ensures that the constructed ambiguity tube contains the trajectories of the error's distribution.
\begin{theorem}[Soundness of the Ambiguity Tube]
\label{thm:ambiguity_tube}
    The ambiguity tube $(\mathcal{P}_t^e)_{t \in \mathbb N_0}$ obtained from Alg.~\ref{alg:tube} contains the trajectory of the distribution of the error with confidence $1-\beta$, i.e.,
    %
    % \begin{align*}
    %     P_t^e \in \mathcal{P}_t^e \qquad \forall\:t\in\mathbb N_0
    % \end{align*}
    %
    $P_t^e \in \mathcal{P}_t^e$ for all $t\in\naturals_0$
    with confidence $1-\beta$.
\end{theorem}
We illustrate  Theorem~\ref{thm:ambiguity_tube} in Fig.~\ref{fig:ambiguity_tube}.
\begin{figure}
    \centering
    \includegraphics[width=0.6\linewidth]{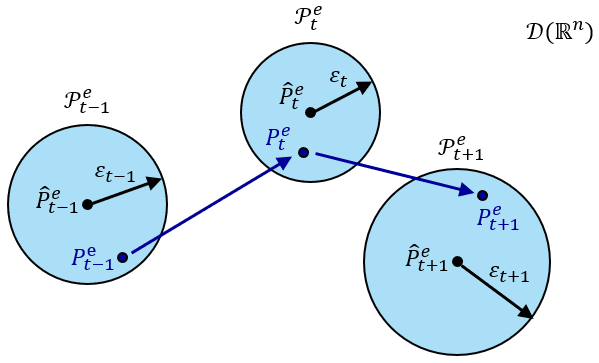}
    \caption{Visual interpretation of Theorem~\ref{thm:ambiguity_tube}. It can be observed that the ambiguity tube is sound, i.e., the ambiguity set $\mathcal{P}_t^e$ contains the true distribution $P_t^e$ at every time step $t$.
    }
    \label{fig:ambiguity_tube}
\end{figure}

\begin{remark}
    Theorem~\ref{thm:ambiguity_tube} guarantees that the confidence on the tube containing the entire trajectories of the distribution of the error does not depend on the planning horizon, but only on $J$. This makes it possible to obtain a tube of infinite length while keeping the confidence finite.
\end{remark}

It is easy to observe that as the number of samples goes to infinity, each ambiguity set $\mathcal{P}_t^e$ shrinks to the singleton $\{P_t^e\}$ with confidence $1-\beta$, thus eliminating the conservativeness of our approach. This is an advantage of using Wasserstein ambiguity sets over, for example, moment ambiguity sets  \cite{mohajerin2018data}.

\begin{remark}[Clustering Empirical Distributions]
\label{rem:cluster}
    By Lemma~\ref{lemma:data_driven_ambiguity_set}, the number of samples $N$ needed for $\mathcal{P}_t^e = \mathbb B(\widehat P_t^e, \varepsilon_t)$ to contain $P_t^e$ is inversely proportional to $\varepsilon_t$, so tight ambiguity sets require many samples, which are costly to store and compute with. To alleviate this, we apply scenario reduction \cite{rujeerapaiboon2022scenario} to cluster the samples into a weighted set $\{(\hat{\boldsymbol x}^{(i)}, \boldsymbol a_i)\}_{i = 1}^{N_c}$ and redefine the center of the ambiguity set as the discrete distribution
    % \begin{align*}
    %     \widehat P_t^{e,c} := \sum_{i = 1}^{N_c} \boldsymbol a_i \delta_{\boldsymbol{\hat x}^{(i)}}.
    % \end{align*}
    $\widehat P_t^{e,c} := \sum_{i = 1}^{N_c} \boldsymbol a_i \delta_{\boldsymbol{\hat x}^{(i)}}.$
    Quantifying the discrepancy $\mathcal{W}(\widehat P_t^e, \widehat P_t^{e,c})$ between the empirical and clustered distributions and adding it to the radius, the triangle inequality yields
    \begin{align*}
        P_t^e \in \mathcal{P}_t^{e,c} := \mathbb B(\widehat P_t^{e,c}, \varepsilon_t + \mathcal{W}(\widehat P_t^e, \widehat P_t^{e,c})).
    \end{align*}
    We use k-means clustering \cite{lloyd1982least}, which minimizes the $2$-Wasserstein distance between $\widehat P_t^e$ and $\widehat P_t^{e,c}$ \cite{rujeerapaiboon2022scenario}.
\end{remark}

% To simplify the notation, in the rest of this paper, we assume that all data-driven ambiguity sets have been clustered without explicitly indicating so.

For notational simplicity, we henceforth assume that all data-driven ambiguity sets are already clustered.
\section{Validity Check}
\label{sec:validity_checking}

% As discussed in the preliminary Section~\ref{sec:treebasedplanners}, there are three main steps in sampling-based motion planning: select, extend (propagate), and validity check (see Alg.~\ref{alg:treebasedplanners}). While in Section~\ref{sec:ambiguity_tube} we obtained an ambiguity tube for the robot's trajectory, here, In this section we formulate the validity-checking algorithm, which leverages the ambiguity tube to decide whether state $\boldsymbol{x}_t$ satisfies the chance constraints of \eqref{eq:cc_constraints_safety}. We begin by formulating Alg.~\ref{alg:collision}; the basic validity checking algorithm. Then, in Subsection~\ref{sec:lazy} we propose a more conservative, but significantly faster validity checking algorithm, and in Subsection~\ref{sec:combined_validity_checkers} we propose to combine both to achieve good empirical performance.

As discussed in Section~\ref{sec:treebasedplanners}, sampling-based motion planning has three main steps: select, extend, and validity check (see Alg.~\ref{alg:treebasedplanners}). Having obtained an ambiguity tube in Section~\ref{sec:ambiguity_tube}, we now use it to decide whether $\boldsymbol{x}_t$ satisfies the chance constraints \eqref{eq:cc_constraints_safety}. We first formulate the basic validity checker in Alg.~\ref{alg:collision}; Subsection~\ref{sec:lazy} presents a more conservative but significantly faster variant, and Subsection~\ref{sec:combined_validity_checkers} combines the two for strong empirical performance.

% The basic validity checking method follows this logic: first, given the ambiguity set $\mathcal{P}^e_t$ for $\boldsymbol{e}_t$ and the reference state $\bar x_t$, we obtain the state ambiguity set $\mathcal{P}_t$ as described in Proposition~\ref{prop:ambiguity_tube_error}.
%
%
% Next, we determine whether or not the robot is not in collision with the obstacles by checking the distributionally robust inequality:
% %
% \begin{align}
% \label{eq:dr_cc_obs}
%     \min_{P\in \mathbb B(\widehat P_t, \varepsilon_t)} P\big[ \boldsymbol x_t \notin X_\text{obs} \big] > p_\text{safe}.
% \end{align}
% %
% Since, by Thm.~\ref{thm:ambiguity_tube} $P_t^e \in \mathcal{P}_t^e$, it is evident that the previous inequality implies that the robot is not in collision. Similarly, we conclude that the robot has reached the goal if
% %
% \begin{align}
% \label{eq:dr_cc_goal}
%     \min_{P\in \mathbb B(\widehat P_t, \varepsilon_t)} P\big[ \boldsymbol x_t \in X_\text{goal} \big] > p_\text{safe}.
% \end{align}
%

The basic validity checker proceeds as follows. Given $\mathcal{P}^e_t$ and the reference state $\bar x_t$, we form the state ambiguity set $\mathcal{P}_t$ via Proposition~\ref{prop:ambiguity_tube_error} and check collision chance constraint by checking the distributionally robust inequality
\begin{align}
\label{eq:dr_cc_obs}
    \min_{P\in \mathbb B(\widehat P_t, \varepsilon_t)} P\big[ \boldsymbol x_t \notin X_\text{obs} \big] > p_\text{safe}.
\end{align}
Since $P_t^e \in \mathcal{P}_t^e$ by Theorem~\ref{thm:ambiguity_tube}, \eqref{eq:dr_cc_obs} implies the robot is collision-free. Goal condition is checked analogously via
\begin{align}
\label{eq:dr_cc_goal}
    \min_{P\in \mathbb B(\widehat P_t, \varepsilon_t)} P\big[ \boldsymbol x_t \in X_\text{goal} \big] > p_\text{safe}.
\end{align}

\iffalse
Note that the problems of finding the worst-case probabilities in \eqref{eq:dr_cc_obs} and \eqref{eq:dr_cc_goal} are equivalent, since we can write
%
\begin{align*}
    \min_{P\in \mathbb B(\widehat P_t, \varepsilon_t)} P\big[ \boldsymbol x_t \in X_\text{goal} \big] = 1 - \max_{P\in \mathbb B(\widehat P_t, \varepsilon_t)} P\big[ \boldsymbol x_t \in X_\text{goal}\setminus\reals^n \big].
\end{align*}
%
Therefore,
\fi

Due to the similarities in \eqref{eq:dr_cc_obs} and \eqref{eq:dr_cc_goal}, we only describe the algorithm to check Condition~\eqref{eq:dr_cc_obs}. Adapting it for the goal check is straightforward. 
% 
% When the center of the ambiguity set is an empirical distribution, a fast method to solve the uncertainty quantification Problem~\eqref{eq:dr_cc_obs} is provided in \cite[Example 7]{gao2023distributionally}. Said method does not involve solving expensive optimization problems, relying instead on computing the distances between the points in the support of the empirical distribution and the set. Note that when the center of the ambiguity set, $\widehat P_t$, has been clustered to reduce the number of points in its support (see Remark~\ref{rem:cluster}), it is no longer an empirical distribution, but a discrete distribution with rational weights $\boldsymbol a_i \in [0,1]$. As such, we can interpret it as an empirical distribution where some points are repeated, enabling us to employ the same approach when $\widehat P_t$, has been clustered.
% 
When the center of the ambiguity set is an empirical distribution, \cite[Example 7]{gao2023distributionally} provides a fast, optimization-free method based on the distances between the support points and the set. After clustering to reduce the support size (see Remark~\ref{rem:cluster}), $\widehat P_t$ is no longer empirical but a discrete distribution with rational weights $\boldsymbol a_i \in [0,1]$. Interpreting it as an empirical distribution with repeated points lets us apply the same method.

The uncertainty quantification method relies on the fact that there exists a finitely-supported worst-case distribution $P^*_t$ that attains the minimum in \eqref{eq:dr_cc_obs}, i.e.,
\begin{align*}
    P^*_t\big[ \boldsymbol x_t \notin X_\text{obs} \big] = \min_{P\in \mathbb B(\widehat P_t,\varepsilon_t)} P \big[ \boldsymbol x_t \notin X_\text{obs} \big],
\end{align*}
and which is easily constructed by transporting mass from $\widehat P_t$ to $X_\text{obs}$ in a greedy fashion, i.e., transporting first the mass from the atoms of $\widehat P_t$ that are closer to $X_\text{obs}$. 
The method starts by computing $\{\text{dist}(\hat{\boldsymbol x}^{(i)}, X_\text{obs})\}_{i = 1}^{N_c}$ and sorting $\{\hat{\boldsymbol x}^{(i)}\}_{i = 1}^{N_c}$ according to $\text{dist}(\hat{\boldsymbol x}^{(i)}, X_\text{obs})$, in an increasing fashion. Let $i_0-1$ be the maximum number of atoms that can be transported to $X_\text{obs}$ while satisfying the transport cost constraint $\mathcal{W}(\widehat P_t, P^*_t) \le \varepsilon_t$, and $m_0$ be the maximum mass from $\hat{\boldsymbol x}^{(i_0)}_t$that we can transport to $X_\text{obs}$ without violating this constraint, i.e.,
$i_0 := \arg\max \left\{j \le {N_c} : \sum_{i = 1}^{j - 1} \boldsymbol a_i \text{dist}(\hat{\boldsymbol x}^{(i)}, X_\text{obs}) \le \varepsilon_t \right\}$,
\begin{equation}
\begin{split}
\label{eq:collision_probability0}
    % i_0 &:= \arg\max \bigg\{j \le {N_c} : \sum_{i = 1}^{j - 1} \boldsymbol a_i \text{dist}(\hat{\boldsymbol x}^{(i)}, X_\text{obs}) \le \varepsilon_t \bigg\}\\
    m_0 &:= \frac{\varepsilon_t - \sum_{i = 1}^{i_0 - 1} \boldsymbol a_i \text{dist}(\hat{\boldsymbol x}^{(i)}, X_\text{obs})}{\text{dist}(\hat{\boldsymbol x}^{(i_0)}_t, X_\text{obs})}.
    \end{split}
\end{equation}
The worst-case distribution $P^*_t$ is obtained as the one supported on the atoms after the perturbation, which yields
\begin{align}
\label{eq:collision_probability}
    \min_{P\in \mathbb B(\widehat P_t,\varepsilon_t)} P \big[\boldsymbol x_t \notin X_\text{obs}\big] = P^*_t \big[\boldsymbol x_t \notin X_\text{obs}\big] = \sum_{i = i_0}^{N_c} \boldsymbol a_i -  m_0.
\end{align}
If this probability is higher than $p_\text{safe}$, we conclude that $\boldsymbol{x}_t$ is not in collision. Otherwise, the state is denoted valid. The pseudocode for the validity checking algorithm is given in Alg.~\ref{alg:collision} and depicted in Fig.~\ref{fig:dr_uq}.
%\ml{give a quick overview of the steps in the alg.}\IG{is it necessary? We already explained the steps in english and the pseudocode is fairly simple}
%
% \begin{algorithm}
% \caption{Validity Checking Via Probability Mass Transport}\label{alg:collision}
% \begin{algorithmic}[1]
% \Require{$X_\text{obs}, (\mathcal{P}_t)_{t \in \mathbb N_0}, p_\text{safe}, t, \bar x_t$}
% \Ensure{$\text{isvalid}$}
% \State Obtain $\mathcal{P}_t = \mathbb B(\widehat P_t, \varepsilon_t)$ via Proposition~\ref{prop:ambiguity_tube_error}
% \State Obtain $\{(\hat{\boldsymbol x}^{(i)}, \boldsymbol a_i)\}_{i = 1}^{N_c}$ from $\widehat P_t$
% \State Compute $\{\text{dist}(\hat{\boldsymbol x}^{(i)}, X_\text{obs})\}_{i = 1}^{N_c}$
% \State Sort $\{\hat{\boldsymbol x}^{(i)}\}_{i=1}^{N_c}$ according to $\text{dist}(\hat{\boldsymbol x}^{(i)}, X_\text{obs})$ in an increasing fashion
% \State Obtain $P^*_t\big[ \boldsymbol x_t \notin X_\text{obs} \big]$ via \eqref{eq:collision_probability0} and
% \eqref{eq:collision_probability}
% \If{$P^*_t\big[ \boldsymbol x_t \notin X_\text{obs} \big] > p_\text{safe}$}
% \State $\text{isvalid} \gets \text{true}$
% \Else
% \State $\text{isvalid} \gets \text{false}$
% \EndIf
% %
% \end{algorithmic}
% \end{algorithm}
%
\begin{algorithm}[t]
\caption{Validity Checking Via Probability Mass Transport}
\label{alg:collision}
\KwIn{$X_\text{obs}, (\mathcal{P}_t^e)_{t \in \mathbb{N}_0}, p_\text{safe}, t, \bar{x}_t$}
\KwOut{$\text{isvalid}$}

Obtain $\mathcal{P}_t = \mathbb{B}(\widehat{P}_t, \varepsilon_t)$ via Proposition~\ref{prop:ambiguity_tube_error}\;

Obtain $\{(\hat{\boldsymbol{x}}^{(i)}, \boldsymbol{a}_i)\}_{i = 1}^{N_c}$ from $\widehat{P}_t$\;

Compute $\{\text{dist}(\hat{\boldsymbol{x}}^{(i)}, X_\text{obs})\}_{i = 1}^{N_c}$\;

Sort $\{\hat{\boldsymbol{x}}^{(i)}\}_{i=1}^{N_c}$ according to $\text{dist}(\hat{\boldsymbol{x}}^{(i)}, X_\text{obs})$ in increasing order\;

Obtain $P^*_t\left[ \boldsymbol{x}_t \notin X_\text{obs} \right]$ via \eqref{eq:collision_probability0} and \eqref{eq:collision_probability}\;

\If{$P^*_t\left[ \boldsymbol{x}_t \notin X_\text{obs} \right] > p_\text{safe}$}{
    $\text{isvalid} \gets \text{true}$\;
}
\Else{
    $\text{isvalid} \gets \text{false}$\;
}

\end{algorithm}

\begin{figure}[t]
    \centering
    \begin{subfigure}{0.46\columnwidth}
        \includegraphics[width=\linewidth]{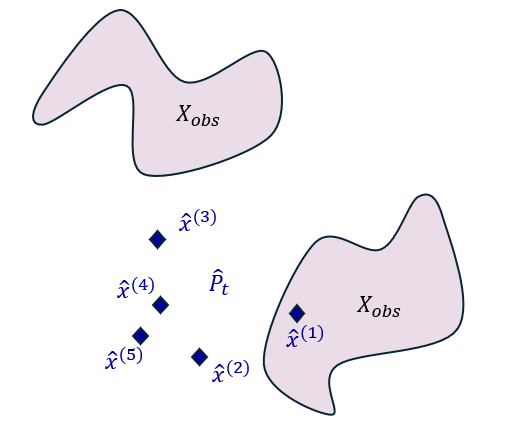}
        \caption{}
    \end{subfigure}
    ~~~
    \begin{subfigure}{0.43\columnwidth}
        \includegraphics[width=\linewidth]{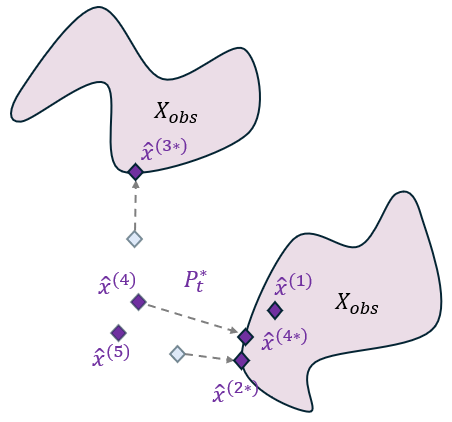}
        \caption{}
    \end{subfigure}
    \caption{Graphical explanation of Alg.~\ref{alg:collision}. On the left, the the nominal distribution $\widehat P_t$ and the atoms $\{\hat{x}^{(1)},\dots, \hat{x}^{(5)}\}$ in its support. On the right, the worst-case distribution $P^*$, supported on $\{\hat{x}^{(1)},\hat{x}^{(2*)},\hat{x}^{(3*)},\hat{x}^{(4)},\hat{x}^{(4*)}, \hat{x}^{(5)}\}$, and which minimizes the probability of $\boldsymbol x_t$ belonging to $X_\text{obs}$. In this example $\hat x^{(1)}\in X_\text{obs}$, so it is not perturbed, and $\varepsilon_t$ is such that it allows to fully transport $\hat x^{(2)}$ and $\hat x^{(3)}$ to the boundary of $X_\text{obs}$, but only the fraction $m_0$ of mass from $\hat x^{(4)}$, and no mass from $\hat x^{(5)}$. Note that the set $X_\text{obs}$ has arbitrary shape.}
\label{fig:dr_uq}
\end{figure}

\iffalse
\begin{figure}[h]

\begin{subfigure}{0.5\textwidth}
\includegraphics[width=0.9\linewidth]{dr_uq_1.png} 
\caption{Nominal distribution}
\label{fig:dr_uq_1}
\end{subfigure}
\begin{subfigure}{0.5\textwidth}
\includegraphics[width=0.9\linewidth]{dr_uq_2.png}
\caption{Worst-case distribution}
\label{fig:dr_uq_2}
\end{subfigure}

\caption{Graphical explanation of Alg.~\ref{alg:collision}. On the left, the the nominal distribution $\widehat P_t$ and the atoms $\{\hat{x}^{(1)},\dots, \hat{x}^{(5)}\}$ in its support. On the right, the worst-case distribution $P^*$, supported on $\{\hat{x}^{(1)},\hat{x}^{(2*)},\hat{x}^{(3*)},\hat{x}^{(4)},\hat{x}^{(4*)}, \hat{x}^{(5)}\}$, and which minimizes the probability of $\boldsymbol x_t$ belonging to $X_\text{obs}$. In this example $\hat x^{(1)}\in X_\text{obs}$, so it is not perturbed, and $\varepsilon_t$ is such that it allows to fully transport $\hat x^{(2)}$ and $\hat x^{(3)}$ to the boundary of $X_\text{obs}$, but only the fraction $m_0$ of mass from $\hat x^{(4)}$, and no mass from $\hat x^{(5)}$. Note that the set $X_\text{obs}$ has arbitrary shape.}
\label{fig:dr_uq}
\end{figure}
\fi

Note that Alg.~\ref{alg:collision} computes the exact worst-case probability of safety w.r.t. all distributions in the ambiguity set, instead of obtaining a lower bound. Furthermore, it is optimization-free, and only requires computing $N_c$ distances, making it relatively fast. Additionally, since it scales with $N_c$, namely, the number of atoms in $\widehat P_t$, clustering these samples can greatly improve the efficiency of the algorithm.
%Furthermore, when the center of the ambiguity set is an empirical distribution, the number of variables in the resulting optimization problem is proportional to $N_c$ \cite{aolaritei2022uncertainty,hakobyan2020wasserstein}. Moreover, when using other ambiguity sets like moment ambiguity \cite{summers2018distributionally}, a fast collision checking algorithm can be used, but its result is an overapproximation of the probability of unsafety.

\subsection{Lazy Validity Check}
\label{sec:lazy}

% Here, we propose a way to obtain an upper bound in the worst-case probability of collision in \ref{eq:dr_cc_obs}. This is faster than Alg.~\ref{alg:collision}, which solves the uncertainty quantification problem~\ref{eq:dr_cc_obs} exactly. If this upper bound is smaller than $p_\text{safe}$, we conclude that the robot is not in collision with that obstacle. Otherwise, we make use of Alg.~\ref{alg:collision}. Making use of the method proposed in this section to determine if the goal has been reached is straightforward, and we omit that discussion for the sake of brevity.

We propose an upper bound on the worst-case collision probability in \eqref{eq:dr_cc_obs} that is faster to compute than the exact solution of Alg.~\ref{alg:collision}. If the bound is below $p_\text{safe}$, we conclude the robot is collision-free; otherwise, we fall back on Alg.~\ref{alg:collision}. 
% The corresponding goal check is analogous and omitted for brevity.

The algorithm is as follows. We rely on having computed, in an offline fashion, a sequence of sets $(S^e_t)_{t\in\naturals_0}$ such that
\begin{align}
\label{eq:confidence_balls}
    \min_{P\in\mathcal{P}_t} P\big[ \boldsymbol x_t \in S^e_t\big] > p_\text{safe}, \quad\text{for all}\: t\in\naturals_0.
\end{align}
For simplicity, we restrict ourselves to balls centered at the origin. In fact, we approximately compute the smallest balls that satisfy \eqref{eq:confidence_balls} by optimizing their radius via the bisection algorithm. Then, at time-step $t$ of the planning phase, if the ball $S^e_t + \bar x_t$ does not intersect with any obstacle, we conclude that $x_t$ is not in collision. %Otherwise, we compute the exact probability of collision via Alg.~\ref{alg:collision} (\IG{erase this after bandit algorithm is included}).
We denote the tube $(S^e_t)_{t\in\naturals_0}$ as \emph{confidence tube}, since it contains more than $p_\text{safe}$ probability mass from the trajectories of $P_t^e$ with high confidence, as we state in Lemma~\ref{lemma:confidence_balls}. The pseudocode for obtaining the confidence tube is described in Alg.~\ref{alg:confidence_balls}.
%
% \begin{algorithm}
% \caption{Obtain Confidence Tube}\label{alg:confidence_balls}
% \begin{algorithmic}[1]
% \Require{$\mathcal{M}_p(P_w), \mathcal{M}_p(P_0), \beta,  \{\tau_j\}_{j = 1}^J, (\mathcal{P}_t^e = \mathbb B(\widehat P_t^e, \varepsilon_t))_{t \in \naturals_0}$ from Alg.~\ref{alg:tube}, $A_\text{cl}, G$}
% \Ensure{Confidence tube $(S^e_t)_{t \in \mathbb N_0}$}
% %
% \For{$j \in \{1,\dots, J\}$}
% \State $I_j \gets \big\{t \in \naturals_0 : \nexists k \in \{1,\dots, J\} \:\:\text{s.t.}\: f_{\tau_k}(t) < f_{\tau_j}(t) \big\}$
% \State $\bar \varepsilon_{\tau_j} \gets \max\{ f_{\tau_j}(t) : t \in I_j\}$ 
% \State Obtain $S^e_{\tau_j}$ s.t. $\min\big\{P\big[ \boldsymbol e_t \in S^e_{\tau_j}\big] : P\in\mathbb B(\widehat P_{\tau_j}^e, \bar \varepsilon_{\tau_j}) \big\} > p_\text{safe}$
% \State $S^e_t \gets S^e_{\tau_j}$ for all $t \in I_j$
% \EndFor
% %

% \end{algorithmic}
% \end{algorithm}

\begin{algorithm}
\caption{Obtain Confidence Tube}
\label{alg:confidence_balls}
\KwIn{$\mathcal{M}_p(P_w), \mathcal{M}_p(P_0), \beta, \{\tau_j\}_{j = 1}^J$, $(\mathcal{P}_t^e = \mathbb{B}(\widehat{P}_t^e, \varepsilon_t))_{t \in \mathbb{N}_0}$ from Alg~\ref{alg:tube},
$A_\text{cl}, G$}
\KwOut{Confidence tube $(S^e_t)_{t \in \mathbb{N}_0}$}

\For{$j \in \{1,\dots, J\}$}{
    $I_j \gets \big\{t \in \mathbb{N}_0 : \nexists k \in \{1,\dots, J\} \text{ s.t. } f_{\tau_k}(t) < f_{\tau_j}(t) \big\}$\;

    $\bar{\varepsilon}_{\tau_j} \gets \max\{ f_{\tau_j}(t) : t \in I_j \}$\;

    Obtain $S^e_{\tau_j}$ such that\newline
    $\min\big\{ P\big[ \boldsymbol{e}_t \in S^e_{\tau_j} \big] : P \in \mathbb{B}(\widehat{P}_{\tau_j}^e, \bar{\varepsilon}_{\tau_j}) \big\} > p_\text{safe}$\;
    $S^e_t \gets S^e_{\tau_j}$ for all $t \in I_j$\;
}
\end{algorithm}
\begin{algorithm}
\caption{Lazy Validity Checking}
\label{alg:lazy_check}
\KwIn{$X_\text{obs}, (S^e_t)_{t \in \mathbb{N}_0}, t, \bar{x}_t$}
\KwOut{$\text{isvalid}$}

\If{$X_\text{obs} \cap (S^e_t + \bar{x}_t) = \emptyset$}{
    $\text{isvalid} \gets \text{true}$\;
}
\Else{
    $\text{isvalid} \gets \text{false}$\;
}
\end{algorithm}

\begin{lemma}[Soundness of the Confidence Tube]
\label{lemma:confidence_balls}
    Let
    %$(\mathcal{P}_t)_{t\in\naturals_0}$ be the ambiguity tube obtained from Alg.~\ref{alg:tube}, and
    $(S^e_t)_{t\in\naturals_0}$ be the confidence tube obtained from Alg.~\ref{alg:confidence_balls}. Then, $P_t^e[\boldsymbol e_t \in S^e_t] > p_\text{safe}$ for all $t\in\naturals_0$ with confidence $1-\beta$.
\end{lemma}
Lemma~\ref{lemma:confidence_balls} is illustrated in Fig.~\ref{fig:confidence_ball}.
\begin{figure}
    \centering
    \begin{subfigure}{0.39\columnwidth}
        \includegraphics[width=\linewidth]{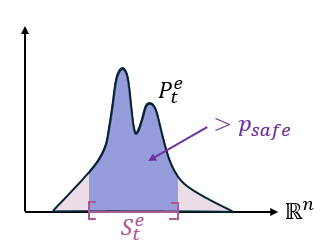} 
        \caption{Depiction of Lemma~\ref{lemma:confidence_balls}.}
        \label{fig:confidence_ball}
    \end{subfigure}
    \begin{subfigure}{0.59\columnwidth}
        \includegraphics[width=\linewidth]{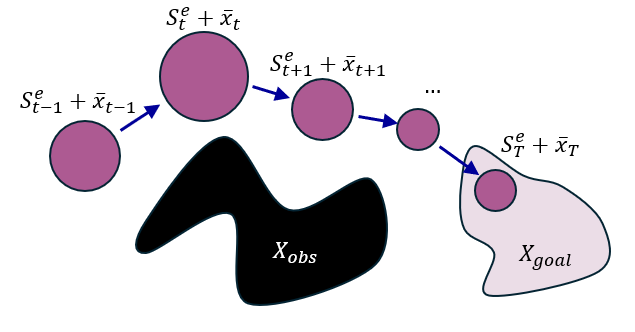}
        \caption{Planning using the confidence tube. 
        % Alg.~\ref{alg:lazy_check} determines that the trajectory is safe at eveery time step and reaches the goal at $t=T$ because the confidence tube do not intersect $X_\text{obs}$ and the ball for time step $T$ is completely contained in $X_\text{goal}$.
        }
        \label{fig:confidence_balls}
    \end{subfigure}
    \caption{Planning with lazy validity checker. (a) Confidence region at time $t$ per Lemma~\ref{lemma:confidence_balls}.
    (b) Illustration of a trajectory that is safe at every time step and reaches the goal at $t=T$ since the confidence tube does not intersect $X_\text{obs}$ and the ball for time step $T$ is completely contained in $X_\text{goal}$}
\end{figure}
Note that although the sequence $(S_t^e)_{t\in\naturals_0}$ has infinite length, we avoid having to compute an infinite number of confidence regions by leveraging Lemma~\ref{lemma:ambiguity_dynamics}. The pseudocode for the lazy check is given in Alg.~\ref{alg:lazy_check}. Additionally, in Fig.~\ref{fig:confidence_balls} we show an example of a trajectory that is valid at every time step and that reaches the goal at $t=T$, all according to Alg.~\ref{alg:lazy_check}.

%
% \begin{algorithm}
% \caption{Lazy Validity Checking}\label{alg:lazy_check}
% \begin{algorithmic}[1]
% \Require{$X_\text{obs}, (S^e_t)_{t \in \mathbb N_0}, t, \bar x_t$}
% \Ensure{$\text{isvalid}$}
% \If{$X_\text{obs}\bigcap\big( S^e_t + \bar x_t \big)  = \emptyset$}
% \State $\text{isvalid} \gets \text{true}$
% \Else
% \State $\text{isvalid} \gets \text{false}$
% \EndIf
% %
% \end{algorithmic}
% \end{algorithm}
%

\subsection{Hybrid Validity Checkers}
\label{sec:combined_validity_checkers}

% Whether to perform validity checking via probability mass transport (Alg.~\ref{alg:collision}) or via the lazy check in Alg.~\ref{alg:lazy_check}, it naturally gives rise to conflicting trade-offs between efficiency and accuracy. The lazy validity checker is faster but more conservative than Alg.~\ref{alg:collision}. Only using Alg.~\ref{alg:lazy_check} may lead to faster solution times for ``simpler" problems but may suffer for problems which require less conservatism. Therefore, we propose two validity checking algorithms that automatically combine both methods to improve efficiency while maintaining accuracy.

% Whether validity checking is performed via probability mass transport (Alg.~\ref{alg:collision}) or with the lazy check (Alg.~\ref{alg:lazy_check}), a natural trade-off emerges between efficiency and accuracy. The lazy procedure is faster but more conservative than Alg.~\ref{alg:collision}. Relying solely on Alg.~\ref{alg:lazy_check} can yield shorter solution times for “simpler’’ problems, but it may underperform in scenarios that demand reduced conservatism. To address this, we propose two hybrid validity-checking algorithms that automatically integrate both methods, achieving improved efficiency while preserving accuracy.

Validity checking via probability mass transport (Alg.~\ref{alg:collision}) and the lazy check (Alg.~\ref{alg:lazy_check}) trade off efficiency and accuracy: the lazy procedure is faster but more conservative. Relying solely on Alg.~\ref{alg:lazy_check} can speed up ``simpler" problems but underperform when less conservatism is needed, e.g., narrow passages. We therefore propose two hybrid algorithms that combine both methods, improving efficiency while preserving accuracy.

% On one hand, Alg.~\ref{alg:collision} is more accurate but slow. On the other hand,
\subsubsection{Naive Hybrid}

The naive method of improving efficiency is to first use the lazy validity checker, followed by Alg.~\ref{alg:collision} if Alg.~\ref{alg:lazy_check} returns \emph{invalid}. This allows the validity checker to quickly find solutions for ``simpler" problems but also allowing it to find solutions for more ``complex" problems.

Since this Naive Hybrid validity checker first uses the more conservative Alg.~\ref{alg:lazy_check} before Alg.~\ref{alg:collision}, it is also sound. Its conservativeness follows that of Alg.~\ref{alg:collision}.

\subsubsection{Bandit-based Validity Checker}

% The above Naive Hybrid method improves on using either Alg.~\ref{alg:collision} or Alg.~\ref{alg:lazy_check} alone. However, it can be inefficient for nodes in which both Alg.~\ref{alg:collision} and Alg.~\ref{alg:lazy_check} are both invalid. 

% We propose a bandit-based algorithm for validity checking that is able to decide on whether to use Alg.~\ref{alg:collision} when Alg.~\ref{alg:lazy_check} is invalid. The algorithm is presented in Alg.~\ref{alg:bandit_check}. The main idea is to formulate this decision problem as a Bernoulli bandit by estimating the relative conservatism between the two collision checking methods over partitions of the workspace %\qh{please proof read}.

The Naive Hybrid method improves on either Alg.~\ref{alg:collision} or Alg.~\ref{alg:lazy_check} alone, but is inefficient on nodes where both return invalid. We therefore propose Alg.~\ref{alg:bandit_check}, a bandit-based checker that decides whether to invoke Alg.~\ref{alg:collision} after Alg.~\ref{alg:lazy_check} returns invalid. The idea is to cast this decision as a Bernoulli bandit that estimates the relative conservatism of the two checkers over partitions of the workspace.

We first partition the interval $[0, 1]$. Then, for a state $\bar x_t$ determined invalid by Alg.~\ref{alg:lazy_check}, we estimate the volume of the confidence ball centered at $\bar x_t$ that intersects with obstacles. Then, we use a Bernoulli bandit to decide on whether to pull the other arm (use Alg.~\ref{alg:collision}). In this way, each partition has a bandit arm with number of successes (and failures), namely, the frequency of deciding to use Alg.~\ref{alg:collision} after Alg.~\ref{alg:lazy_check} returns invalid, and (not) being successful.

At each iteration, if Alg.~\ref{alg:collision} returns invalid, we first calculate the volume ratio
and determine into which partition of $[0,1]$ this percentage falls. Then, we sample from a beta distribution $p \sim Beta(1 + \alpha, 1+ \beta)$, where $\alpha$ is the number of successes and $\beta$ is the number of failures for that partition. Then, $p$ is compared to a random sample from a uniform distribution in $[0,1]$ in order to decide on whether to call Alg.~\ref{alg:collision}. Finally, if Alg.~\ref{alg:collision} is called, we update the arm's (partition) success or failure count based on the result of the validity checker.

The Bandit-based Validity Checker stochastically chooses between Alg.~\ref{alg:collision} and Alg.~\ref{alg:lazy_check}. Since both of these methods are sound, the Bandit-based checker is also sound. In the worst case, it is as conservative as Alg.~\ref{alg:collision}.

% \begin{algorithm}
% \caption{Bandit-based Validity Checking}\label{alg:bandit_check}
% \begin{algorithmic}[1]
% \Statex \LeftComment{Pre-processing Step}
% \For{$i \in 1:n$ Partitions} 
%     \State $N^{succ}(i) \gets 1$ 
%     \State $N^{fail}(i) \gets 1$
% \EndFor

% \Statex \LeftComment{Validity Checking Step}
% \Require{$X_\text{obs}, (S^e_t)_{t \in \mathbb N_0}, (\mathcal{P}_t)_{t \in \mathbb N_0}, p_\text{safe}, t, \bar x_t$}
% \State $\text{isvalid} \gets $ LazyValidityChecker()
% \If{$\neg\text{isvalid}$}
%     \State $V_{\%} \gets Vol(\big( S^e_t + \bar x_t \big))\setminus\big( S^e_t + \bar x_t \big))\bigcap X_\text{obs}/Vol(\big( S^e_t + \bar x_t \big))$
%     \State $i \gets \lfloor n\cdot V_{\%} \rfloor$
%     \State $r \sim Unif(0,1), p \sim Beta(N^{succ}(i), N^{fail}(i))$
%     \If{r < p}
%     \State $\text{isvalid} \gets $validity checking with Alg.~\ref{alg:collision}
%     \If{isvalid}
%         $N^{succ} \gets N^{succ} + 1$
%     \Else
%         $N^{fail} \gets N^{fail} + 1$
%     \EndIf
%     \EndIf
% \EndIf
% \end{algorithmic}
% \end{algorithm}

\begin{algorithm}[t]
\caption{Bandit-based Validity Checking}
\label{alg:bandit_check}

\KwIn{$X_\text{obs}, (S^e_t)_{t \in \mathbb{N}_0}, (\mathcal{P}_t)_{t \in \mathbb{N}_0}, p_\text{safe}, t, \bar{x}_t$}
\KwOut{$\text{isvalid}$}

% Pre-processing Step
\SetKwComment{LeftComment}{}{\hfill //}

% \LeftComment{Pre-processing Step}
\For{$i \in 1:n$ partitions}{
    $N^{\text{succ}}(i) \gets 1$; \quad    $N^{\text{fail}}(i) \gets 1$\;
}

% Validity Checking Step
% \LeftComment{Validity Checking Step}
$\text{isvalid} \gets$ LazyValidityChecker($X_\text{obs}, (S^e_t)_{t \in \mathbb{N}_0}, t, \bar{x}_t$)\;

\If{$\neg\text{isvalid}$}{
    $V_{\%} \gets {\text{Vol}((S^e_t + \bar{x}_t) \cap X_\text{obs})} / {\text{Vol}(S^e_t)}$\;

    $i \gets \lfloor n \cdot V_{\%} \rfloor$\;

    $r \sim \text{Unif}(0,1),\quad p \sim \text{Beta}(N^{\text{succ}}(i), N^{\text{fail}}(i))$\;

    \If{$r < p$}{
        $\text{isvalid} \gets$ validity checking with Alg.~\ref{alg:collision}\;

        \If{$\text{isvalid}$}{
            $N^{\text{succ}}(i) \gets N^{\text{succ}}(i) + 1$\;
        }
        \Else{
            $N^{\text{fail}}(i) \gets N^{\text{fail}}(i) + 1$\;
        }
    }
}
\end{algorithm}

\section{Lower-Dimensional Ambiguity sets}
\label{sec:lower_dimensional}

The sample complexity of learning an $n$-dimensional distribution scales as $\varepsilon \propto N^{-n}$ \cite{fournier2015rate}, making tight guarantees impractical in high dimensions \cite{chaouach2023structured}. Employing scenario reduction is similarly inefficient on large high-dimensional datasets, and mitigating methods such as clustering are not effective since the inflation factor needed to account for clustering error grows with dimension. Fortunately, in many cases it suffices to learn the projection of the state distribution onto a lower-dimensional subspace. For example, in a high-dimensional problem with a $2$-D workspace where $X_\text{obs}$ and $X_\text{goal}$ are defined by workspace obstacles and a goal region, only the robot's $2$-D position matters for the chance constraints in \eqref{eq:cc_constraints}; the resulting ambiguity sets then scale as $N^{-2}$ rather than $N^{-n}$. More generally, when $X_\text{obs}$ and $X_\text{goal}$ involve multiple constraints (e.g., on control or velocity), one can use several lower-dimensional ambiguity sets (one per constraint space) to improve sample and computational complexity while preserving safety. This section formalizes this idea, starting with the following generalization of Lemma~\ref{lemma:ambiguity_dynamics}.
\begin{lemma}[Lower-Dimensional Ambiguity Dynamics]
\label{lemma:ambiguity_dynamics_low_dimensional}
    Let $\tau, t \in \mathbb N_0$, $M \in\mathbb R^{l\times n}$ and $M_\#P_\tau^e \in \mathbb B(\widehat P_{M,\tau}^e, \varepsilon_\tau)$ for some $\widehat P_{M,\tau}^e \in \mathcal{D}_1(\mathbb R^l)$, $\varepsilon_\tau > 0$. Then $\mathcal{W}(M_\#P_t^e, \widehat P_{M,\tau}^e) \le f_\tau^M(t)$, with $f_\tau^M(t)$ given by
    \begin{align}
    \label{eq:ambiguity_radius_lower_dimensional}
        \begin{cases}
            \varepsilon_\tau + \|M(A^\tau - A^t)\| \mathcal{M}_p(P_0) + \mathcal{M}_p(P_w)\sum_{i=t}^{\tau-1} \|MA^iG\| \\
            \qquad \text{if}\:\: t \le \tau\\
            \varepsilon_\tau + \|M(A^t - A^\tau)\| \mathcal{M}_p(P_0) + \mathcal{M}_p(P_w)\sum_{i=\tau}^{t-1} \|MA^iG\|\\
            \qquad \text{otherwise}
        \end{cases}
    \end{align}
    % Then, $M_\#P_t^e \in \mathbb B(M_\#\widehat P_t^e, \varepsilon_t)$.
\end{lemma}
\begin{proof}
   Proof follows the same reasoning as that of Lemma~\ref{lemma:ambiguity_dynamics}.
\end{proof}
To provide intuition on Lemma~\ref{lemma:ambiguity_dynamics_low_dimensional}, let $M$ be the projection matrix which maps each state $x\in\reals^n$ into its workspace components. This lemma shows that we do not need to learn the $n$-dimensional distribution of the error $e_t$ to reason about how its projection $M_\# P_t^e$ evolves over time, which is way more sample efficient when $l < n$.

We now make the following assumption on the obstacle and goal sets, which allows us to use the lower-dimensional ambiguity tubes:

\begin{assumption}
\label{ass:obs_goal}
    The obstacle and goal sets can be expressed as $X_\text{obs} := \bigcup_{l=1}^L X_\text{obs}^l$ and $X_\text{goal} := \bigcap_{l=1}^L X_\text{goal}^l$, where $X_\text{obs}^l := \{x\in\mathbb R^n : M_l x  \in Y_\text{obs}^l\}$ and $X_\text{goal}^l := \{x\in\mathbb R^n : M_l x  \in Y_\text{goal}^l\}$ for some matrix $M_l \in\mathbb R^{n_l\times n}$, $n_l\le n$, and (possibly time-dependent) $Y_\text{obs}^l,Y_\text{goal}^l \subseteq \mathbb R^{n_l}$ for all $l \in \{1,\dots, L\}$.
\end{assumption}
%

% Assumption~\ref{ass:obs_goal} states that the obstacle (and goal) set can be decomposed into $L$ sets, each of them related to a set $Y_\text{obs}^l$ in some space $\reals^{n_l}$, typically of smaller dimension than $X_\text{obs}^l$. It is easy to check that Assumption~\ref{ass:obs_goal} does not reduce generality, since obstacle and goal sets of any shape can be expressed in that way.

Assumption~\ref{ass:obs_goal} states that the obstacle (and goal) set decomposes into $L$ sets, each related to a set $Y_\text{obs}^l$ in some space $\reals^{n_l}$, typically of smaller dimension than $X_\text{obs}^l$. This assumption does not reduce generality, as obstacle and goal sets of any shape can be expressed in this form.

\begin{example}
\label{ex:example}
    As an illustrative example, consider a robotic system surrounded by workspace obstacles $O_1,\dots,  O_{n_\text{obs}} \subset \reals^\text{ws}$, and where the control is constrained to $U \subset \reals^{n_u}$. Therefore,
    \begin{multline*}
        X_\text{obs} = \{x\in\mathbb R^n : M_1 x \in \cup_{j = 1}^{n_{obs}} O_j\} \ \cup \\
        \{x\in\mathbb R^n : -K(x - \bar x_t ) + \bar u_t \notin U\},
    \end{multline*}
    with $M_1$ being the projection matrix that maps the system's state into its workspace components. Here, we let the goal region represent a physical location $\text{Goal} \subset \reals^\text{ws}$ in the workspace. It is therefore possible to express $X_\text{obs}$ and $X_\text{goal}$ as required in Assumption~\ref{ass:obs_goal} with $L=2$, $M_2 = -K$, $Y_\text{obs}^1 = \bigcup_{j = 1}^{n_{obs}} O_j$, $Y_\text{obs}^2 = \reals^{n_u}\setminus (U - \bar u_t -K\bar x_t)$, $Y_\text{goal}^1 = \text{Goal}$ and $Y_\text{goal}^2 = \reals^{n_u}$. Fig.~\ref{fig:lower_dimensional} illustrates how to leverage the first lower-dimensional ambiguity tube for this example.

    \iffalse
\begin{figure}[h]

\begin{subfigure}[t]{0.5\textwidth}
\includegraphics[width=\linewidth]{projection.png} 
\caption{Projection of the state into the workspace.}
\label{fig:projection}
\end{subfigure}
\begin{subfigure}[t]{0.5\textwidth}
\includegraphics[width=\linewidth]{lower_dimensional.png}
\caption{$l$-th lower-dimensional ambiguity tube of Example~\ref{ex:example}, with $l = 1$. The spheres constitute the ambiguity tube for $P_t^e$, whereas the light-blue balls are the ambiguity tube for $M_{1\#}P_t^e$.}
\label{fig:lower_dimensional}
\end{subfigure}

\caption{Exploiting lower-dimensional ambiguity tubes in Example~\ref{ex:example}. Since the workspace obstacles only constrain the workspace position of the system, we learn an ambiguity tube directly for the position. Note that lower-dimensional ambiguity sets are not the projections of the higher-dimensional ones. In fact, they are way smaller.}
\label{fig:}
\end{figure}
\fi
\end{example}

\begin{figure}[h]
    \centering
    \begin{subfigure}{0.49\linewidth}
        \includegraphics[width=\linewidth]{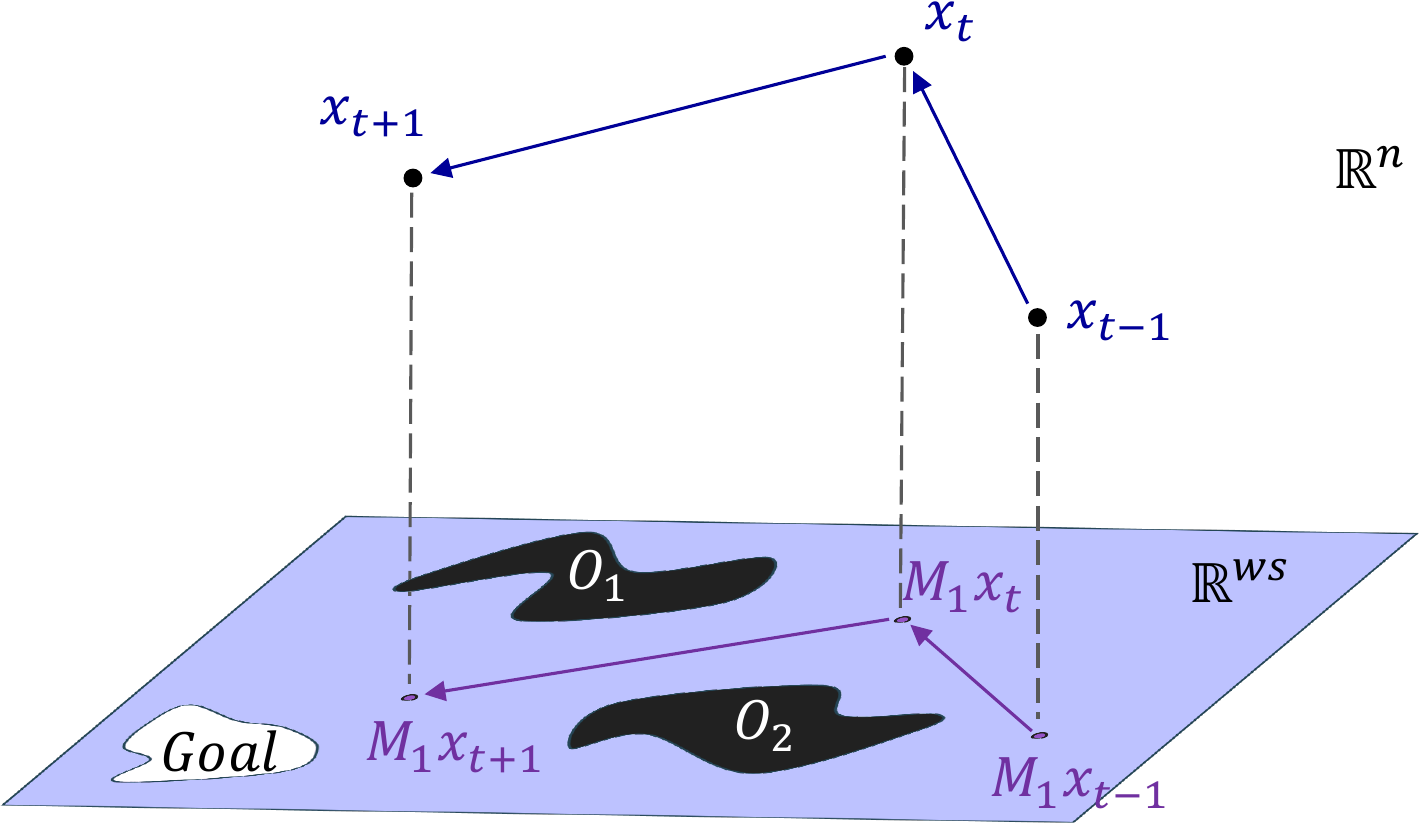}    
        \caption{Projection of $X$ into $\WS$}
    \end{subfigure}
    \begin{subfigure}{0.49\linewidth}
        \includegraphics[width=\linewidth]{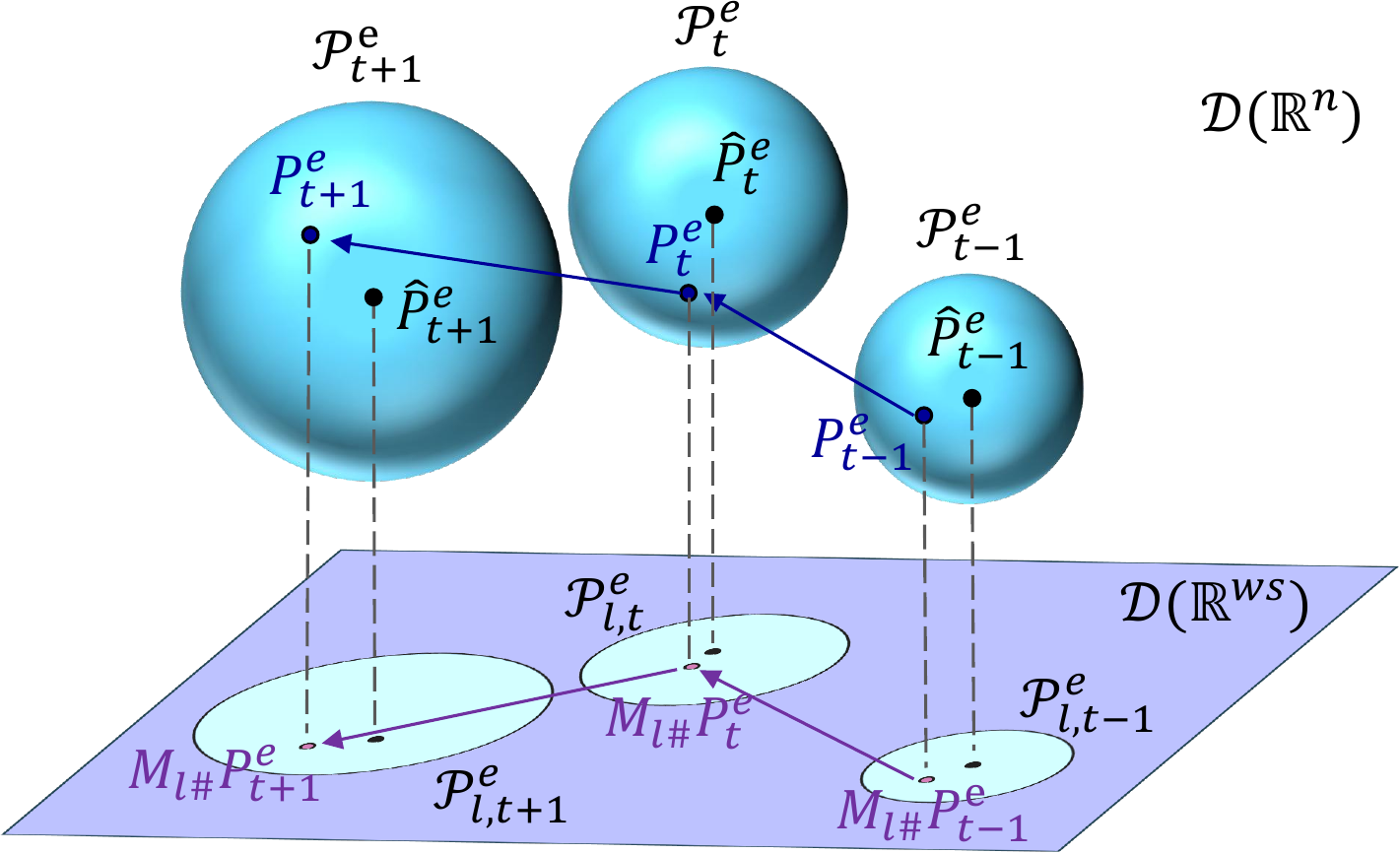}
        \caption{Lower-dim ambiguity tube}
    \end{subfigure}
    \caption{
    Exploiting lower-dimensional ambiguity tubes in Example~\ref{ex:example}. 
    % Left: projection of the state into the workspace. 
    % Right: the $l$-th lower-dimensional ambiguity tube of with $l = 1$. 
    In (b), $l = 1$ and the spheres constitute the ambiguity tube for $P_t^e$, while the light-blue balls are the ambiguity tube for $M_{1\#}P_t^e$. 
    Since the workspace obstacles only constrain the workspace position of the system, we learn an ambiguity tube directly for the position. Note that lower-dimensional ambiguity sets are not the projections of the higher-dimensional ones. In fact, they are way smaller.}
    \label{fig:lower_dimensional}
\end{figure}

%with $M_1$ as previously stated and with $M_2 := K$, $Y_1 := M_1(\mathbb R^n\setminus \bigcup_{j = 1}^{n_{obs}} O_j)$ and $Y_2 := U + K\bar x_t - \bar u_t$. Therefore, according to Theorem~\ref{thm:lower_dimensional_ambiguity_sets}, learning $2$ ambiguity tubes for $2$-dimensional probabilities is enough to guarantee safety.

Next, we formally define the \emph{lower-dimensional ambiguity tubes} that we leverage in this section:
\begin{definition}[Lower-Dimensional Ambiguity Tubes]
    Let $L \in \naturals$ and $M_l \in\reals^{n_l \times n}$ for all $l \in\{1,\dots,L\}$. We refer by lower-dimensional ambiguity tubes to $L$ sets $(\mathcal{P}_{l,t})_{t \in \naturals_0}$, $l \in\{1,\dots,L\}$, of the form $\mathcal{P}_{l,t} := \mathbb B(\widehat P_{l,t}, \varepsilon_{l,t}) \subset \mathcal{D}_1(\reals^{n_l})$, with $\widehat P_{l,t} \in \mathcal{D}_1(\reals^{n_l})$ and $\varepsilon_{l,t}$, for all $t \in \naturals_0$, $l \in\{1,\dots,L\}$.
\end{definition}

We formalize the construction of the tubes $(\mathcal{P}_{l,t}^e)_{t \in \naturals_0}$, $l \in\{1,\dots,L\}$ in such a way that the $l$-th tube contains $(M_{l,\#}P_t^e)_{t\in\naturals_0}$, namely, the pushforward (through $M_l$) of the trajectory of the error's distribution, with high confidence. Let $\tau_1, \tau_1,\dots, \tau_J \le H$ be a set of $J < \infty$ time steps at which we construct the ambiguity sets $\mathcal{P}_{l,\tau_j}^e : =\mathbb B(\widehat P_{l,\tau_j}^e, \varepsilon_{l,\tau_j})$, with $\widehat P_{l,\tau_j}^e := M_{l\#} \widehat  P_{\tau_j}^e$, for all $l \in\{1,\dots,L\}$  from the samples of $\boldsymbol e_{\tau_j}$ and each of them with confidence $1-\beta/(JL)$ as described in Section~\ref{sec:data_driven_construction_ambiguity_set}. Then, for each $t\in\mathbb N_0\setminus\{\tau_j\}_{j = 1}^J$, we pick some $\tau \in \{\tau_j\}_{j = 1}^J$ and define $\mathcal{P}_{l,t}^e := \mathbb B(\widehat P_{l,t}^e, \varepsilon_{l,t})$ with $\widehat P_{l,t}^e := \widehat P_{l,\tau}^e$ and $\varepsilon_{l,t} := f_\tau^{M_l}(t)$ as in \eqref{eq:ambiguity_radius_lower_dimensional}, thus ensuring that $M_{l,\#}P_t^e \in \mathcal{P}_{l,t}^e$ with the same confidence $1-\beta/J$. The pseudocode for obtaining the ambiguity tube is in Alg.~\ref{alg:tubes_low_dimensional}.

%
% \begin{algorithm}
% \caption{Obtain Lower Dimensional Ambiguity Tubes}
% \label{alg:tubes_low_dimensional}
% \begin{algorithmic}[1]
% \Require{$\mathcal{M}_p(P_w), \mathcal{M}_p(P_0), \beta, \{\tau_j\}_{j = 1}^J, A_\text{cl}, \{M_l\}_{l = 1}^L G, \newline \{\hat{\boldsymbol e}^{(i)}_{\tau_1} \: \hat{\boldsymbol e}^{(i)}_{\tau_2} \ldots \hat{\boldsymbol e}^{(i)}_{\tau_J} \}_{i=1}^N$}
% \Ensure{Lower-Dimensional Ambiguity tubes $(\mathcal{P}_{l,0}^e, \mathcal{P}_{l,1}^e, \dots )_{l = 1}^L$}
% \State Construct $\mathbb B(M_{l,\#}\widehat P_{\tau_j}^e, \varepsilon_{\tau_j}^l)$ from $\{\hat {\boldsymbol e}^{(i)}_{\tau_j}\}_{i=1}^N$, with $\varepsilon_{\tau_j}^l$ as in Lemma~\ref{lemma:data_driven_ambiguity_set} and with confidence $1 - \beta/(JL)$ for all $j \in \{1, \dots, J\}$, $l \in \{1, \dots, L\}$
% \State $\mathcal{P}_{l,\tau_j}^e \gets \mathbb B(M_{l,\#}\widehat P_{\tau_j}^e, \varepsilon_{\tau_j}^l)$ for all $j \in \{1, \dots, J\}$, $l \in \{1, \dots, L\}$
% \For{$t \in \mathbb N_0\setminus\{\tau_j\}_{j=1}^J$}
% \For{$l \in \{1, 2, \dots, L\}$}
% \State $\tau_l \gets \arg\min\big\{ f_{\tau_j}^{M_l}(t) : j \in \{1,\dots, J\} \big\}$ with $f_{\tau_j}^{M_l}(t)$ as given in \eqref{eq:ambiguity_radius_lower_dimensional}
% \State $\widehat P_{l,t}^e \gets \widehat P_{\tau_l}^e$
% \State $\varepsilon_t^l \gets f_{\tau_l}^{M_l}(t)$
% \State $\mathcal{P}_{l,t}^e \gets \mathbb B(\widehat P_{l,t}^e, \varepsilon_t^l)$
% \EndFor
% \EndFor
% \end{algorithmic}
% \end{algorithm}

\begin{algorithm}
\caption{Obtain Lower Dimensional Ambiguity Tubes}
\label{alg:tubes_low_dimensional}
\KwIn{$\mathcal{M}_p(P_w), \mathcal{M}_p(P_0), \beta, \{\tau_j\}_{j = 1}^J, A_\text{cl}, \{M_l\}_{l = 1}^L, G,$
$\{\hat{\boldsymbol e}^{(i)}_{\tau_1}, \hat{\boldsymbol e}^{(i)}_{\tau_2}, \ldots, \hat{\boldsymbol e}^{(i)}_{\tau_J} \}_{i=1}^N$}
\KwOut{Lower-Dimensional Ambiguity Tubes $(\mathcal{P}_{l,0}^e, \mathcal{P}_{l,1}^e, \dots )_{l = 1}^L$}

\For{$j \in \{1, \dots, J\},\ l \in \{1, \dots, L\}$}{
    Construct $\mathbb{B}(M_{l,\#}\widehat{P}_{\tau_j}^e, \varepsilon_{\tau_j}^l)$ from $\{\hat{\boldsymbol{e}}^{(i)}_{\tau_j}\}_{i=1}^N$ \\
    where $\varepsilon_{\tau_j}^l$ is from Lemma~\ref{lemma:data_driven_ambiguity_set}, with confidence $1 - \beta/(JL)$\;
    $\mathcal{P}_{l,\tau_j}^e \gets \mathbb{B}(M_{l,\#}\widehat{P}_{\tau_j}^e, \varepsilon_{\tau_j}^l)$\;
}

\For{$t \in \mathbb{N}_0 \setminus \{\tau_j\}_{j=1}^J$}{
    \For{$l \in \{1, 2, \dots, L\}$}{
        $\tau_l \gets \arg\min\big\{ f_{\tau_j}^{M_l}(t) : j \in \{1,\dots, J\} \big\}$\;
        $\widehat{P}_{l,t}^e \gets \widehat{P}_{\tau_l}^e$\;
        $\varepsilon_t^l \gets f_{\tau_l}^{M_l}(t)$\;
        $\mathcal{P}_{l,t}^e \gets \mathbb{B}(\widehat{P}_{l,t}^e, \varepsilon_t^l)$\;
    }
}
\end{algorithm}

In Theorem~\ref{thm:ambiguity_tubes_low_dimensional}, we formalize an important property of these ambiguity tubes.
\begin{theorem}[Soundness of the Lower-Dimensional Ambiguity Tubes]\label{thm:ambiguity_tubes_low_dimensional}
    Let $L \in \naturals$ and $M_l \in\reals^{n_l \times n}$ for all $l \in\{1,\dots,L\}$. Let $(\mathcal{P}_{l,t}^e)_{t \in \mathbb N_0}$, $l \in \{1, \dots, L\}$, be the lower-dimensional ambiguity tubes obtained from Alg.~\ref{alg:tubes_low_dimensional}. Then, each $l$-th tube contains the trajectory of the distribution of $M_l e_t$, and this holds with overall confidence $1-\beta$, i.e., 
    %\begin{align*}
        $M_{l\#}P_t^e \in \mathcal{P}_{l,t}^e, \:\forall\:t\in\naturals_0, l \in \{1,\dots,L\}$, 
%    \end{align*}
    %
    with confidence $1-\beta$.
\end{theorem}

\subsection{Validity Checking using Lower-Dim. Ambiguity Tubes}

% In this section, we show how we can leverage Assumption~\ref{ass:obs_goal} and the lower-dimensional ambiguity tubes obtained via Alg.~\ref{alg:tubes_low_dimensional} to determine if a node in the planning tree is valid. Determining if the tree has reached the goal region follows an analogous logic. First, given a node of the tree corresponding to time step $t$ and reference state $\bar x_t$, we obtain the state ambiguity sets $\mathcal{P}_{l,t} := \mathbb B(\widehat P_{l,t}, \varepsilon_{l,t})$, with with $\widehat P_{l,t} = \widehat P_{l,t}^e * \delta_{M_l\bar x_t}$ for all $l\in\{1\dots,L\}$, for this node. Next, we determine whether or not the robot is not in collision with the obstacles by obtaining the worst-case probabilities
% %
% $\min_{P\in \mathcal{P}_{l, t}} P(\reals^{n_l}\setminus Y_\text{obs}^l)$ of not colliding with the $l$-th obstacle, for all $l\in\{1\dots,L\}$. If the sum of the previous probabilities is bigger than $ p_\text{safe}-1+L$, we conclude that the node is not in collision. Otherwise, we declare the node invalid. Similarly, we conclude that the robot has reached the goal if by computing $\min_{P\in \mathcal{P}_{l, t}} P(\reals^{n_l}\setminus Y_\text{goal}^l)$, for all $l\in\{1\dots,L\}$, and determining if the sum of these probabilities is greater than $p_\text{safe} -1 +L$. The worst-case probabilities are computed through the approach described in Section~\ref{sec:validity_checking}. The pseudocode to perform validity checking employing the lower-dimensional ambiguity sets is given in Alg.~\ref{alg:collision_lower_dimensional}.
% 
We now use Assumption~\ref{ass:obs_goal} and the lower-dimensional ambiguity tubes from Alg.~\ref{alg:tubes_low_dimensional} to check node validity (the goal check is analogous). For a node at time step $t$ with reference state $\bar x_t$, we form the state ambiguity sets $\mathcal{P}_{l,t} := \mathbb B(\widehat P_{l,t}, \varepsilon_{l,t})$, where $\widehat P_{l,t} = \widehat P_{l,t}^e * \delta_{M_l\bar x_t}$ for $l\in\{1,\dots,L\}$. We then compute the worst-case non-collision probabilities $\min_{P\in \mathcal{P}_{l, t}} P(\reals^{n_l}\setminus Y_\text{obs}^l)$ via the method of Section~\ref{sec:validity_checking}; the node is collision-free if their sum exceeds $p_\text{safe}-1+L$, and invalid otherwise. The goal is declared reached when $\sum_l \min_{P\in \mathcal{P}_{l, t}} P(\reals^{n_l}\setminus Y_\text{goal}^l) > p_\text{safe}-1+L$. Alg.~\ref{alg:collision_lower_dimensional} gives the pseudocode.
%
% \begin{algorithm}
% \caption{Validity Checking Via Probability Mass Transport and Lower-Dimensional Ambiguity Tubes}
% \label{alg:collision_lower_dimensional}
% \begin{algorithmic}[1]
% \Require{$M_l, Y_\text{obs}^l, (\mathcal{P}_{l,t})_{t \in \mathbb N_0} l\in\{1,\dots,L\}, p_\text{safe}, t, \bar x_t$}
% \Ensure{$\text{isvalid}$}
% \For{$l\in\{1\dots, L\}$}
% \State Obtain $\mathcal{P}_{l,t} = \mathbb B(\widehat P_{l,t}, \varepsilon_{l, t})$, with $\widehat P_{l,t} = \widehat P_{l,t}^e * \delta_{M_l\bar x_t}$
% \State Obtain $\{(\hat{\boldsymbol x}^{(i)}, \boldsymbol a_i)\}_{i = 1}^{N_c}$ from $\widehat P_{l,t}$
% \State Compute $\{\text{dist}(\hat{\boldsymbol x}^{(i)}, Y_\text{obs}^l)\}_{i = 1}^{N_c}$
% \State Sort $\{\hat{\boldsymbol x}^{(i)}\}_{i=1}^{N_c}$ according to $\text{dist}(\hat{\boldsymbol x}^{(i)}, Y_\text{obs}^l)$ in an increasing fashion
% \State Obtain $\alpha_l := \min_{P\in \mathcal{P}_{l, t}} P(\reals^{n_l}\setminus Y_\text{obs}^l)$ via \eqref{eq:collision_probability0} and
% \eqref{eq:collision_probability}
% \EndFor
% \If{$1 - L + \sum_{l = 1}^L \alpha_l > p_\text{safe}$}
% \State $\text{isvalid} \gets \text{true}$
% \Else
% \State $\text{isvalid} \gets \text{false}$
% \EndIf
% %
% \end{algorithmic}
% \end{algorithm}
%
\begin{algorithm}
\caption{Validity Checking via Probability Mass Transport and Lower-Dimensional Ambiguity Tubes}
\label{alg:collision_lower_dimensional}

\KwIn{$M_l,\ Y_\text{obs}^l,\ (\mathcal{P}_{l,t})_{t \in \mathbb{N}_0,\ l \in \{1,\dots,L\}},\ p_\text{safe},\ t,\ \bar{x}_t$}
\KwOut{$\text{isvalid}$}

\For{$l \in \{1, \dots, L\}$}{
    Obtain $\mathcal{P}_{l,t} = \mathbb{B}(\widehat{P}_{l,t}, \varepsilon_{l,t})$, where $\widehat{P}_{l,t} = \widehat{P}_{l,t}^e * \delta_{M_l \bar{x}_t}$\;

    Obtain $\{(\hat{\boldsymbol{x}}^{(i)}, \boldsymbol{a}_i)\}_{i = 1}^{N_c}$ from $\widehat{P}_{l,t}$\;

    Compute $\{\text{dist}(\hat{\boldsymbol{x}}^{(i)}, Y_\text{obs}^l)\}_{i = 1}^{N_c}$\;

    Sort $\{\hat{\boldsymbol{x}}^{(i)}\}_{i=1}^{N_c}$ by increasing $\text{dist}(\hat{\boldsymbol{x}}^{(i)}, Y_\text{obs}^l)$\;

    Compute $\alpha_l := \min_{P \in \mathcal{P}_{l,t}} P\left( \mathbb{R}^{n_l} \setminus Y_\text{obs}^l \right)$ using \eqref{eq:collision_probability0} and \eqref{eq:collision_probability}\;
}

\If{$1 - L + \sum_{l = 1}^L \alpha_l > p_\text{safe}$}{
    $\text{isvalid} \gets \text{true}$\;
}
\Else{
    $\text{isvalid} \gets \text{false}$\;
}

\end{algorithm}

\subsubsection{Lazy Check using Lower-dim Ambiguity Tubes}

We adapt the lazy check in Section~\ref{sec:lazy} to the lower-dimensional ambiguity tubes as described in Alg.~\ref{alg:tubes_low_dimensional}. 
% The approach consists on obtaining $L$ tubes, which we denote \emph{lower-dimensional confidence tubes}, where the $l$-th tube contains more than $p_\text{safe}$ probability mass from the random variable $M_l \boldsymbol e_t$, for all time $t\in\naturals_0$, with overall confidence $1-\beta$. Using these tubes we obtain an upper bound in the probability of collision in an efficient way. If this upper bound is smaller than $1 - p_\text{safe}$, we conclude that the robot is not in collision with that obstacle. Otherwise, we make use of Alg.~\ref{alg:collision_lower_dimensional}. Again, making use of this approach to determine if the goal has been reached is straightforward, and we omit that discussion for the sake of brevity.
We compute $L$ \emph{lower-dimensional confidence tubes}, where the $l$-th tube contains more than $p_\text{safe}$ probability mass of $M_l \boldsymbol e_t$ for all $t\in\naturals_0$, with overall confidence $1-\beta$. These tubes yield an efficient upper bound on the collision probability: if the bound falls below $1 - p_\text{safe}$, the node is collision-free; otherwise, we fall back on Alg.~\ref{alg:collision_lower_dimensional}. 
% The goal check is analogous and omitted.

% The algorithm is as follows. Similarly to Alg.~\ref{alg:lazy_check}, we rely on having already computed, in an offline fashion, the sequences of sets $(S^e_{l,t})_{t\in\naturals_0}$, with $l\in\{1,\dots,L\}$, such that
% %
% \begin{align}
% \label{eq:confidence_balls_lower_dimensional}
%     \min_{P\in\mathcal{P}_{l,t}^e} P(S^e_{l,t}) > p_\text{safe}/L, \;\;\forall\: t\in\naturals_0,\; l\in\{1,\dots,L\}.
% \end{align}
% %
% Again, we restrict ourselves to balls centered at the origin, and compute the (approximately) smallest balls that satisfy \eqref{eq:confidence_balls} via bisection on their radius. Then, at time-step $t$ of the planning phase, if no ball $S^e_{l,t} + M_l\bar x_t$ intersects $Y_\text{obs}^l$, we conclude that the node is not in collision. Otherwise, we compute the exact probability of collision via Alg.~\ref{alg:collision_lower_dimensional}.
% %We denote the balls $(S^e_t)_{t\in\naturals_0}$ as \emph{confidence balls}, since they are guaranteed to contain at least $p_\text{safe}$ probability mass from $P_t$ with high confidence, as we state in Lemma~\ref{lemma:confidence_balls}.
% The pseudocode for obtaining the lower-dimensional confidence tubes is described in Alg.~\ref{alg:confidence_balls_lower_dimensional}.

The pseudocode is in Alg.~\ref{alg:confidence_balls_lower_dimensional}.
As in Alg.~\ref{alg:lazy_check}, we precompute sequences $(S^e_{l,t})_{t\in\naturals_0}$ for $l\in\{1,\dots,L\}$ such that
\begin{align}
\label{eq:confidence_balls_lower_dimensional}
    \min_{P\in\mathcal{P}_{l,t}^e} P(S^e_{l,t}) > p_\text{safe}/L, \;\;\forall\: t\in\naturals_0,\; l\in\{1,\dots,L\},
\end{align}
restricting to balls centered at the origin and finding the (approximately) smallest valid radius by bisection. At time $t$, if no ball $S^e_{l,t} + M_l\bar x_t$ intersects $Y_\text{obs}^l$, the node is collision-free. 
% otherwise, we invoke Alg.~\ref{alg:collision_lower_dimensional}. 

%
% \begin{algorithm}
% \caption{Obtain Lower-Dimensional Confidence Tubes}\label{alg:confidence_balls_lower_dimensional}
% \begin{algorithmic}[1]
% \Require{$\mathcal{M}_p(P_w), \mathcal{M}_p(P_0), \beta,  \{\tau_j\}_{j = 1}^J, (\mathcal{P}_{l,t}^e = \mathbb B(\widehat P_{l,t}^e, \varepsilon_{l,t}))_{t \in \naturals_0}$ with $l\in\{1,\dots,L\}$ from Alg.~\ref{alg:tubes_low_dimensional}, $A_\text{cl}, G, M_l$ with $l\in\{1,\dots,L\}$}
% \Ensure{Lower-dimensional confidence tubes $(S^e_{l,t})_{t \in \mathbb N_0}$, $l\in\{1,\dots,L\}$}
% %
% \For{$l \in \{1,\dots, L\}$}
% \For{$j \in \{1,\dots, J\}$}
% \State $I_j \gets \big\{t \in \naturals_0 : \nexists k \in \{1,\dots, J\} \:\:\text{s.t.}\: f_{\tau_k}^{M_l}(t) < f_{\tau_j}^{M_l}(t) \big\}$
% \State $\bar \varepsilon_{l,\tau_j} \gets \max\{ f_{\tau_j}^{M_l}(t) : t \in I_j\}$ 
% \State Obtain $S^e_{l, \tau_j}$ s.t. $\min\big\{P(S^e_{l, \tau_j}) : P\in\mathbb B(\widehat P_{l, \tau_j}^e, \bar \varepsilon_{l, \tau_j}) \big\} > p_\text{safe}/L$
% \State $S^e_{l, t} \gets S^e_{l, \tau_j}$ for all $t \in I_j$
% \EndFor
% \EndFor
% %

% \end{algorithmic}
% \end{algorithm}

\begin{algorithm}
\caption{Construct Lower-Dim. Confidence Tubes}
\label{alg:confidence_balls_lower_dimensional}

\KwIn{$\mathcal{M}_p(P_w),\ \mathcal{M}_p(P_0),\ \beta,\ \{\tau_j\}_{j=1}^J$, $(\mathcal{P}_{l,t}^e = \mathbb{B}(\widehat{P}_{l,t}^e, \varepsilon_{l,t}))_{t \in \mathbb{N}_0,\ l \in \{1,\dots,L\}}$ from Alg.~\ref{alg:tubes_low_dimensional},
$A_\text{cl},\ G,\ M_l$ for $l \in \{1,\dots,L\}$}
\KwOut{Lower-dimensional confidence tubes $(S^e_{l,t})_{t \in \mathbb{N}_0}$ for $l \in \{1,\dots,L\}$}

\For{$l \in \{1,\dots,L\}$}{
    \For{$j \in \{1,\dots,J\}$}{
        $I_j \gets \big\{ t \in \mathbb{N}_0 : \nexists k \in \{1,\dots,J\} \text{ s.t. } f_{\tau_k}^{M_l}(t) < f_{\tau_j}^{M_l}(t) \big\}$\;
        $\bar{\varepsilon}_{l,\tau_j} \gets \max \big\{ f_{\tau_j}^{M_l}(t) : t \in I_j \big\}$\;

        Obtain $S^e_{l, \tau_j}$ such that
        $\min \big\{ P(S^e_{l, \tau_j}) : P \in \mathbb{B}(\widehat{P}_{l, \tau_j}^e, \bar{\varepsilon}_{l, \tau_j}) \big\} > \dfrac{p_\text{safe}}{L}$\;

        $S^e_{l, t} \gets S^e_{l, \tau_j}$ for all $t \in I_j$\;
    }
}
\end{algorithm}

\begin{lemma}[Lower-Dimensional Confidence Tubes]
\label{lemma:confidence_balls_lower_dimensional}
    Let
    %$(\mathcal{P}_t)_{t\in\naturals_0}$ be the ambiguity tube obtained from Alg.~\ref{alg:tube}, and
    $(S^e_{l,t})_{t\in\naturals_0}$, $l\in\{1,\dots,L\}$, be the confidence tubes obtained from Alg.~\ref{alg:confidence_balls_lower_dimensional}
    %and denote $e_{l,t} := M_l e_t$ for all $t,l$.
    . Then, $M_{l\#}P_t^e(S^e_{l,t}) > p_\text{safe}/L$ for all $t\in\naturals_0$, $l\in\{1, \dots, L\}$, with confidence $1-\beta$.
\end{lemma}

\section{WDR-$\mathcal{X}$ Algorithm}
\label{sec:algorithm}

% \begin{algorithm}
%     \caption{Precomputation Phase}
%     \label{alg:precomputation}
%     \begin{algorithmic}[1]
%     \State \IG{Fill in precomputation phase}
%     \end{algorithmic}
% \end{algorithm}

% \begin{algorithm}
% \caption{Precomputation Phase}
% \label{alg:precomputation}

% \BlankLine
% % Placeholder for the actual algorithm steps
% \IG{Fill in precomputation phase}

% \end{algorithm}

 % A pseudocode for an algorithm for the first phase is shown in Alg.~\ref{alg:precomputation}.

% In this section, we present our sampling-based motion planning algorithmic framework, Wasserstein Distributionally Robust-$\mathcal{X}$ (WDR-$\mathcal{X}$), where $\mathcal{X}$ is any tree-based sampling algorithm per Section~\ref{sec:treebasedplanners}, to solve Prob.~\ref{problem}. A pseudocode of the framework is shown in Alg.~\ref{alg:wdr-rrt}\footnote{Note: If a lower-dimensional ambiguity set is used, the corresponding algorithms in Section~\ref{sec:lower_dimensional} will be applied instead.}.

In this section, we present our general motion planning framework, \textit{Wasserstein Distributionally Robust-$\mathcal{X}$} (WDR-$\mathcal{X}$), to solve Problem~\ref{problem}. Here, $\mathcal{X}$ is any tree-based kinodynamic  sampling-based algorithm from Section~\ref{sec:treebasedplanners}. Pseudocode is given in Alg.~\ref{alg:wdr-rrt}.\footnote{If lower-dimensional ambiguity sets are used, the corresponding algorithms from Section~\ref{sec:lower_dimensional} apply instead.}

% The framework operates in two phases: construction of ambiguity sets and planning. In the first phase, we learn and construct ambiguity sets from data, as discussed in Section~\ref{sec:data_driven_construction_ambiguity_set}. For each System~\eqref{eq:system}, we only need to compute the ambiguity and confidence tubes once, using Algs. $2$ and $4$ (Lines $1-2$). After this construction phase, we can use (and re-use) the ambiguity tubes in a typical tree search motion planner in the second phase for different environments.

The framework operates in two phases: ambiguity set construction and planning. In the first phase (Lines 1–2), we learn ambiguity and confidence tubes from data via Algs. 2 and 4 (Section~\ref{sec:data_driven_construction_ambiguity_set}). Note that this phase needs to be done only once per system; the resulting tubes can be reused across planning queries and environments.

% In the second phase, we search for a motion plan using the dynamics of the system (Lines $4-10$). The search for a solution is performed by growing a motion tree, in a manner similar to Alg.~\ref{alg:treebasedplanners}. A node is the nominal state of the system, which follow Eq.~\eqref{eq:reference_dynamics}, and the Sample (Line $5$), Select (Line $6$), and Extend (Line $7$) subroutines are performed on the nominal state via a randomly sampled feedforward control. A new node is added to the tree if it satisfies Eq.~\eqref{eq:cc_constraints_safety}, which is checked using the validity checkers in Section~\ref{sec:validity_checking} on the ambiguity tube $(\mathcal{P}_t^e)_{t \in \mathbb{N}_0}$ and confidence tube $(S^e_t)_{t \in \mathbb{N}_0}$ (Line $8$). A solution is found if a valid node satifies Eq.~\eqref{eq:cc_constraints_goal}, and the solution is extracted (Line $12$). 

In the second phase (Lines 4–10), we grow a motion tree as in Alg.~\ref{alg:treebasedplanners}. Each node is a nominal state evolving per \eqref{eq:reference_dynamics}, and Sample (Line 5), Select (Line 6), and Extend (Line 7) act on the nominal state via a randomly sampled feedforward control. A new node is added if it satisfies Constraint \eqref{eq:cc_constraints_safety}, checked using the validity checkers of Section~\ref{sec:validity_checking} on $(\mathcal{P}_t^e)_{t \in \mathbb{N}_0}$ and $(S^e_t)_{t \in \mathbb{N}_0}$ (Line 8). A solution is returned once a valid node satisfies Constraint \eqref{eq:cc_constraints_goal} (Line 12).

% The planner performs the same subroutines as Alg.~\ref{alg:treebasedplanners} for a deterministic linear system, except for the validity check subroutine. An appealing aspect of our algorithm is that we only need to modify the validity checking subroutine in Alg.~\ref{alg:treebasedplanners}, and all other subroutine utilizes the nominal state and (feedforward) dynamics of the nominal system. Therefore, we only propagate the nominal dynamics instead of propagating uncertainty during tree search. This allows much faster planning (and re-planning) for the same system in different environments, as we can re-use the ambiguity tube.

The planner performs the same subroutines as Alg.~\ref{alg:treebasedplanners} for a deterministic linear system except in the validity check: all other subroutines act on the nominal state and feedforward dynamics, so we propagate only the nominal dynamics during tree search rather than the uncertainty. This enables much faster planning and replanning across environments for the same system, since the ambiguity tube is reused.

% In this work, we focus on the Rapidly Exploring Random Trees (RRT) \cite{lavelle1998rrt} as the base algorithm, but any kinodynamic sampling-based motion planning algorithm can be used in principle.

% \begin{algorithm}
%     \caption{WDR-RRT ($\mathcal{X}, \mathcal{U}, \mathcal{X}_{goal}, \mathcal{X}_{obs}, x_{init}$, N) \qh{this might be too generic?}}
%     \label{alg:wdr-rrt}
%     \begin{algorithmic}[1]
%     \State $G = (\mathbb{V} \leftarrow \{x_{init}$\}, $\mathbb{E} \leftarrow \emptyset)$
%     \For{N iterations}
%         \State $x_{rand} \leftarrow$ Sample()
%         \State $n_{select} \leftarrow$ Select($x_{rand}$)
%         \State $n_{new} \leftarrow$ Extend($n_{selected}$)
%         \If {isDistributionallyRobustValid($n_{new}$)}
%             \State $\mathbb{V} \leftarrow \mathbb{V} \cup \{n_{new}$\}
%             \State $\mathbb{E} \leftarrow \mathbb{E} \cup \{edge(n_{select}, n_{new})\}$
%     \EndIf
%     \EndFor
%         \State Prune($\mathbb{V}, \mathbb{E}$)
%     \State \Return $G=(\mathbb{V}, \mathbb{E})$
%     \end{algorithmic}
% \end{algorithm}
%
\begin{algorithm}
\caption{WDR-$\mathcal{X}$
}
% $(\mathcal{X}, \mathcal{U}, \mathcal{X}_{\text{goal}}, \mathcal{X}_{\text{obs}}, x_{\text{init}}, N)$
\label{alg:wdr-rrt}

\KwIn{State space $\mathcal{X}$, input space $\mathcal{U}$, $X_{\text{obs}}$, $X_{\text{goal}}$, $\mathcal{M}_p(P_w), \mathcal{M}_p(P_0), \beta, \{\tau_j\}_{j = 1}^J$, $A_\text{cl}, G, X_0, W,$
$\{\hat{\boldsymbol e}^{(i)}_{\tau_1}, \hat{\boldsymbol e}^{(i)}_{\tau_2}, \ldots, \hat{\boldsymbol e}^{(i)}_{\tau_J} \}_{i=1}^N$, $p_{safe}$, iterations $N$}
\KwOut{Tree $G = (\mathbb{V}, \mathbb{E})$}

\vspace{0.5em}
\nlnonumber
\KwPhase{\textbf{1: Offline Tube Construction}}

Ambiguity Tubes $(\mathcal{P}_t^e)_{t \in \mathbb{N}_0}$ using Alg.~\ref{alg:tube}\;
Confidence Tubes $(S^e_t)_{t \in \mathbb{N}_0} using Alg.~\ref{alg:confidence_balls}$\;
\vspace{0.5em}
\nlnonumber
\KwPhase{\textbf{2: Online Planning}}

Initialize $G = (\mathbb{V} \gets \{x_{\text{init}}\},\ \mathbb{E} \gets \emptyset)$\;

%\For{$i = 1$ \KwTo $N$}
\While{$\text{True}$}{
    $\bar{x}_{\text{rand}}, \bar{u}_{rand} \gets$ \texttt{Sample}()\;
    $n_{\text{select}} \gets$ \texttt{Select}($x_{\text{rand}}$)\;
    $n_{\text{new}} \gets$ \texttt{Extend}($n_{\text{select}}, \bar{u}_{rand}$)\;
    
    \If{\texttt{ValidityCheck}($n_{\text{new}}, (\mathcal{P}_t^e)_{t \in \mathbb{N}_0}, (S^e_t)_{t \in \mathbb{N}_0})$}{
        $\mathbb{V} \gets \mathbb{V} \cup \{n_{\text{new}}\}$\;
        
        $\mathbb{E} \gets \mathbb{E} \cup \{\texttt{edge}(n_{\text{select}}, n_{\text{new}})\}$\;

        \If{$n_{\text{new}}$ satifies Eq.~\eqref{eq:cc_constraints_goal}}{
            \Return \texttt{ExtractPath}($G, n_{\text{new}}$)\;
        }
    }
}
\Return Failure 
% $G = (\mathbb{V}, \mathbb{E})$\;
\end{algorithm}

% In order to plan leveraging lower-dimensional ambiguity tubes, the algorithm follows the same logic as Alg.~\ref{alg:wdr-rrt}  but with a few differences: first, in Line $1$ one should obtain the lower-dimensional ambiguity tubes $(\mathcal{P}_{l,t}^e)_{t \in \mathbb{N}_0}$ for all $l \in\{1, \dots, L\}$ according to Alg.~\ref{alg:tubes_low_dimensional}. The same should be done in Line $2$ for the confidence tubes $(S_{l,t}^e)_{t\in \mathbb{N}_0}$ with Alg.~\ref{alg:confidence_balls_lower_dimensional}. Finally, Alg.s~\ref{alg:collision_lower_dimensional} and \ref{alg:lazy_check_lower_dimensional}, or a combination of them in the spirit of Section~\ref{sec:combined_validity_checkers}, should be used for validity check in Line $9$.

Planning with lower-dimensional ambiguity tubes follows the same logic as Alg.~\ref{alg:wdr-rrt}, with a few differences: in Line 1, the lower-dimensional ambiguity tubes $(\mathcal{P}_{l,t}^e)_{t \in \mathbb{N}_0}$ for $l \in\{1, \dots, L\}$ are obtained via Alg.~\ref{alg:tubes_low_dimensional}; in Line 2, the confidence tubes $(S_{l,t}^e)_{t\in \mathbb{N}_0}$ are obtained via Alg.~\ref{alg:confidence_balls_lower_dimensional}; and in Line 9, validity is checked using Alg.~\ref{alg:collision_lower_dimensional}, lazy check, or a hybrid of them as in Section~\ref{sec:combined_validity_checkers}.

\subsection{Theoretical Analysis}
\label{sec:analysis}

In this section, we analyze the theoretical properties of our algorithmic framework. All proofs can be found in the Appendix. We first show that Alg.~\ref{alg:wdr-rrt} is sound.
\begin{theorem}[Soundness]
\label{thm:soundness}
    Let Alg.~\ref{alg:wdr-rrt} be equipped with Algs.~\ref{alg:collision}, \ref{alg:lazy_check} or any of the hybrid validity checkers in Section~\ref{sec:combined_validity_checkers}. Then, every motion plan $((\bar u_t, \bar x_t))_{t = 0}^T$ returned by Alg.~\ref{alg:wdr-rrt} solves Problem~\ref{problem}, i.e., \eqref{eq:cc_constraints} is satisfied.
\end{theorem}

Alg.~\ref{alg:wdr-rrt} is also probabilistically complete with respect to the conservativeness of the validity checker.\footnote{All the validity checking methods are conservative with respect to the true probability of safety and goal-reachability, which is desirable for soundness of the computed solutions.} As the number of samples in the first phase goes to infinity, the conservative collision checker converges to exact collision detection probabilities. This means that as the number of iterations of the algorithm approaches infinity, the probability of finding a solution, if one exists (and is deemed valid with the conservative validity checkers), approaches $1$.

\begin{theorem}[Probabilistic Completeness]
\label{thm:pc}
    Let Alg.~\ref{alg:wdr-rrt} be equipped with Alg.~\ref{alg:collision}. Let $\mathcal{P}^e := \mathbb{B}(\widehat P^e, \varepsilon)$ be an ambiguity ball of radius $\varepsilon >0$ and centered on a discrete distribution $\widehat P^e$, such that $\bigcup_{t\in\naturals_0} \mathcal{P}^e_t \subseteq \mathcal{P}^e$. Assume also that there exists a valid motion plan $\{(\bar u_t, \bar x_t)\}_{t=0}^T$ such that
    \begin{align*}
        & \min_{P\in\mathcal{P}^e} P\big[\boldsymbol{x}_t \notin  X_\text{obs}\big] > p_\text{safe} && \forall t \in \{0,\dots,T\},\\
        & \min_{P\in\mathcal{P}^e} P\big[\boldsymbol{x}_T \in X_\text{goal}\big] > p_\text{safe}.
    \end{align*}
    Then, as the number of iterations goes to infinity, the probability that Alg.~\ref{alg:wdr-rrt} finds a solution approaches $1$.
\end{theorem}

For simplicity, in Theorem~\ref{thm:pc}, we assume that there exists a valid motion plan that is robust with respect to a rigid, i.e., time-invariant, ambiguity tube for the distribution of the error. Note the time-varying ambiguity tube obtained via Alg.~\ref{alg:tube} is tighter than the rigid one, i.e., it contains less trajectories of distributions while having the same guarantees of containing the true trajectory. However, the rigid tube assumption greatly simplifies the planning algorithm and the proof of probabilistic completeness. To relax this assumption, ones needs to plan in a different space, e.g., the belief space~\cite{Ho2022gbt}.

\begin{remark}
\label{rem:completeness_extension}
    % Theorem~\ref{thm:pc} also applies when validity is checked via our hybrid validity checkers in Section~\ref{sec:combined_validity_checkers}: first, the Naive Hybrid either deems a node as valid or assesses its validity via Alg.~\ref{alg:lazy_check}. Furthermore, Alg.~\ref{alg:bandit_check} always checks for validity via Alg.~\ref{alg:lazy_check} with some probability, with a fixed lower bound on this probability thus also inheriting the property of probabilistic completeness.

    Theorem~\ref{thm:pc} also applies to the hybrid validity checkers of Section~\ref{sec:combined_validity_checkers}: first, the Naive Hybrid algorithm either deems a node valid via Alg.~\ref{alg:lazy_check} or defers to Alg.~\ref{alg:collision}. On the other hand, the Bandit-based algorithm always checks validity of a node via Alg.~\ref{alg:collision}, which is less conservative, with positive probability. Furthermore, when an attempt to propagate a node is not allowed by the more conservative algorithm, the probability of using Alg.~\ref{alg:collision} in the next iteration increases. By these considerations, both hybrid algorithms inherit the probabilistic completeness property under the assumptions in Thm.~\ref{thm:pc}.
\end{remark}

In the following two theorems, we conclude that Alg.~\ref{alg:wdr-rrt}, leveraging the lower-dimensional ambiguity tubes for validity checking, is also sound and probabilistically complete with respect to the conservatism of the collision checkers.
\begin{theorem}
\label{thm:solution_low_dimensional}
    Let Alg.~\ref{alg:wdr-rrt} be equipped with low-dimensional ambiguity tubes and Alg.~\ref{alg:collision_lower_dimensional}, 
    % \ref{alg:lazy_check_lower_dimensional} 
    lazy checker, 
    or any of the hybrid validity checkers in Section~\ref{sec:combined_validity_checkers}. Every motion plan $((\bar u_t, \bar x_t))_{t = 0}^T$ returned by the algorithm solves Problem~\ref{problem}.
\end{theorem}

\begin{theorem}
\label{thm:pc_lower_dimensional}
    Let Alg.~\ref{alg:wdr-rrt} be equipped with Alg.~\ref{alg:collision_lower_dimensional}. Also, let $\{\mathcal{P}_l^e := \mathbb{B}(\widehat P_l^e, \varepsilon_l)\}_{l=1}^L$ be lower-dimensional ambiguity balls of radii $\varepsilon_l >0$ and centered on the discrete distributions $\widehat P_l^e$ for $l \in\{1,\dots, L\}$, such that $\bigcup_{t=0}^T \mathcal{P}^e_{l,t} \subseteq \mathcal{P}_l^e$. Assume also that there exists a valid motion plan $\{(\bar u_t, \bar x_t)\}_{t=0}^T$ such that
    \begin{align*}
        & \sum_{l=1}^L \max_{P\in \mathcal{P}_l^e} P[ M_l \boldsymbol x_t \in Y_\text{obs}^l] <1 - p_\text{safe} && \forall t \in \{0,\dots,T\},\\
        & \sum_{l=1}^L \max_{P\in \mathcal{P}_l^e} P[ M_l \boldsymbol x_T \notin Y_\text{goal}^l] <1 - p_\text{safe}.
    \end{align*}
    Then, as the number of iterations goes to infinity, the probability that Alg.~\ref{alg:wdr-rrt} finds a solution approaches $1$.
\end{theorem}
%
% \begin{remark}
%     By the same reasoning of Remark~\ref{rem:completeness_extension}, Theorem~\ref{thm:pc_lower_dimensional} also applies when validity is checked via hybrid validity checkers as the ones presented in Section~\ref{sec:combined_validity_checkers}, when these leverage Alg.~\ref{alg:collision_lower_dimensional} or \ref{alg:lazy_check_lower_dimensional}.
% \end{remark}

\section{Evaluation}
\label{sec:evaluation}

% The uncertainty propagation and collision-checking approaches are tested

We evaluate the performance of our proposed method in several case studies and benchmarks.  
To put our results in perspective, we compare the following uncertainty propagation and collision-checking approaches: 
\begin{enumerate}[label=\arabic*.]
    \item TS: moment-based distributionally robust-RRT from \cite{summers2018distributionally}.
    \item Risk-Assigned: moment-based distributionally robust-RRT with risk allocation from~\cite{ekenberg2023distributionally}.
    \item Particle-WDR-RRT: Alg.~\ref{alg:wdr-rrt} with particle-based validity checker.
    \item Confidence-WDR-RRT: Alg.~\ref{alg:wdr-rrt} with confidence-region validty checker.
    \item Hybrid-WDR-RRT: Alg.~\ref{alg:wdr-rrt} with Naive Hybrid validity checker.
    \item Bandit-WDR-RRT: Alg.~\ref{alg:wdr-rrt} with Bandit-based validity Checker.
\end{enumerate}

All algorithms were implemented in C++ with the Open Motion Planning Library (OMPL)\cite{sucan2012the-open-motion-planning-library}. We utilize kinodynamic RRT \cite{Kleinbort2019RRTPC} as the underlying planner for WDR-$\mathcal{X}$ for all experiments. The simulations were conducted single threaded on a machine with 3.7GHz CPU and 32GB of RAM.

\subsection{4-D Linear System}

% We first consider a 4-D linear system subject to process noise and initial state uncertainty from~\cite{summers2018distributionally}. 
% We benchmarked the systems over 4 types of environments to represent a diverse range of planning scenarios: (i) Scattered: Relatively sparse obstacles with more free space, (ii) Cluttered: Dense obstacle space with less free space, (iii) Narr(`width'): To assess the conservatism of the validity checkers, we vary the width of the narrow-passage environment, denoted Narr(`width'), and (iv) Random: A randomized environment with 10 obstacles with random width and height placed in random positions. The Scattered, Cluttered, and Narr(1.5) environments are shown in Fig.~\ref{fig:environments}.

We first consider a 4-D linear system from~\cite{summers2018distributionally}, subject to process noise and initial state uncertainty. We benchmark over four environment types spanning a diverse range of planning scenarios: (i) \textit{Scattered}, with relatively sparse obstacles and ample free space; (ii) \textit{Cluttered}, with dense obstacles and little free space; (iii) \textit{Narr(`width')}, a narrow-passage environment whose width we vary to assess the conservatism of the validity checkers; and (iv) \textit{Random}, containing 10 obstacles of random width, height, and position. The Scattered, Cluttered, and Narr(1.5) environments are shown in Fig.~\ref{fig:environments}.

\begin{figure*}[t]
    \centering
    ~~
    \begin{subfigure}[b]{0.22\linewidth}
        \includegraphics[width=\linewidth]{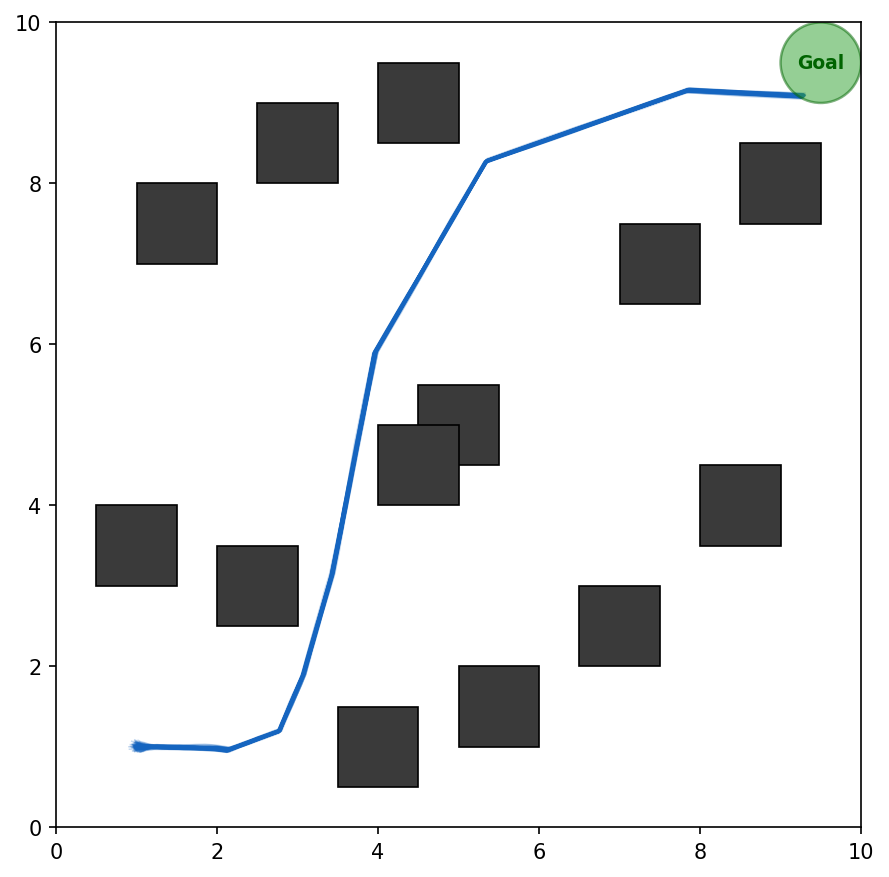}
        \caption{Scattered Obstacles.}
        \label{fig:scattered}
    \end{subfigure}
    ~~
    \begin{subfigure}[b]{0.22\linewidth}
        \includegraphics[width=\linewidth]{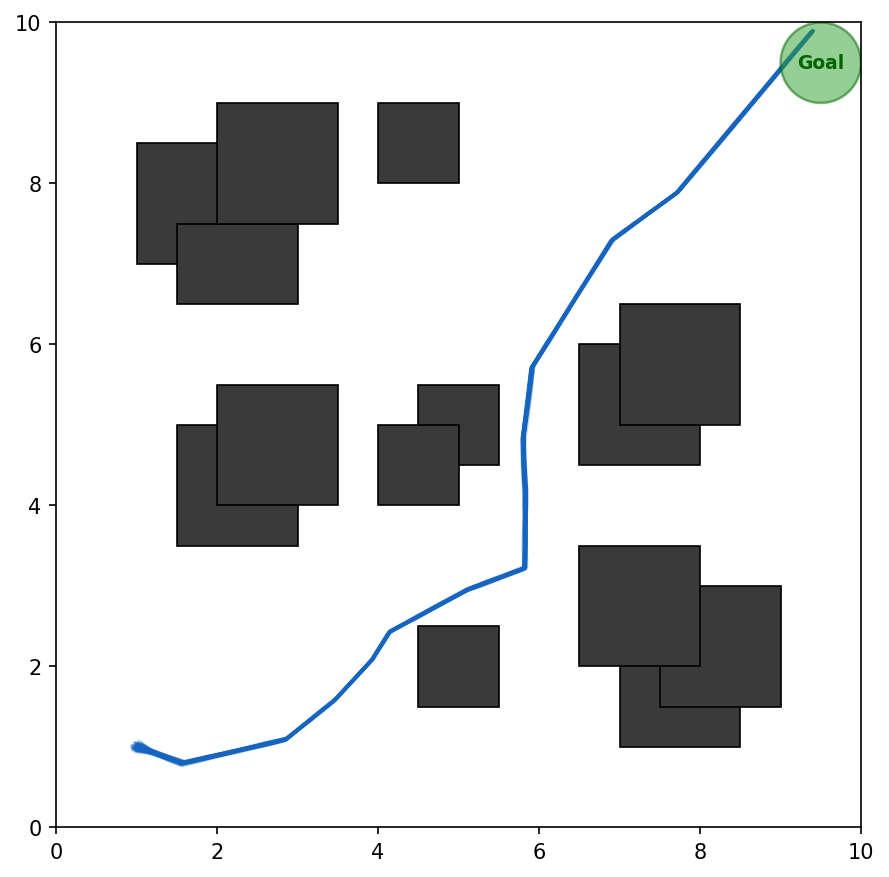}
        \caption{Cluttered Obstacles.}
        \label{fig:cluttered}
    \end{subfigure}
    ~~
    \begin{subfigure}[b]{0.22\linewidth}
        \includegraphics[width=\linewidth]{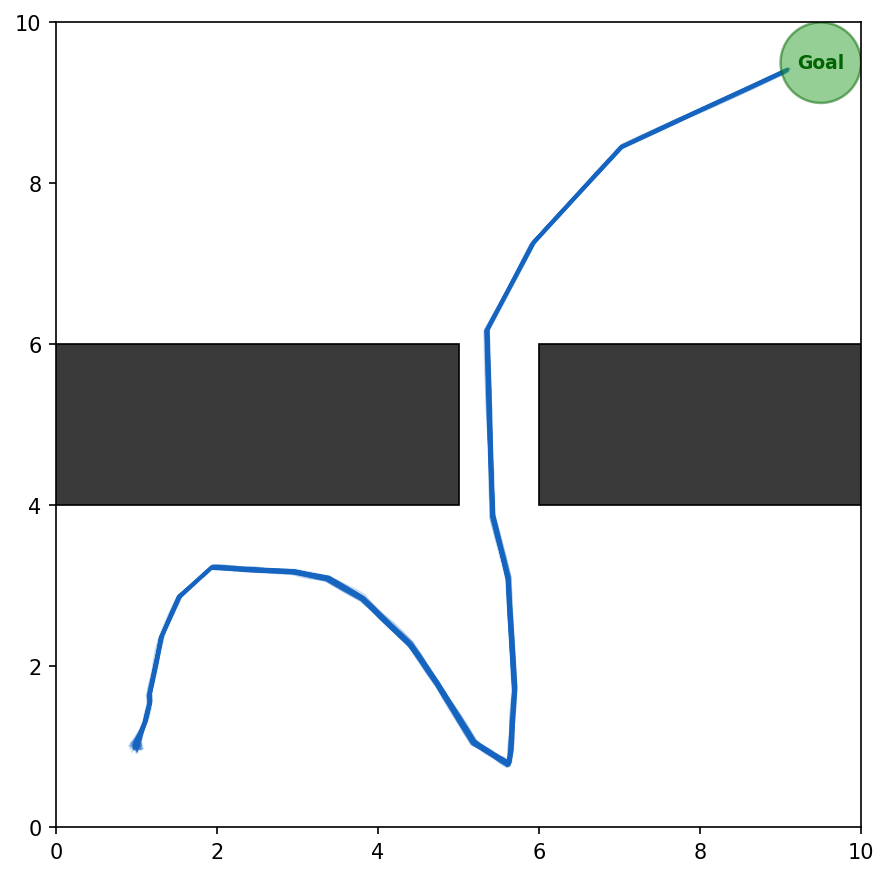}
        \caption{Narrow Passage Narr(1.0)}
        \label{fig:narrow}
    \end{subfigure}
    ~~
    \begin{subfigure}[b]{0.22\linewidth}
        \includegraphics[width=\linewidth]{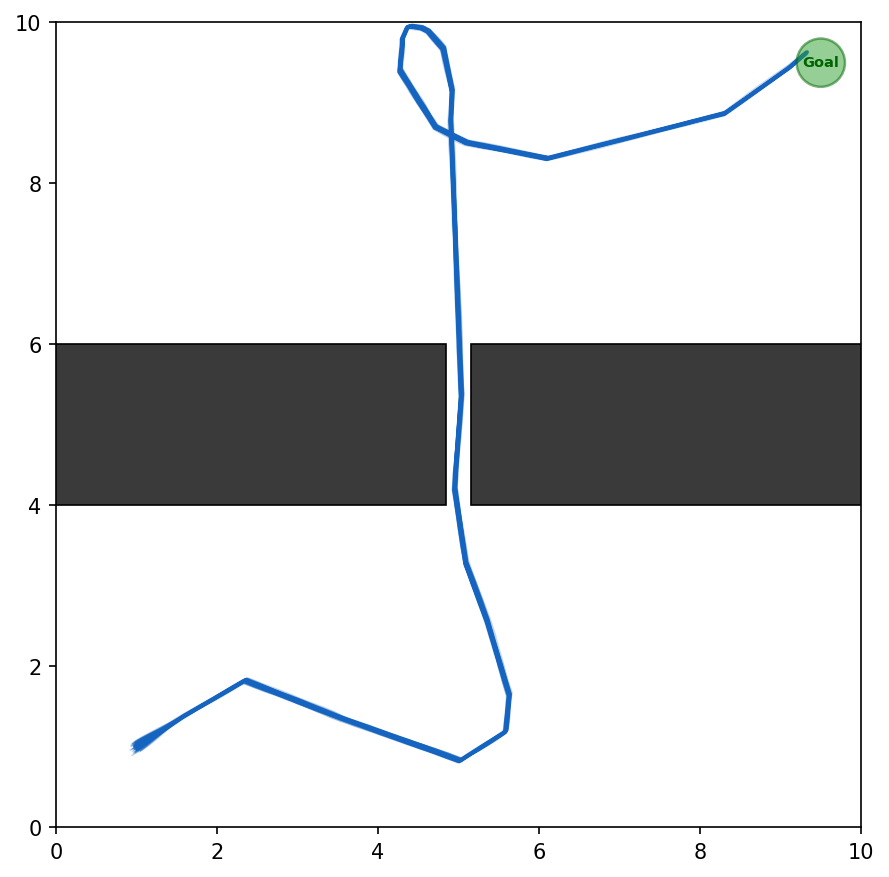}
        \caption{Narrow Passage Narr(0.3)}
        \label{fig:narrow_tight}
    \end{subfigure}
    \caption{Benchmark environments for 4-D Linear System. Obstacles are shown in black and the goal region is shown in green. Representative plots of Monte carlo simulation of solution trajectories are shown in blue for Gaussian noise.} 
    \label{fig:environments}
\end{figure*}
\begin{table*}[t]
\caption{Benchmark results for 4-D linear system over 100 runs.
  % Each environment is evaluated over 100 independent trials.
  `Succ' is the fraction of trials in which a valid path was found within 300\,s;
  `Time' is the mean planning time over successful trials; 
  `--' indicates no successful trial. Best success rate in each row is \textbf{bold}; among methods with equal success, the fastest is also \textbf{bold}.}
\label{tab:4d gauss noise}
\centering
%\scalebox{1}{
\begin{tabular}{l|c c|c c|c r|c c|c r|c r}
\toprule
        \multirow{2}{*}{Env.} & \multicolumn{2}{c|}{TS \cite{summers2018distributionally}}  &  \multicolumn{2}{c|}{Risk Assigned \cite{ekenberg2023distributionally}} & \multicolumn{2}{c|}{Particle-WDR-RRT} & \multicolumn{2}{c|}{Confidence-WDR-RRT} & \multicolumn{2}{c|}{Hybrid-WDR-RRT} & \multicolumn{2}{c}{Bandit-WDR-RRT} \\
        & Succ  & Time (s)& Succ  & Time (s) & Succ     & Time (s)   & Succ         & Time (s)        & Succ    & Time (s)  & Succ  & Time (s)\\
        \midrule
        \multicolumn{13}{c}{4D Linear System with Gaussian Noise}\\
        \midrule
Random    & 0.13 & 0.88 & 0.25 & 0.98 & 0.34 & 111.16 & 0.75 & 3.82 & 0.73 & 24.54 & \textbf{0.81} & \textbf{15.29} \\
Scattered & \textbf{1.00} & \textbf{0.01} & \textbf{1.00} & 1.40 & 0.52 & 154.04 & \textbf{1.00} & \textbf{0.01} & \textbf{1.00} & 6.87 & \textbf{1.00} & 4.52 \\
Cluttered & \textbf{1.00} & \textbf{0.01} & 0.97 & 7.14 & 0.45 & 173.79 & \textbf{1.00} & \textbf{0.01} & \textbf{1.00} & 11.18 & \textbf{1.00} & 6.81 \\
Narr(1.5) & \textbf{1.00} & 0.02 & \textbf{1.00} & 0.02 & 0.85 & 124.54 & \textbf{1.00} & \textbf{0.01} & \textbf{1.00} & 4.92 & \textbf{1.00} & 2.96 \\
% \hline
Narr(1.0) & \textbf{1.00} & 0.05 & \textbf{1.00} & 0.04 & 0.89 & 142.31 & \textbf{1.00} & \textbf{0.03} & \textbf{1.00} & 7.87 & \textbf{1.00} & 5.05 \\
% \hline
Narr(0.5) & \textbf{1.00} & 0.42 & \textbf{1.00} & 0.36 & 0.86 & 135.88 & \textbf{1.00} & \textbf{0.10} & \textbf{1.00} & 20.18 & \textbf{1.00} & 12.99 \\
% \hline
Narr(0.3) & 0.00 & -- & 0.00 & -- & 0.82 & 155.21 & \textbf{1.00} & \textbf{0.28} & \textbf{1.00} & 43.58 & \textbf{1.00} & 24.25 \\
% \hline
Narr(0.18) & 0.00 & -- & 0.00 & -- & 0.81 & 135.71 & 0.00 & -- & 0.86 & 108.14 & \textbf{0.97} & \textbf{83.80} \\
% \hline
              \midrule
              \multicolumn{13}{c}{4D Linear System with Non-Gaussian Noise}\\
            \midrule
Random    & 0.03 & 0.18 & 0.09 & 1.79 & 0.27 & 115.95 & 0.60 & 2.22 & 0.66 & 39.29 & \textbf{0.76} & \textbf{14.47} \\
Scattered & \textbf{1.00} & 0.02 & \textbf{1.00} & 1.31 & 0.41 & 190.00 & \textbf{1.00} & \textbf{0.01} & \textbf{1.00} & 14.46 & \textbf{1.00} & 8.28 \\
Cluttered & \textbf{1.00} & 0.01 & 0.95 & 8.94 & 0.34 & 178.80 & \textbf{1.00} & \textbf{0.01} & \textbf{1.00} & 14.00 & \textbf{1.00} & 11.24 \\
Narr(1.5) & \textbf{1.00} & 0.04 & \textbf{1.00} & 0.03 & 0.80 & 127.04 & \textbf{1.00} & \textbf{0.01} & \textbf{1.00} & 7.40 & \textbf{1.00} & 4.92 \\
% \hline
Narr(1.0) & \textbf{1.00} & 0.35 & \textbf{1.00} & 0.14 & 0.87 & 113.27 & \textbf{1.00} & \textbf{0.03} & \textbf{1.00} & 14.70 & \textbf{1.00} & 8.60 \\
% \hline
Narr(0.5) & 0.00 & -- & 0.00 & -- & 0.73 & 122.36 & \textbf{1.00} & \textbf{0.23} & 0.99 & 52.50 & \textbf{1.00} & 28.39 \\
% \hline
Narr(0.3) & 0.00 & -- & 0.00 & -- & \textbf{0.71} & \textbf{124.89} & 0.00 & -- & 0.33 & 116.57 & 0.33 & \textbf{90.91} \\
% \hline
\bottomrule
\end{tabular}
%}
\end{table*}

We use $N = 10^8$ samples, a safety probability threshold $p_{\text{safe}} = 0.01$, and a confidence parameter $\beta = 10^{-3}$. The time steps are selected as $\{\tau_j\}_{j=1}^J = \{0\text{--}11, 13\text{--}18, 20, 39\}$.

The distributions (which are unknown to the planner) for the initial state $P_0$ and process noise $P_w$ are zero-mean Gaussians with covariances:
% We evaluate and benchmark the algorithms in a $4$-D linear system from \cite{summers2018distributionally}, with $N = 10^8$ samples, $p_{\text{safe}} = 0.01$, $\{\tau_j\}_{j=1}^J = \{0\text{-}11, 13\text{-}18, 20, 39\}$,
% and $\beta = 10^{-3}$. The distributions of $P_0$ and $P_w$ are zero mean Gaussians with covariances 
%
\begin{align*}
    % \text{Cov}(P_0) = 
    10^{-3}\begin{bmatrix}
        1 & 0 & 0 & 0\\
        0 & 1 & 0 & 0\\
        0 & 0 & 0 & 0\\
        0 & 0 & 0 & 0
    \end{bmatrix},
    \quad
    % \text{Cov}(P_w) = 
    10^{-3}\begin{bmatrix}
        0 & 0 & 0 & 0\\
        0 & 0 & 0 & 0\\
        0 & 0 & 2 & 1\\
        0 & 0 & 1 & 2
    \end{bmatrix}
\end{align*}
Both $P_0$ and $P_w$ are truncated at $4$ standard deviations. 
Since the obstacles in~\cite{summers2018distributionally} are defined only in the workspace, it suffices to learn a single ambiguity tube for the 2-D projection of the state distribution. We achieve this using the formulation in Section~\ref{sec:lower_dimensional} with $L=1$.
% , where $M_1$ projects the state onto the workspace and $Y_\text{obs}^1 \subset \mathbb{R}^{\text{ws}}$ represents the workspace obstacles.
%
% for $P_0$ and $P_w$,
% respectively. We truncate $P_0$ and $P_w$ at $4$ standard deviations. Since \cite{summers2018distributionally} only considers workspace obstacles, it suffices to learn a single ambiguity tube for the $2$-dimensional projection of the state distribution onto the workspace. We achieve this by using the theory described in Section~\ref{sec:lower_dimensional} with $L=1$, $M_1$ being the matrix that projects the state into its workspace components and $Y_\text{obs}^1\subset \reals^\text{ws}$ representing the workspace obstacles.

We also consider the same system with non-Gaussian noise, obtained by pushing a bi-variate uniform distribution through the map
\begin{align*}
    (\omega_1,\omega_2) \mapsto 4\omega_1^{1/4}\bigg(10^{-3}\begin{bmatrix}
        2 & 1\\
        1 & 2\\
    \end{bmatrix}\bigg)^{1/2}\begin{bmatrix}
        \cos(2\pi\omega_2)\\
        \sin(2\pi\omega_2)
    \end{bmatrix}.
\end{align*}

\begin{figure}[t]
    \centering
    \begin{subfigure}{0.45\linewidth} 
        \includegraphics[width=\linewidth]{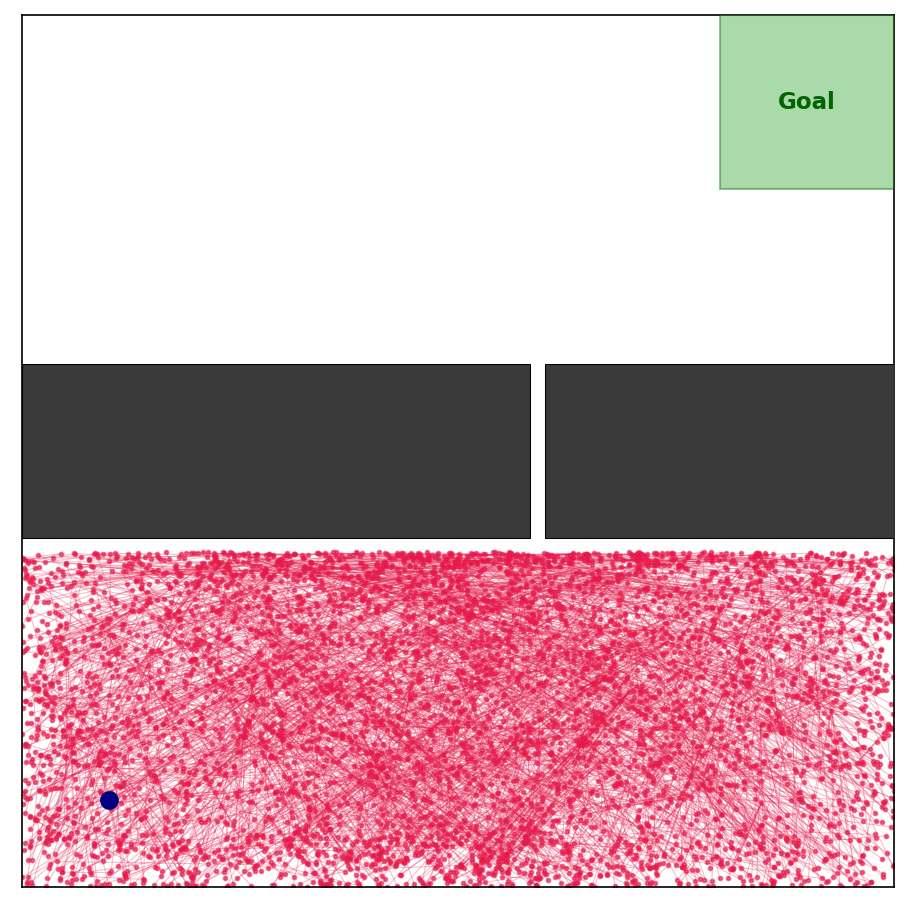}
        \caption{Moment-based Checker.}
        \label{fig:moment}
    \end{subfigure}
    \begin{subfigure}{0.45\linewidth}
        \includegraphics[width=\linewidth]{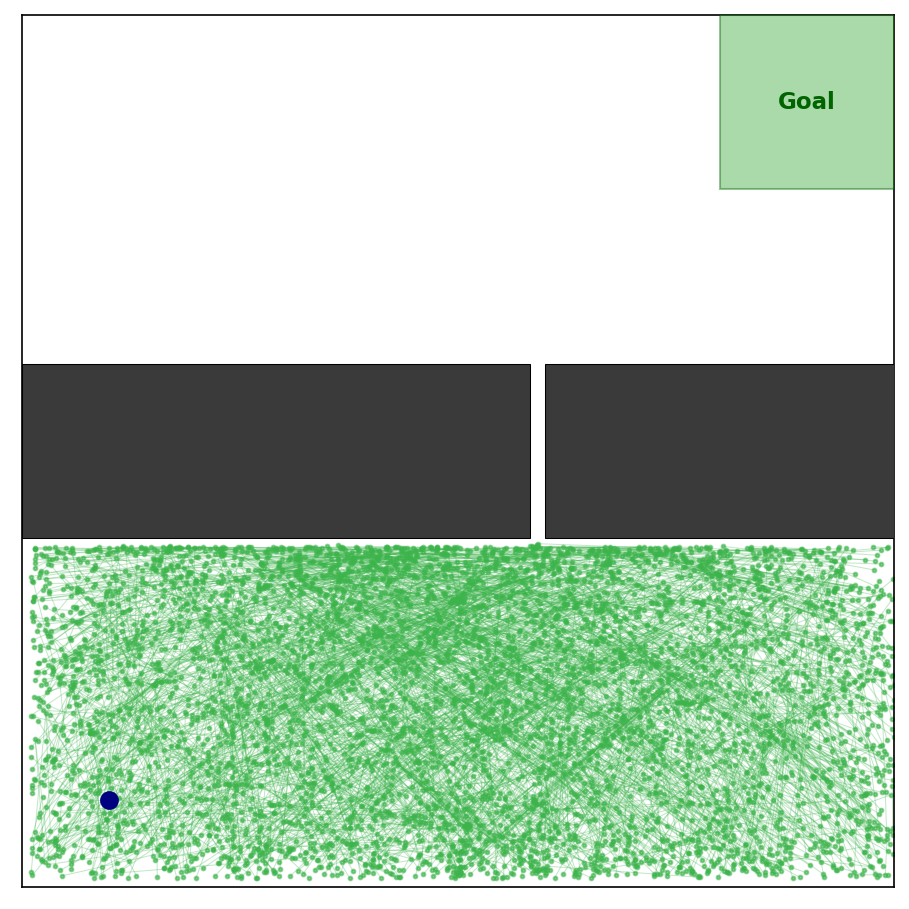}
        \caption{Ours: Confidence Checker.}
        \label{fig:confidencetube}
    \end{subfigure}
    \begin{subfigure}{0.45\linewidth}
        \includegraphics[width=\linewidth]{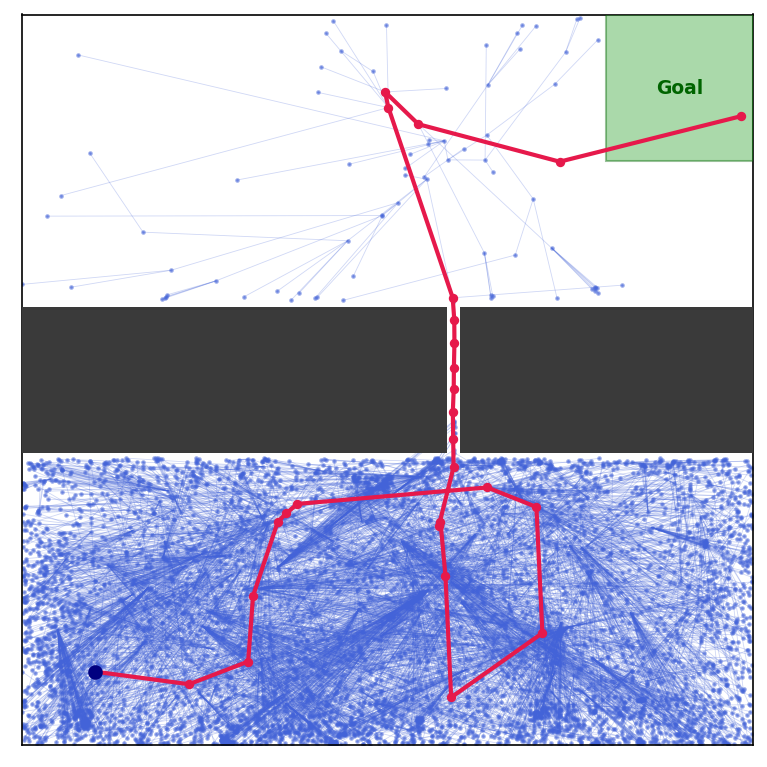}
        \caption{Ours: Bandit Checker.}
        \label{fig:ours}
    \end{subfigure}
    \caption{Plot of search trees after $60$ seconds. The number of nodes are 361240, 579358, and 30063, for Moment-based, Confidence, and Bandit Checkers, respectively. The Bandit Checker finds a solution with nominal trajectory in red.
    } 
    \label{fig:conservativeness}
\end{figure}

Table~\ref{tab:4d gauss noise} 
% and~\ref{tab:4d non-gauss noise} 
summarizes the results 
% mean computation times (capped at 300\,s) and success rates 
across all environments.

\paragraph{Comparison with baselines.}
The moment-based approaches~\cite{summers2018distributionally,ekenberg2023distributionally} exhibit two fundamental limitations that our method overcomes.
First, their validity checkers inflate obstacles using analytic bounds derived from the first two moments of the state distribution. This over-approximation renders them overly conservative in constrained geometries: both baselines fail at Narr(0.3) and Narr(0.18) under Gaussian noise, whereas Bandit-WDR-RRT achieves 1.00 and 0.97 success, respectively. 
Further, under non-Gaussian noise, the moment bounds used by both baselines become
overly conservative, where TS~\cite{summers2018distributionally}
achieves only 3\% success and Risk Allocation~\cite{ekenberg2023distributionally}
only 9\% on the Random environment, compared to 76\% for Bandit-WDR-RRT. More strikingly, both baselines fail at Narr(0.5) because the looser moment bounds over-approximate the state
distribution's support, making the passage appear infeasible. In contrast, our Wasserstein ambiguity tube learns a distribution-free confidence set directly from samples, capturing the actual shape of the noise distribution rather than relying on moment surrogates, thus avoiding unnecessary conservatism regardless of the true distribution.

It is important to note that our approach is inherently distributionally robust, meaning the true underlying noise distribution remains completely unknown to the planner. Because the ambiguity-tube constraints must guarantee safety for the worst-case distribution within the confidence set, our method naturally exhibits more conservative planning behaviors than would be observed if the exact distribution were known \textit{a priori}. However, unlike the moment-based baselines that suffer from geometric over-approximation, our conservatism is tight with respect to the ambiguity set while providing sound safety guarantees under unknown distributions.

In the easier environments (Scattered, Cluttered, wide narrow passages), where there is more free space and solutions are easier to find, the baselines achieve comparable success rates to our method and are faster due to their closed-form validity checks. This is expected: the analytic bounds are tight when the geometry is forgiving. However, our methods remain competitive even here: Confidence-WDR-RRT matches the baselines' speed at 0.01\,s in these cases.

\paragraph{Comparison among WDR-RRT variants.}
Among our proposed validity checkers, the Particle-based variant is the most general but incurs high per-node evaluation cost (110--190\,s mean planning time), making it impractical as a standalone checker. The Confidence-Region checker is fast (comparable to the baselines in easy environments) but inherits some conservatism from its geometric over-approximation, failing at Narr(0.18) under Gaussian noise and Narr(0.3) under non-Gaussian noise. The Hybrid checker alleviates this by falling back to particle evaluation near obstacles, but remains slower than necessary. Bandit-WDR-RRT offers the best balance: it adaptively allocates computational effort, matching or exceeding the success rate of all other variants while being 1.5--3$\times$ faster than Hybrid-WDR-RRT across environments.

For each of the runs, the planned trajectories are validated via Monte Carlo simulation, with the found solution paths achieving between 99-100\% success rate across trials, confirming that the ambiguity-tube constraints provide sound safety guarantees. 
Example trajectories for the Gaussian noise case are shown in Fig.~\ref{fig:environments}.

Fig.~\ref{fig:conservativeness} visualizes the search trees grown after 60\,s in Narr(0.18) under Gaussian noise. The moment-based checker produces trajectories that stays far from all obstacles and never enters the narrow passage. The Confidence-Region checker allows nodes closer to the walls but cannot go through the gap  to reach the goal within the time limit. The Bandit checker successfully explores the passage and reaches the goal region, consistent with its 0.97 success rate in Table~\ref{tab:4d gauss noise}.

% \subsubsection{Conservativeness of each technique}

% \begin{figure}[htbp]
%     \centering
%     \begin{subfigure}[b]{0.31\textwidth}
%         \includegraphics[trim={24cm 4cm 24cm 4cm}, clip, width=\linewidth]{moment_results.png}
%         \caption{Moment-based Checker.}
%         \label{fig:moment}
%     \end{subfigure}
%     \hfill
%     \begin{subfigure}[b]{0.31\textwidth}
%         \includegraphics[trim={24cm 4cm 24cm 4cm}, clip, width=\linewidth]{confidence.png}
%         \caption{Ours with Confidence Checker.}
%         \label{fig:confidencetube}
%     \end{subfigure}
%     \hfill
%     \begin{subfigure}[b]{0.31\textwidth}
%         \includegraphics[trim={24cm 4cm 24cm 4cm}, clip, width=\linewidth]{bandit_result.png}
%         \caption{Ours with Bandit Checker.}
%         \label{fig:ours}
%     \end{subfigure}
%     \caption{Plot of nodes in tree after $60$ seconds.}
%     \label{fig:conservativeness}
% \end{figure}

\subsection{8-D Drone System.}

Finally, we show the scalability of our approach during planning with a 8-D drone system navigating on a 2D plane, i.e., $n = 8$ and $n_\text{ws} = 2$, and with $n_u = 2$. In this example, we use the Stable Sparse RRT (SST) algorithm \cite{SST}. The model is taken from \cite{hakobyan2020wasserstein}. The noise distribution is Gaussian with zero mean and covariance via $G$, a $4 \times 2$ matrix mapping the 2D disturbance $\mathbf{w}$ into the 4D state:
\begin{align*}
    G = \begin{bmatrix} 0 & 0 & a & b \\ 0 & 0 & b & a \end{bmatrix}^\top,
    \quad
    a = \frac{\sqrt{3}+1}{2\sqrt{1000}},
    \quad
    b = \frac{\sqrt{3}-1}{2\sqrt{1000}},
\end{align*}
so that only the velocity states $(\dot{x}, \dot{y})$ are directly perturbed; the position states $(x, y)$ are unaffected. The state-space covariance is then given by
% \begin{align*}
%     W = GG^{\top} &=
%     \begin{bmatrix}
%         0 & 0 & 0 & 0 \\
%         0 & 0 & 0 & 0 \\
%         0 & 0 & a^{2}+b^{2} & 2ab \\
%         0 & 0 & 2ab & a^{2}+b^{2}.
%     \end{bmatrix}
% \end{align*}
$W = GG^{\top}$, a $4\times 4$ matrix.

We set the probability of safety to $p_\text{safe} = 0.02$. In this case study, we consider both collision avoidance constraints with respect to the workspace obstacles and also control constraints, where we let $U = \{u\in\reals^m : \|u\| \le 0.025\}$. We also let the goal region be a physical location in the 2D workspace.

Note that, although the system is high-dimensional, the goal and obstacles are related to the lower-dimensional spaces $\reals^\text{ws}$ and $\reals^{n_u}$. Because of this, we make use of the approach described in Section~\ref{sec:lower_dimensional} and learn two lower-dimensional ambiguity tubes: one for $M_{1\#} P_t^e$, related to the position of the drone and another for $-K_\# P_t^e$, related to the control effort. To this end and, as described in Example~\ref{ex:example}, we first express $X_\text{obs}$ and $X_\text{goal}$ as required in Assumption~\ref{ass:obs_goal}. We learn the tubes with $N = 10^8$, $\{\tau_j\}_{j=1}^J = 
% \{0, 1, 2, 3, 4, 5, 6, 7, 8, 9, 10, 11, 13, 14, 15, 16, 17, 18, 20, 39\}
\{0\text{-}11, 13\text{-}18, 20, 39\}
$ and $\beta = 10^{-3}$.

\begin{figure}[t]
    \centering
    \begin{subfigure}{0.49\linewidth}
        \centering
        \includegraphics[width=.9\linewidth]{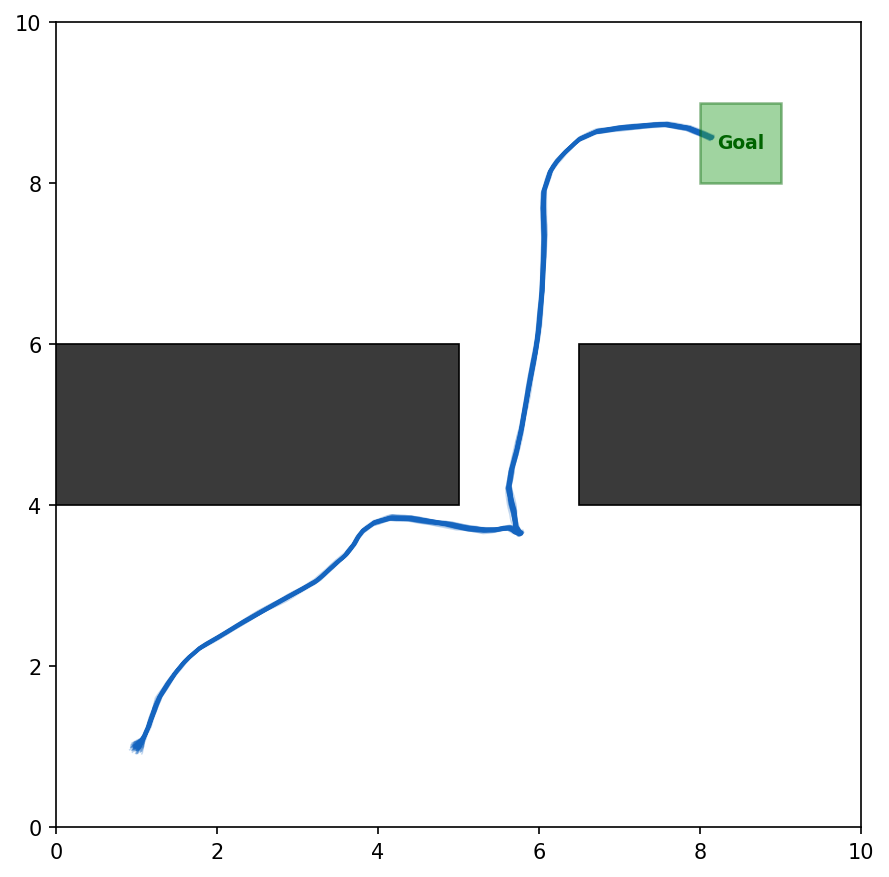}
        \caption{Narrow Passage Narr(1.5).}
        \label{fig:quadcopter1}
    \end{subfigure}
    \begin{subfigure}{0.49\linewidth}
        \centering
        \includegraphics[width=.9\linewidth]{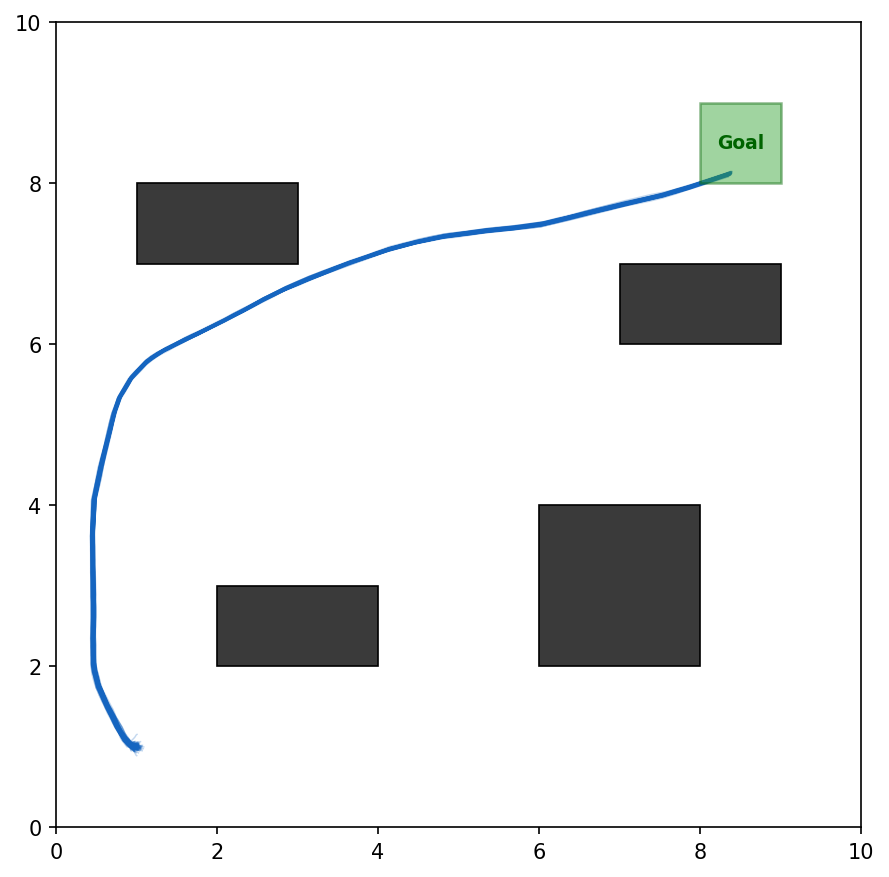}
        \caption{Scattered Rectangle Obstacles.}
        \label{fig:quadcopter2}
    \end{subfigure}
    \begin{subfigure}{0.6\linewidth}
        \centering
        \includegraphics[width=0.74\linewidth]{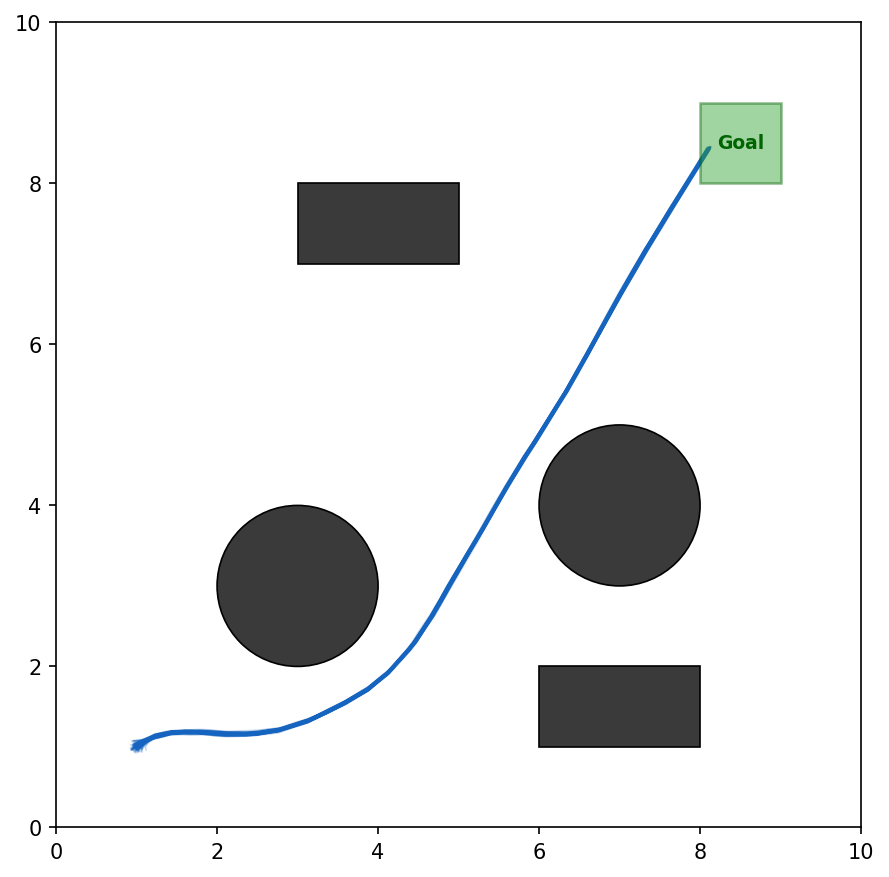}
        \caption{Circle and Rectangle Obstacles.}
        \label{fig:quadcopter3}
    \end{subfigure}
    \caption{Environments for the 8-D drone system case study with Monte Carlo simulations of example solution trajectories.} 
    \label{fig:quadcopter}
\end{figure}
\begin{table}[ht!]
    \caption{8-D quadcopter system with Gaussian noise. Each scenario is run with the SST planner using the Bandit validity checker, stopping at first solution (60\,s timeout). Success rate and mean planning time over successful trials are reported.}
    \label{tab:8d quadcopter}
    \centering
    \scalebox{1}{
    \begin{tabular}{l | cc}
        \toprule
                  \multirow{2}{*}{Env.} & \multicolumn{2}{c}{Bandit-WDR-SST}\\
                  & Success Rate & Time (s) \\
                  \midrule
    Fig.~\ref{fig:quadcopter1} & 1.00 & $9.89 \pm 7.54$ \\
    Fig.~\ref{fig:quadcopter2} & 1.00 & $4.29 \pm 3.00$ \\
    Fig.~\ref{fig:quadcopter3}  & 1.00 & $4.84 \pm 2.25$ \\
    \bottomrule
    \end{tabular}
    }
\end{table}

% \begin{table}[ht!]
% \caption{8D quadcopter system}
% \label{tab:8d quadcopter}
% \centering
% \begin{tabular}{l|c|c}
%     \toprule
%               \multirow{2}{*}{Env.} & \multicolumn{2}{c}{Confidence-WDR-RRT}\\
%               & Success Rate & Time (s) \\
%               \midrule
% Environment 1 (Circular) & 1.0 & $2.23 \pm 0.639$\\
% Environment 2 (Scattered) & 0.92 & $11.30 \pm 3.58$\\
% Enviroment 3 (Narrow 1.5) & \\
% \bottomrule
% \end{tabular}
% \end{table}

% \begin{table}[ht!]
% \caption{8D quadcopter system}
% \label{tab:8d quadcopter}
% \centering
% \begin{tabular}{l|c c|c c}
%     \toprule
%               \multirow{2}{*}{Env.} & \multicolumn{2}{c|}{Bandit-WDR-RRT} & \multicolumn{2}{c}{SST} \\
%               & Success Rate & Time (s) & Success Rate & Time (s) \\
%               \midrule
% Environment 1 (Circular)   & 1.00 & $2.23 \pm 0.64$  & 1.00 & $60.08 \pm 0.04$ \\
% Environment 2 (Scattered)  & 0.92 & $11.30 \pm 3.58$ & 1.00 & $60.06 \pm 0.03$ \\
% Environment 3 (Narrow 1.5) &      &                  & 1.00 & $60.07 \pm 0.02$ \\
% \bottomrule
% \end{tabular}
% \end{table}

Table~\ref{tab:8d quadcopter} and Fig.~\ref{fig:quadcopter} show that Bandit-WDR-SST finds valid paths in all three environments, including the narrow passage, with planning times well under 60\,s. The Monte Carlo simulations of the trajectories show a 100\% success rate. The ability to handle an 8-D system with two separate lower-dimensional ambiguity tubes shows that the decomposition in Section~\ref{sec:lower_dimensional} directly enables tractable planning in higher-dimensional systems.

\section{Conclusion}

% In this work, we presented a probabilistically complete tree-based framework to perform motion planing for linear or (feedback linearizable) systems with additive random disturbances of unknown distribution. Our approach leverages observations from previous trajectories of the system to learn an ambiguity tube that contains the distribution of the system's state at all times, and uses this tube to obtain safe paths. We also presented a second algorithm that relies on learning several lower-dimensional tubes in order to reduce the sample and computational complexity of our approach. We also proposed a bandit-based validity checker, which smartly chooses to either employ a precise but slow checker or a conservative but fast one. Finally, our empirical evaluation showcases the effectiveness of our approach in finding safe paths and its superiority with respect to state-of-the-art approaches.

In this work, we presented a probabilistically complete tree-based framework for motion planning of linear and feedback-linearizable systems with additive random disturbances of unknown distribution. Our approach uses observations from previous trajectories to learn an ambiguity tube containing the state distribution at all times, then leverages this tube to find safe paths. We also presented a variant that learns several lower-dimensional tubes to reduce sample and computational complexity, and a bandit-based validity checker that adaptively chooses between a precise but slow checker and a conservative but fast one. Empirical evaluation demonstrates the effectiveness of our approach in finding safe paths and its superiority over state-of-the-art methods.

% \begin{credits}
% \subsubsection{\ackname} 
% \ml{complete this}
% This study was funded
% by X (grant number Y)
% .

% \subsubsection{\discintname}
% \ml{It is now necessary to declare any competing interests or to specifically
% state that the authors have no competing interests. Please place the
% statement with a bold run-in heading in small font size beneath the
% (optional) acknowledgments,
% for example: The authors have no competing interests to declare that are
% relevant to the content of this article. Or: Author A has received research
% grants from Company W. Author B has received a speaker honorarium from
% Company X and owns stock in Company Y. Author C is a member of committee Z.}
% \end{credits}

\appendix
\section{Appendix}
\label{sec:appendix}

\iffalse
For the proofs of this appendix we rely in the notion of \emph{optimal transport cost} and its properties. \IG{Maybe we don't need optimal transport} 
%
\begin{definition}[Optimal Transport Cost]
\label{dfn:optimal_transport_cost}
Let the space 
$X \subseteq \reals^n$ be equipped with the euclidean norm, $M\in\reals^{m\times n}$ and define the cost function $c:X\rightarrow \reals_{\ge0}$, with $c(x) := \|Mx\|$. The optimal transport discrepancy between distributions $P,P'\in \mathcal{D}(X)$ is defined as
%
\begin{align*}
    \mathcal{T}_{c}(P,P') = \inf_{\pi\in\Pi(P,P')}\int_{X\times X} c(x-x')d\pi(x,x'),
\end{align*}
% 
where $\Pi(P,P')$ is the set of probability distributions on $\mathcal{D}(X\times X)$ with marginals $P$ and $P'$.
\end{definition}
%
Note that, when $M$ is the identity matrix, $\mathcal{T}_{c}(P,P') = \mathcal{W}(P,P')$.
\fi

For the proofs in this appendix we rely on the following four properties of the Wasserstein distance:
%
\iffalse
\begin{assumption}
\label{assumption:cost}
    Let $p\in\mathbb N$. The cost $c:\mathbb R^n \times \mathbb R^n \rightarrow \mathbb R_{\ge 0}$ can be written as $c(x_1, x_2) = \|M(x_1-x_2)\|^p$ for some $n\times n$ matrix $M$.
\end{assumption}
%
Taking Assumption~\ref{assumption:cost} into account, from now on we consider costs of the form $c:X\longrightarrow \mathbb R$, and write $c(x-x') := c(x,x')$.
\fi
%
\begin{proposition}
    Let %$c:X\longrightarrow \mathbb R$ be the transportation cost in Definition~\ref{dfn:optimal_transport_cost} and 
    $P,P',P''\in\mathcal{D}(X)$ and $A\in\reals^{n\times n}$. Then
    \begin{align*}
        \mathcal{W}(P*P'',P'*P'') &\le \mathcal{W}(P,P'),\\
        \mathcal{W}(P,P'*P'') &\le \mathcal{W}(P,P') + \int_X \|x\| dP''(x),\\
         \mathcal{W}(A_\#P, A_\#P') &\le \|A\| \mathcal{W}(P, P'),\\
        \mathcal{W}(P,P') &\le \mathcal{W}(P,P'') + \mathcal{W}(P'',P').
    \end{align*}
\end{proposition}
\begin{proof}
    The first three statements are proved in \cite[Prop. 10]{aolaritei2022uncertainty}, \cite[Lemma 16]{boskos2023high} and \cite{villani2021topics}, respectively, and the last one is the triangle inequality of the Wasserstein distance \cite{villani2021topics}.
\end{proof}
%
\iffalse
\begin{proposition}
    Let $c:X\longrightarrow \mathbb R$ be the transportation cost in Definition~\ref{dfn:optimal_transport_cost}, $A\in\reals^{n\times n}$ and $P,P'\in\mathcal{D}(X)$. Then it holds that \cite[Prop. 3]{aolaritei2022uncertainty},
    %
    \begin{align*}
        \mathcal{T}_c(A_\#P,A_\#P') &= \mathcal{T}_{c\circ A}(P, P'). 
    \end{align*}
    %
    Furthermore, if $c(x) = \|x\|$ \cite{villani2021topics},
    %
    \begin{align*}
        \mathcal{T}_{c\circ A}(P, P') &\le \|A\| \mathcal{W}(P, P').
    \end{align*}
\end{proposition}
\fi

\begin{proof}[Proof of Prop.~\ref{prop:ambiguity_tube_error}]
    Since 
    %
    %\begin{align*}
        $\mathcal{W}(P_t, \widehat P_t) = \mathcal{W}(P_t^e   * \delta_{\bar x_t}, \widehat P_t^e  * \delta_{\bar x_t}) \le \mathcal{W}(P_t^e, \widehat P_t^e) \le \varepsilon_t$, 
%    \end{align*}
    %
    then $P_t \in \mathcal{P}_t$.
\end{proof}
\begin{proof}[Proof of Lemma~\ref{lemma:data_driven_ambiguity_set}]
    Let $\beta_1, \beta_2 \in (0,1)$. First, Hoeffding's inequality bounds the $q$-th moment of $P$:
    \vspace{-0.25cm}
\begin{align}
\label{eq:beta1}
    P^N\big[ \mathcal{M}_q(P) \le \widehat{\mathcal{M}}_q(P)\big] \ge 1 - \beta_1,
\end{align}
Next, we use McDiarmid's inequality. Let $x_1, \dots, x_i, x_i', x_{i+1}, \dots, x_N \in \mathbb R^d$, $\widetilde P$ be the empirical distribution of $\{x_j\}_{j = 1}^N$, and $\widetilde P'$ differ from $\widetilde P$ only on $x_i'$. The symmetry and triangle inequality of the Wasserstein distance yield $|\mathcal{W}(P, \widetilde P) - \mathcal{W}(P, \widetilde P') | \le \mathcal{W}(\widetilde P, \widetilde P')$. Furthermore, by geometric inspection, $\mathcal{W}(\widetilde P, \widetilde P')$ is maximized when all $\{x_j\}_{j=1}^N$ overlap and $x_i'$ is located diametrically opposite, implying that $\mathcal{W}(\widetilde P, \widetilde P') \le \phi/N$, for all $x_1, \dots, x_i, x_i', x_{i+1}, \dots, x_N \in \mathbb R^d$. By McDiarmid's inequality,
    %
    % \begin{align}
    % \label{eq:mcdiarmid}
    % {\small
        $P^N\Big[\mathcal{W}(P, \widehat P) - \mathbb E[\mathcal{W}(P, \widehat P)] \ge t \Big] \le \exp{(-2Nt^2/\phi^2)} =: \beta_2$.
    %     }
    % \end{align}
    %
    %Letting $\beta_2$ equal the right-hand side of \eqref{eq:mcdiarmid} we obtain that, with probability no less than $1 - \beta_2$ over the choice of the $N$ samples,
% %
%     \begin{align}
%     \label{eq:beta2}
%         \mathcal{W}(P, \widehat P) \le g(d, q, \mathcal{M}_q(P), N) + \phi\sqrt{\frac{\log(1/\beta_2)}{2N}}.
%     \end{align}
    %
We conclude the proof by combining this expression with \eqref{eq:beta1} via the union bound, and noting that $\mathcal{M}_q(P)\mapsto g(d, q, \mathcal{M}_q(P), N)$ increases monotonically.
\end{proof}
\begin{proof}[Proof of Lemma~\ref{lemma:ambiguity_dynamics}]
   To prove the first statement, let $t < \tau$.
\begin{equation*}
%\label{eq:ambiguity_radius2}
\begin{split}
    &\mathcal{W}(\widehat P_\tau^e, P_t^e) \le \mathcal{W}(\widehat P_\tau^e, P_\tau^e) + \mathcal{W}(P_\tau^e, P_t^e)\\
    &\le \varepsilon_\tau + \mathcal{W}(A_{\text{cl},\#}^\tau P_0, \Conv_{i=0}^{\tau-1} (A_\text{cl}^{\tau-1-i}G)_\#P_w,\\
    &A_{\text{cl},\#}^t P_0 \Conv_{i=0}^{t-1} (A_\text{cl}^{t-1-i}G)_\#P_w)\\
    &\le \varepsilon_\tau + \mathcal{W}(A_{\text{cl},\#}^\tau P_0 \Conv_{i=0}^{\tau-1} (A_\text{cl}^{\tau-1-i}G)_\#P_w,\\
    &A_{\text{cl},\#}^t P_0 \Conv_{i=0}^{\tau-1} (A_\text{cl}^{\tau-1-i}G)_\#P_w)\\
    &+ \mathcal{W}(A_{\text{cl},\#}^t P_0 \Conv_{i=0}^{\tau-1} (A_\text{cl}^{\tau-1-i}G)_\#P_w,\\
    &A_{\text{cl},\#}^t P_0 \Conv_{i=0}^{t-1} (A_\text{cl}^{t-1-i}G)_\#P_w)\\
    &\le \varepsilon_\tau + \mathcal{W}(A_{\text{cl},\#}^\tau P_0, A_{\text{cl},\#}^t P_0) + \\
    &  + \mathcal{W}(\Conv_{i=0}^{\tau-1} (A_\text{cl}^{\tau-1-i}G)_\#P_w, \Conv_{i=0}^{t-1} (A_\text{cl}^{t-1-i}G)_\#P_w)\\
    &\le \varepsilon_\tau +  \int_{\mathbb R^n \times \mathbb R^n} \|x-y\| d (A_\text{cl}^\tau\times A_{\text{cl}}^t)_\# P_0(x,y)\\
    &+ \mathcal{W}(\Conv_{i=t}^{\tau-1} (A_\text{cl}^iG)_\#P_w, \delta_0)\\
    &\le \varepsilon_\tau + \int_{\mathbb R^n} \|A_\text{cl}^\tau x-A_\text{cl}^t x\| dP_0(x) + \mathcal{W}(\Conv_{i=t}^{\tau-1} (A_\text{cl}^iG)_\#P_w, \delta_0)\\
    &\le \varepsilon_\tau + \|A_\text{cl}^\tau - A_\text{cl}^t\| \mathcal{M}_1(P_0) + \mathcal{M}_1(P_w)\sum_{i=t}^{\tau-1} \|A_\text{cl}^iG\|,
\end{split}
\end{equation*}
where in the fifth inequality we relied on the fact that $(A_\text{cl}^\tau \times A_{\text{cl}}^t)P_0$ is a feasible coupling between $A_{\text{cl},\#}^\tau P_0$ and $A_{\text{cl},\#}^t P_0$.
Similarly, 
%
%\begin{align*}
%\label{eq:ambiguity_radius3}
    $\mathcal{W}(\widehat P_\tau^e, P_t^e) \le \varepsilon_\tau + \|A_\text{cl}^t - A_\text{cl}^\tau\| \mathcal{M}_1(P_0) + \mathcal{M}_1(P_w)\sum_{i=\tau}^{t-1} \|A_\text{cl}^iG\|$ for $t \ge \tau$. 
%\end{align*}
%
Next we prove the second statement by bounding $f_\tau(t)$.
%is dominated by a sequence $(\epsilon_\tau)_{\tau\in\naturals_0} \subset \reals$ with $\lim_{\tau\to\infty} \epsilon_\tau = \varepsilon$.
Let $\tau \in \naturals_0$ and $t \ge \tau$, and note that $\|A_\text{cl}^t - A_\text{cl}^\tau\| \le \max_{t' \ge 0}\|A_\text{cl}^{t'} - I\|\| A_\text{cl}^\tau\|,\quad\forall t\ge \tau$. We now focus on the third term on the right-hand side of \eqref{eq:ambiguity_radius}. By the Jordan-Chevalley decomposition of $A_\text{cl}$, there exists an invertible $V\in\reals^{n\times n}$ such that $A_\text{cl} = V(\Lambda + N)V^{-1}$, with $\Lambda$ diagonal and sharing eigenvalues with $A_\text{cl}$, $N$ nilpotent and $A_\text{cl}N = NA_\text{cl}$. Newton's binom yields that, for all $i \in \naturals_0$,
\begin{align*}
    %\label{eq:aux_proof1}
    %\begin{split}
    \|A_\text{cl}^iG\| &\le \|G\| V(\Lambda+ N)^iV^{-1}\| \le \|G\|\|V\|\|V^{-1}\| \|(\Lambda+ N)^i\|\\
    &\le \|G\|\|V\|\|V^{-1}\| \|\sum_{k = 0}^n \binom{i}{k} \Lambda^{i - k} N^k\|\\
    &\le \|G\|\|V\|\|V^{-1}\| \bigg[ \sum_{k = 0}^n \bigg(\frac{e}{k}\bigg)^k\|\Lambda^{ - k}N^k\| \bigg] \|\Lambda^i\|i^n.
    %\end{split}
\end{align*}
Therefore,
\begin{align*}
% \label{eq:aux_proof2}
    % \begin{split}
    f_\tau(t) \le \varepsilon_\tau + \max_{t' \ge 0}\|A_\text{cl}^{t'} - I\|\| A_\text{cl}^\tau\|\mathcal{M}_p(P_0) +\\
    + \mathcal{M}_p(P_w)\|G\|\|V\|\|V^{-1}\| \bigg[\sum_{k = 0}^n \bigg(\frac{e}{k}\bigg)^k\|\Lambda^{ - k}N^k\| \bigg] \sum_{i=\tau}^\infty \|\Lambda^i\|i^n,
    % \end{split}
\end{align*}
where $\|\Lambda^i\| = |\lambda|^i$ for some eigenvalue $\lambda$ of $\Lambda$. Since $A_\text{cl}$ is stable, $|\lambda| < 1$, and thus the infinite series converges. The proof is concluded by observing that if $\tau$ is big enough, then the right-hand side becomes arbitrarily small for all $t \ge \tau$.
\end{proof}

\begin{proof}[Proof of Thm.~\ref{thm:ambiguity_tube}]
    Let $j \in \{1,\dots,J\}$ and $I_j$ denote the set of time steps $t$ for which the ambiguity set $\mathbb B(\widehat P_t^e, \varepsilon_t)$ has been derived from the data-driven one $\mathbb B(\widehat P_{\tau_j}^e, \varepsilon_{\tau_j})$. By Lemma~\ref{lemma:ambiguity_dynamics}, if $P_{\tau_j}^e \in \mathbb B(\widehat P_{\tau_j}^e, \varepsilon_{\tau_j})$, then $P_t^e \in \mathbb B(\widehat P_t^e, \varepsilon_t)$ for all $t\in I_j$. By construction, $P_{\tau_j}^e \in \mathbb B(\widehat P_{\tau_j}^e, \varepsilon_{\tau_j})$ with confidence $1-\beta/J$. Then $P_t^e \in \mathbb B(\widehat P_t^e, \varepsilon_t)$ for all $t\in I_j$ with confidence $1-\beta/J$. Since this result holds for every $j\in\{1,\dots,J\}$, an application of the union bound yields an overall confidence of $1-J\beta/J = 1-\beta$, which concludes the proof.
\end{proof}

\begin{proof}[Proof of Lemma~\ref{lemma:confidence_balls}]
    Let $j\in\{1, \dots, J\}$. Then, if $P_{\tau_j}^e \in \mathcal{P}_{\tau_j}$, then, Lemma~\ref{lemma:ambiguity_dynamics} gives $P_t^e \in \mathbb B(\widehat P_{\tau_j}^e, f_{\tau_j}(t)) \subseteq \mathbb B(\widehat P_{\tau_j}^e, \bar \varepsilon_t)$ for all $t\in I_j$. Since this is the case for all $j\in\{1, \dots, J\}$, then $P_t^e \in \mathbb B(\widehat P_{\tau_j}^e, \bar \varepsilon_t)$ and
    $P_t^e(S^e_t) > p_\text{safe}$ for all $t\in\naturals_0$. Since $P_{\tau_j}^e \in \mathcal{P}_{\tau_j}$ for all $j\in\{1, \dots, J\}$ holds with confidence of $1-\beta$, the previous result holds with the same confidence..
\end{proof}

\begin{proof}[Proof of Thm.~\ref{thm:ambiguity_tubes_low_dimensional}]
    Let $l \in \{1,\dots,L\}$. By construction of the $l$-th tube, following the same reasoning used in Thm.~\eqref{thm:ambiguity_tube}, but relying on Lemma~\ref{lemma:ambiguity_dynamics_low_dimensional} instead of Lemma~\ref{lemma:ambiguity_dynamics}, we obtain that $M_{l\#}P_t^e \in \mathcal{P}_{l,t}^e, \:\forall\:t\in\mathbb N_0$ with confidence $1 - \beta/L$. Therefore, by the union bound we obtain the desired result.
\end{proof}

\begin{proof}[Proof of Lemma~\ref{lemma:confidence_balls_lower_dimensional}]
    Let $l\in\{1, \dots, L\}$ and assume that $M_{l\#}P_t^e \in \mathcal{P}_{l, t}^e$ for all $t\in\naturals_0$. Then, by the same reasoning in the proof of Lemma~\ref{lemma:confidence_balls}, we obtain that $M_{l\#}P_t^e( S_{l,t}^e ) > p_\text{safe}/L$ for all $t\in\naturals_0$. Since the the previous assumption holds with confidence $1-\beta/L$ for each $l\in\{1, \dots, L\}$, an application of the union bound yields the desired result.
\end{proof}
\begin{proof}[Proof of Thm.~\ref{thm:soundness}]
    Without loss of generality let the control constraints in \eqref{eq:cc_constraints} be embedded into the collision-avoidance constraints. Let $(\mathcal{P}_t^e)_{t = 0}^T$ and $(S_t^e)_{t = 0}^T$ be the error ambiguity and confidence tubes obtained through Alg.~\ref{alg:tube} Alg.~\ref{alg:confidence_balls} respectively. Assume that $P_t^e \in \mathcal{P}_t^e$ and $P_t^e[\textbf e_t \in S^e_t] > p_\text{safe}$ for all $t\in\naturals_0$ and let $t \in \{0, \dots, T\}$. Since every node in the motion plan is valid, the following conditions holds:
    \begin{align}
    \label{eq:aux_proof_thm3}
            {\small \min_{P\in\mathcal{P}_t} P\big[\textbf  x_t \notin X_\text{obs}\big] > p_\text{safe} \;\;} \text{or}\;\; {\small X_\text{obs}\bigcap\big( S^e_t + \bar x_t \big)  = \emptyset.}
    \end{align}
    If the first condition holds and noting that $P_t^e \in \mathbb B(\widehat P_t^e, \varepsilon_t)$ is equivalent to $P_t \in \mathcal{P}_t = \mathbb B(\widehat P_t^e*\delta_{\bar x_t}, \varepsilon_t)$ by Prop.~\ref{prop:ambiguity_tube_error}, then $P_t\big[ \boldsymbol{x_t} \notin X_\text{obs}\big] >p_\text{safe}$. On the other hand, let the second condition in \eqref{eq:aux_proof_thm3} hold. Since $P_t^e[\textbf e_t \in S^e_t] > p_\text{safe}$ is equivalent to $P_t[\boldsymbol{x_t} \in S^e_t + \bar x_t] > p_\text{safe}$, then $P_t\big[ \boldsymbol{x_t} \notin X_\text{obs}\big] > p_\text{safe}$. Since $t$ was chosen arbitrarily, this result holds for all $t\in\{0,\dots,T\}$. By the previous reasoning we also obtain that $P_T\big[ \boldsymbol{x_T} \in X_\text{goal}\big] >p_\text{safe}$. Finally, taking into account Thm.~\ref{thm:ambiguity_tube} and Lemma~\ref{lemma:confidence_balls}, we have that our assumptions $P_t^e \in \mathcal{P}_t^e$ and $P_t^e[\textbf e_t \in S^e_t] > p_\text{safe}$ hold for all $t\in\naturals_0$ with confidence $1-\beta$. Therefore we obtain that $P_t\big[ \boldsymbol{x_t} \notin X_\text{obs}\big] >p_\text{safe}$ for all $t$ and $P_T\big[ \boldsymbol{x_T} \in X_\text{goal}\big] >p_\text{safe}$ hold with the same confidence.
\end{proof}

\begin{proof}[Proof of Thm.~\ref{thm:pc}]
    The proof follows the reasoning of the proof of Thm. 2 in \cite{kleinbort2018probabilistic}, which we extend here to account for the randomness and ambiguity inherent to our approach. Leveraging this reasoning is possible because we plan in the state space: we grow a tree of (deterministic) reference motion plans and uniformly sample reference states $\boldsymbol{\bar x_\text{rand}}$, feedforward controls $\boldsymbol{\bar u_\text{rand}}$ and propagation durations $\boldsymbol{\bar t_\text{rand}}$. %We sample these from a uniform distribution.
%
    %Therefore, the extension of their proof to our setting only requires redefining the set of states and controls that satisfy the collision avoidance, goal reachability and control saturation constraints, and take into account that System~\ref{eq:system} evolves in discrete-time.
 %   
    Note that, by Eq.~\eqref{eq:control_constraint_embed}, $X_\text{obs}$ generally depends on $\bar x_t$ and $\bar u_t$. Here we make this dependence explicit and denote $X_\text{obs} \equiv X_\text{obs}(\bar x_t, \bar u_t)$. Let  
    % \begin{align*}
        $F_\text{safe} := \{(x,u) : \min_{P\in\mathcal{P}^e} P\big[\boldsymbol{e} + x \notin  X_\text{obs}(x, u)\big] > p_\text{safe}\}$
    % \end{align*}
    %
    denote the \emph{free state-control space} with respect to the ambiguity, i.e., the set of pairs of safe reference states and feedforward controls at time $t$ with respect to all distributions in $\mathcal{P}^e$. Similarly, let $F_\text{goal} := \{x\in \reals^n : \min_{P\in\mathcal{P}^e} P\big[\boldsymbol{e} + x \in  X_\text{goal}\big] > p_\text{safe}\}$. Note that validity of the motion plan implies that $(\bar x_t,\bar u_t) \in F_\text{safe}$ for all $t\in\{0, 1,\dots, T\}$ and that $\bar x_T \in F_\text{goal}$.
    
    Next, define the sequence of Cartesian products of balls $B(\bar x_t, \epsilon)\times B(\bar u_t, \epsilon) \subset \reals^{n+m}$, with $\epsilon >0$. We prove that if $\epsilon$ is small enough, then the sets $B(\bar x_t, \epsilon)\times B(\bar u_t, \epsilon) \subset \reals^{n+m}$ are completely contained in $F_\text{safe}$ for all $t\in\{0, 1,\dots, T\}$ and $B(\bar x_T, \epsilon) \subset F_\text{goal}$. Pick such a $t$ and recall from Section~\ref{sec:validity_checking} that, since $\widehat P^e$ is discrete, the worst-case probability $\min_{P\in\mathcal{P}^e} P\big[\boldsymbol{e} + \bar x_t \notin  X_\text{obs}(\bar x_t,\bar u_t)\big]$ is a function of the distances $\text{dist}(\boldsymbol{\hat{e}}^i+\bar x_t, \reals^n\setminus X_\text{obs}(\bar x_t,\bar u_t))$ between the atoms $\boldsymbol{\hat{e}}^i+\bar x_t$ of $\widehat P^e*\delta_{\bar x_t}$ and the complement of $X_\text{obs}(\bar x_t,\bar u_t)$. By continuity of these distances w.r.t. $(\bar x_t, \bar u_t)$ and by Eqs.~\eqref{eq:collision_probability0} and \eqref{eq:collision_probability}, the function $(x,u) \mapsto \min_{P\in\mathcal{P}^e} P\big[\boldsymbol{e} + x \notin  X_\text{obs}(x,u)\big]$ is continuous. This implies that, for $\epsilon >0$ small enough, every $(x,u)\in B(\bar x_t, \epsilon)\times B(\bar u_t, \epsilon)$ yields a value $\min_{P\in\mathcal{P}^e} P\big[\boldsymbol{e} + x \notin  X_\text{obs}(x, u)\big]$ that is arbitrarily close to $\min_{P\in\mathcal{P}^e} P\big[\boldsymbol{e} + \bar x_t \notin  X_\text{obs}(\bar x_t, \bar u_t)\big]$. 
    Extending this logic to all $t < T$ and to $F_\text{goal}$ %$\min_{P\in\mathcal{P}^e} P\big[\boldsymbol{e} + x \in  X_\text{goal}\big]$ is also bigger than $p_\text{safe}$ for all $x\in B(\bar x_t, \epsilon)$, thus guaranteeing that $B(\bar x_t, \epsilon) \subset F_\text{safe}$. By the same reasoning, we obtain the radii of balls centered on $\bar x_t$
    we obtain an $\epsilon$ such that $B(\bar x_t, \epsilon)\times B(\bar u_t, \epsilon)\subset F_\text{safe}$ for all $t \in\{0,\dots, T\}$ and that $B(\bar x_T, \epsilon) \subset F_\text{goal}$.

    %We now consider the control constraints in \eqref{eq:cc_constraints}. Note that whether the control constraint at some time $t$ is satisfied or not only depends on $\bar u_t$, regardless of $\bar x_t$. This allows us to define the set $F_\text{control} := \{u\in U : \min_{P\in\mathcal{P}^e} P\big[-K\boldsymbol{e} + u \in  U \big] > p_\text{safe}\}$ from which, if we sample $\boldsymbol{\bar u_\text{rand}}$, the control constraint at time $t$ is satisfied. Proceeding in the same way as before, we obtain that the validity of the motion plan implies that there exists some $\epsilon_u > 0$ such that $B(\bar u_t, \epsilon_u)\subset F_\text{control}$ for all $t$.

    %Let $T_\text{prop} < \infty$ denote the maximum time of propagation of a sampled feedforward control. Note that the probability of sampling the
    %Note that the discrete-time nature of System~\eqref{eq:system} guarantees that 
    
    Next, define, at every $t$, a ball $B(\bar x_t, r_t)\subset \reals^n$ such that $r_{t+1} := 4\|A\| r_t$ and $r_T = \epsilon$. By definition, $B(\bar x_t, r_t)\times B(\bar u_t,\epsilon) \subset F_\text{safe}$, and $B(\bar x_T, \epsilon) \subset F_\text{goal}$. We now prove that, if some node $\bar x$ of the tree lies inside $B(\bar x_t, 2r_t/5)$, the probability of sampling $\boldsymbol{\bar x_\text{rand}}$ such that $\boldsymbol{\bar x_\text{near}} \in B(\bar x_t, r_t)$ and it is propagated to $\boldsymbol{\bar x_\text{new}}\in B(\bar x_{t+1}, 2r_{t+1}/5)$ is lower bounded by some $\rho > 0$, uniformly over $t$. Let $t < T$, $(x,u)\in B(\bar x_t, r_t)\times B(\bar u_t, \epsilon)$ and $x' := Ax + B u$. Taking a look at the reference dynamics in Eq.~\eqref{eq:reference_dynamics}, we obtain that $\|\bar x_{t+1} - x'\| \le \|A\|\| \bar x_t - x \| + \|B\| \|\bar u_t - u\|$. Next, we look for the set of feedforward controls that guarantee $x'\in B(\bar x_{t+1}, 2r_{t+1}/5)$, being $\|\bar u_t - u\| \le 3\|A\|r_t/(8\|B\|)$ a sufficient condition. Hence, sampling a feedforward control $\boldsymbol{\bar u_\text{rand}}\in B(\bar u_t, \min\{3\|A\|r_0/(8\|B\|), \epsilon\})$ ensures that $(x,\boldsymbol{\bar u_\text{rand}}) \in F_\text{safe}$, %i.e., that the safety constraint is satisfied at time $t$,
    and that any node $\bar x\in B(\bar x_t, r_t)$ is one-step propagated to $B(\bar x_{t+1}, 2r_{t+1}/5)$. Furthermore, the probability of uniformly sampling such a $\boldsymbol{\bar u_\text{rand}}$ is always bigger than some $\varrho_{1,t}>0$. Additionally, since Sys.~\eqref{eq:system} is discrete time, the duration $\boldsymbol{t_\text{rand}}$ that $\boldsymbol{\bar u_\text{rand}}$ is applied is sampled from a finite set, which means that $\boldsymbol{t_\text{rand}} = 1$ with probability lower bounded by some $\varrho_{2} > 0$. Next, by \cite[Lemma 4]{kleinbort2018probabilistic}, which guarantees that if there is a node $\bar x\in B(\bar x_t, 2r_t/5)$, then the probability of uniformly sampling $\boldsymbol{\bar x_\text{rand}}$ such that $\boldsymbol{\bar x_\text{near}} \in B(\bar x_t, r_t)$ is always bigger than some constant $\varrho_{3,t} > 0$. Since $\boldsymbol{\bar u_\text{rand}}$, $\boldsymbol{t_\text{rand}}$ and $\boldsymbol{\bar x_\text{rand}}$ are sampled independently, we obtain our desired probability: $\rho := \min\{\varrho_{1,t}\varrho_2\varrho_{3,t} : t \in \{0,1,\dots,T\}\}$.

    Without loss of generality, let $\rho \in (0,1/2)$. By the reasoning in the proof of \cite[Thm. 1]{kleinbort2018probabilistic}, the probability that after $k \gg T$ iterations the algorithm has not found a valid path is at most $T/(T-1)! k^T\exp(-\rho k)$, which approaches $0$ as $k\to \infty$.
\end{proof}

\begin{proof}[Proof of Thm.~\ref{thm:solution_low_dimensional}]
    Without loss of generality, let the control constraints in \eqref{eq:cc_constraints} be embedded into the collision-avoidance constraints. Let $(\mathcal{P}_{l, t}^e)_{t = 0}^T$ and $(S_{l,t}^e)_{t = 0}^T$, with $l\in\{1,\dots, L\}$, be respectively the lower-dimensional error ambiguity and confidence tubes obtained through Algs.~\ref{alg:tubes_low_dimensional} and \ref{alg:confidence_balls_lower_dimensional}, respectively. Assume $M_{l\#}P_t^e \in \mathcal{P}_{l,t}^e$ and $M_{l\#}P_t^e( S^e_{l,t} ) > p_\text{safe}/L$ for all $t\in\naturals_0$, $l\in\{1,\dots, L\}$. First, we prove Eq.~\eqref{eq:cc_constraints_safety}. Let $t \le T$. By validity of the nodes in the motion plan, it holds that:
    \begin{equation}
        \label{eq:aux_proof_thm_lower_dim}
        \begin{split}
            &{\small 1 - L + \sum_{l=1}^L \min_{P\in\mathcal{P}_{l,t}} P( \reals^{n_l} \setminus Y_\text{obs}^l ) > p_\text{safe}, }\\
            % \vspace{\textwidth}
            &\qquad\qquad\qquad\qquad\qquad \quad \\
            \text{or} \;&{\small  Y_\text{obs}^l\bigcap\big( S^e_{l,t} + M_l\bar x_t \big)  = \emptyset,\: \forall l\in\{1,\dots, L\} }.
        \end{split}
    \end{equation}
    If the first condition holds and noting that $M_{l\#}P_t^e \in\mathcal{P}_{l, t}^e =  \mathbb B(M_{l\#}\widehat P_t^e, \varepsilon_{l,t})$ is equivalent to $M_{l\#}P_t^e *\delta_{M_l \bar x_t} \in \mathcal{P}_{l, t} = \mathbb B(M_{l\#}\widehat P_t^e*\delta_{M_l \bar x_t}, \varepsilon_{l,t})$ by Prop.~\ref{prop:ambiguity_tube_error}, then
    \begin{equation}
    \label{eq:proof_thm_lower_dimensional_ambiguity_sets}
\begin{split}
    &{\small P_t\big[\boldsymbol{x_t} \notin X_\text{obs}\big] =
    1 - P_t^e * \delta_{\bar x_t}(X_\text{obs}) \ge  1 - \sum_{l = 1}^L P_t^e * \delta_{\bar x_t}(X_\text{obs}^l)}\\
    &{\small =  1 - \sum_{l = 1}^L M_{l\#}P_t^e * \delta_{M_l\bar x_t}(Y_\text{obs}^l)\ge 1 -  \sum_{l = 1}^L \max_{P \in \mathcal{P}_{l, t}} P(Y_\text{obs}^l)}\\
    &{\small = 1 - L + \sum_{l=1}^L \min_{P\in\mathcal{P}_{l,t}} P( \reals^{n_l} \setminus Y_\text{obs}^l) > p_\text{safe} }.
\end{split}
\end{equation}
Alternatively, let the second condition in \eqref{eq:aux_proof_thm_lower_dim} hold. Since $M_{l\#}P_t^e( S^e_{l,t} ) > p_\text{safe}/L$ is equivalent to $M_{l\#}P_t( S^e_{l, t} + M_l \bar x_t ) > p_\text{safe}/L$, we obtain that
\begin{equation}
\label{eq:proof_thm_lower_dimensional_ambiguity_sets1}
\begin{split}
    &P_t\big[\boldsymbol{x_t} \notin X_\text{obs}\big] \ge  1 - \sum_{l = 1}^L M_{l\#}P_t(Y_\text{obs}^l) \\
    &\ge \sum_{l = 1}^L M_{l\#}P_t(S^e_{l, t} + M_l \bar x_t) > 1 -  \sum_{l = 1}^L \frac{p_\text{safe}}{L} > p_\text{safe}.
\end{split}
\end{equation}
Since $t$ was chosen arbitrarily, this result holds for all $t\le T$. 

Next, we prove that the goal-reachability constraint \eqref{eq:cc_constraints_goal} also holds. Since the motion plan is valid, we have that one of the following conditions holds: 
    %
    % \begin{align*}
        %%\label{eq:aux1_proof_thm_lower_dim}
        % \begin{split}
            either $1 - L + \sum_{l=1}^L \min_{P\in\mathcal{P}_{l,t}} P( \reals^{n_l} \setminus Y_\text{goal}^l ) > p_\text{safe}$ or $( S^e_{l,t} + M_l\bar x_t )  \subseteq Y_\text{goal}^l$ for all $l\in\{1,\dots, L\}$. 
        % \end{split}
    % \end{align*}
    %
Noting that 
%
% \begin{align*}
    $P_T(X_\text{goal}) = P_T(\cap_{l=1}^L X_\text{goal}^l) = 1 - P_T(\cup_{l=1}^L \reals^n\setminus X_\text{goal}^l)$, 
% \end{align*}
%
along with Eqs.~\eqref{eq:proof_thm_lower_dimensional_ambiguity_sets} and \eqref{eq:proof_thm_lower_dimensional_ambiguity_sets1}, we also obtain that, if our assumptions hold, then $P_T\big[ \boldsymbol{x_t} \in X_\text{goal}\big] >p_\text{safe}$. Finally, taking into account Thm.~\ref{thm:ambiguity_tubes_low_dimensional} and Lemma~\ref{lemma:confidence_balls_lower_dimensional}, we have that $M_{l\#}P_t^e \in \mathcal{P}_{l,t}^e$ and $M_{l\#}P_t^e( S^e_{l,t} ) > p_\text{safe}/L$ for all $t\in\naturals_0$, $l\in\{1,\dots,L\}$, with confidence $1-\beta$. Therefore, the union bound yields that $P_t\big[ \boldsymbol{x_t} \notin X_\text{obs}\big] >p_\text{safe}$ for all $t\le T$ and $P_T\big[ \boldsymbol{x_t} \in X_\text{goal}\big] >p_\text{safe}$ hold with the same confidence, concluding the proof.
\end{proof}

\begin{proof}[Proof of Thm.~\ref{thm:pc_lower_dimensional}]
    The proof is analogous to that of Thm.~\ref{thm:pc}, except for defining 
    % \begin{equation*}
    % \begin{split}
        $F_\text{safe} := \{(x,u) : 1 -  \sum_{l = 1}^L \max_{P\in\mathcal{P}_l^e} P\big[M_l(\boldsymbol{e} + x) \notin  Y_\text{obs}^l(x, u)\big] > p_\text{safe}\}$
    % \end{split}
    % \end{equation*}
    %
    and $F_\text{goal}$ in a similar way.
\end{proof}

\bibliographystyle{IEEEtran}
\bibliography{refs}

% Biography text here without a photo.

% \begin{IEEEbiography}[{\includegraphics[width=1in,height=1.25in,clip,keepaspectratio]{fig1.png}}]{IEEE Publications Technology Team}
% In this paragraph you can place your educational, professional background and research and other interests.\end{IEEEbiography}

\end{document}